\documentclass[11pt]{article}
\usepackage{amsmath}
\usepackage{graphicx}
\usepackage{enumerate}
\usepackage{natbib}
\usepackage{url} 
\usepackage{enumitem}
\usepackage{latexsym, epsfig, amssymb, amsthm, mathrsfs, bbm, enumitem}
\usepackage[OT1]{fontenc}
\usepackage{xcolor}
\usepackage[colorlinks,citecolor=blue,linkcolor=blue,urlcolor=blue]{hyperref}
\usepackage{multirow,multicol,booktabs}
\usepackage{anysize}

\marginsize{1.1in}{1.1in}{0.55in}{0.8in} %
\usepackage{mathtools}

\makeatletter

\usepackage[ruled]{algorithm2e}
\usepackage{tikz}
\usepackage{float}
\usepackage{booktabs}
\usepackage{adjustbox}

\DeclareMathOperator{\Var}{Var}
\DeclareMathOperator{\Cov}{Cov}
\DeclareMathOperator{\Corr}{Corr}
\newtheorem{lemma}{Lemma}
\newtheorem{theorem}{Theorem}
\newtheorem{definition}{Definition}
\newtheorem{remark}{Remark}
\newtheorem{example}{Example}
\newtheorem{corollary}{Corollary}

\newcommand{\ignore}[1]{}{}

\def\independenT#1#2{\mathrel{\setbox0\hbox{$#1#2$}%
		\copy0\kern-\wd0\mkern4mu\box0}}

\renewcommand{\hat}{\widehat}

\usepackage{xcolor}
\usepackage[draft,inline,nomargin,index]{fixme}
\fxsetup{theme=color,mode=multiuser}
\FXRegisterAuthor{yf}{ayf}{\color{magenta}}

\newcommand{\bX}{\boldsymbol{X}}
\newcommand{\bI}{\boldsymbol{I}}
\newcommand{\bXT}{\widetilde{\bX} }
\newcommand{\bA}{\boldsymbol{A}}
\newcommand{\bR}{\boldsymbol{R}}
\newcommand{\bZ}{\boldsymbol{Z}}
\newcommand{\bQ}{\boldsymbol{Q}}

\newcommand{\bSigma}{\boldsymbol{\Sigma}}
\newcommand{\bLambda}{\boldsymbol{\Lambda}}

\newcommand{\mbR}{\mathbb{R}}
\newcommand{\mcN}{\mathcal{N}}
\newcommand{\FDP}{{\rm FDP}}
\newcommand{\FDR}{{\rm FDR}}
\newcommand{\Power}{{\rm Power}}
\newcommand{\sign}{{\rm sign}}

\newcommand{\Expected}{\mathbb{E}}
\newcommand{\Prob}{\mathbb{P}}

\newcommand{\hS}{\widehat{S}}
\newcommand{\hth}{\widehat{\theta}}
\newcommand{\wht}{\widehat{t}}

\usepackage[ruled]{algorithm2e}
\usepackage{enumitem}
\usepackage{amsmath}
\usepackage{authblk}

\newtheorem{assumption}{Assumption}

\begin{document}


\title{Differentially Private Model-X Knockoffs via Johnson-Lindenstrauss Transform}

\author[1]{Yuxuan Tao\thanks{Contact: \texttt{yntao@usc.edu}}}
\author[2]{Adel Javanmard\thanks{Contact: \texttt{ajavanma@usc.edu}}}
\affil[1]{Department of Mathematics, University of Southern California, Los Angeles, CA, USA}
\affil[2]{Department of Data Sciences and Operations, University of Southern California,
\mbox{Los Angeles}, CA, USA}

\maketitle

\begin{abstract}

We introduce a novel privatization framework for high-dimensional controlled variable selection. Our framework enables rigorous False Discovery Rate (FDR) control under differential privacy constraints. While the Model-X knockoff procedure provides FDR guarantees by constructing provably exchangeable ``negative control" features, existing privacy mechanisms like Gaussian noise injection disrupt its core exchangeability conditions. In this work we consider privatizing the data knockoff matrix through Johnson--Lindenstrauss Transform (JLT), a dimension reduction technique that simultaneously preserves covariate relationships through approximate isometry for $(\epsilon,\delta)$-differential privacy.

We theoretically characterize both FDR and the power of the proposed private variable selection procedure asymptotically. Our theoretical analysis characterizes the role of different factors, such as the privacy parameters, sample size, and feature dimension, in shaping the privacy-power trade-off. Our analysis is based on a novel `debiasing technique' for high-dimensional private knockoff procedure.
We further establish sufficient conditions under which the power of the proposed procedure converges to one. This work bridges two critical paradigms---knockoff-based FDR control and private data release. Our analysis demonstrates that structural privacy preservation through random projections outperforms the classical noise addition mechanism, maintaining statistical power even under strict privacy budgets.
\end{abstract}


\section{Introduction}
\subsection{Differentially Private FDR Control}
Feature selection plays a pivotal role in high-dimensional statistical analysis and predictive modeling by isolating the most informative variables for outcome prediction. This process is foundational not only in statistics and machine learning but also in a broad spectrum of applied fields, including healthcare, biomedicine, and finance, where interpretability is of crucial importance.

A central challenge in feature selection is the risk of false discoveries which is the erroneous identification of irrelevant variables as significant predictors. To address this challenge, the False Discovery Rate (FDR) has emerged as a standard criterion in multiple hypothesis testing. FDR is defined as the expected proportion of false positives among all variables deemed significant, providing a practical and interpretable measure for controlling error rates in feature selection procedures \citep{Benjamini1995Controlling}. 

In recent years, a series of studies proposed to control FDR via constructing knockoff variables \citep{Barber_2015,candes2017panning,barber2018knockoff}. The Model-X knockoffs \citep{candes2017panning}, compared to the oracle version \citep{Barber_2015}, operates under fewer assumptions about the relationship between the variables and the response, making it more flexible for practical applications.

Notably, feature selection applications frequently involve datasets containing sensitive personal information. For example, healthcare research analyzing associations between diseases and single-nucleotide polymorphisms (SNPs) often relies on highly confidential genetic and medical data \citep{WTCCC}. Similar concerns arise in government census data, social media datasets, and other large-scale collections of personal information. Implementing privacy-preserving feature selection protects sensitive personal data from adversarial information retrieval. In response to these challenges, differential privacy (DP) establishes a formal framework for quantifying privacy guarantees by introducing calibrated randomness into algorithms \citep{DworkMcSherryNissimSmith}. This ensures that two datasets differing in a single entry produce outputs with statistically indistinguishable probability distributions. Differential privacy quantifies privacy through the parameters $\epsilon$ and $\delta$, which represent the strength of privacy guarantees. Smaller values of $\epsilon$ and $\delta$ correspond to stricter privacy protections. 

The first study on differentially private FDR control is conducted by \cite{dwork2021differentially}, which adapts the Benjamini-Hochberg procedure by injecting calibrated noise into p-values before thresholding, ensuring ($\epsilon$, $\delta$)-differential privacy while maintaining FDR control. Building on this, DP-AdaPT \citep{xia2023adaptive} enhances power via adaptive thresholding and p-value transformations, leveraging a mirror peeling algorithm to optimize noise injection. For sequential settings, \cite{zhang2020paprikaprivateonlinefalse} extend DP guarantees to online hypothesis testing, managing cumulative privacy loss in real-time decisions.  \cite{cai2023private} privatizes sparsity computation for high-dimensional linear regression and FDR control through private Bayesian Information Criterion and estimator debiasing. 

In the context of private variable selection with False Discovery Rate (FDR) control, \cite{cai2025knockoffsinferenceprivacyconstraints} introduced a Differentially Private Knockoff (DP-knockoff) method that focuses on protecting privacy during the inference procedure itself. Specifically, their approach integrates mechanisms like the mirror peeling algorithm into the Model-X knockoff framework to ensure the DP requirement. The privacy preservation is achieved by analyzing and adding calibrated noise to the knockoff statistics (the outputs of the data-dependent step), ensuring robust privacy guarantees while maintaining the exact FDR control properties of the original knockoff procedure. The work further examines the trade-off between privacy and power based on the sensitivity level of different knockoff statistics. Unlike our method, which applies a privacy mechanism to the entire raw dataset before any statistical procedure is run, the approach by \cite{cai2025knockoffsinferenceprivacyconstraints} is based on privatizing the intermediate summary statistics used for knockoffs inference. This distinction frames the two approaches as addressing privacy at different stages of the variable selection pipeline: pre-processing the data versus perturbing the output of the dependence measures.

We noticed two major disadvantages of the existing private FDR control methods. First, most of the aforementioned methods lack flexibility regarding other models. Because the existing literature typically applies privatization to statistics or algorithm outputs, researchers must conduct a new privacy analysis and expend additional privacy budgets whenever they wish to apply a different model to the non-private dataset. Given that public datasets are accessed repeatedly, designing a new mechanism and conducting a fresh analysis for every new algorithm is both inconvenient and impractical. Second, although ``data-level" privatization methods, such as direct noise injection to the second-moment data matrix \citep{PournaderiXiang}, are compatible with different algorithms, they do not guarantee that the released second-moment matrix remains positive semi-definite (PSD). Consequently, it leads to ill-posed optimization problems.

\subsection{Input Perturbation vs. Output Perturbation}\label{sec:input-output}

Depending on where privatization occurs, privatization methods can be categorized as either input perturbation or output perturbation. An essential property of differential privacy is its post-processing safety: if an algorithm is differentially private, then any post-processing applied to its output will also be differentially private, provided that the post-processing does not access the original raw data. This property allows privatization to be applied at any stage of a data-processing pipeline. 

Output perturbation typically involves adding noise directly to the final model outputs. Additionally, some methods like differentially private stochastic gradient descent (DP-SGD) \citep{Abadi_2016} inject noise into the gradient calculations during the training process itself, characterizing them as ``in-process'' rather than pure output perturbation. A major drawback of these approaches is that their privatization mechanisms are not easily transferable to other models. For example, the calibration of the injected noise depends heavily on the output's sensitivity to changes in individual data entries. Since this sensitivity varies across models and is non-trivial to compute, a new analysis is required whenever a new model is applied.

On the other hand, input perturbation ensures that all subsequent analyses are performed only on already privatized data. This approach can provide stronger data-level privacy guarantees and enable more flexible data-sharing arrangements, as the raw data is never accessed again by downstream users or data curators. Furthermore, once data has been privatized through input perturbation, third parties can conduct a wide range of analyses or model training without introducing additional privacy risks compared to output perturbation methods. Specifically, if a non-private dataset is queried by $(\epsilon,\delta)$-differentially private algorithms for $k$ times, then the total privacy loss is $(k\epsilon,k\delta)$ by basic composition theorem or $(O(\sqrt{k}\epsilon+k\epsilon^2),O(k\delta))$ by advanced composition theorem \citep{DworkRoth}. In contrast, input perturbation does not suffer from such loss in $k$ because the data is queried only once.

Examples of input perturbation include methods that locally obfuscate data before sharing with any central unit. A notable example is the approach used in \cite{duchiminimax}, which applies local differential privacy to ensure that each individual's data is privatized at the source before being collected or analyzed centrally. \cite{duchiminimax} established minimax bounds that characterize the optimal achievable error when operating under local privacy constraints.
In the context of convex optimization, the minimax error bound indicates that applying $\epsilon$-local differential privacy effectively reduces the usable sample size from $n$ to approximately $\epsilon^2 n / d$, where $d$ is the dimension of the objective function’s domain. As a result, achieving the same convergence rate under local differential privacy requires a much larger sample size, with an additional factor of 
$d/\epsilon^2$, where 
$d$ is the data dimension and 
$\epsilon$ is the privacy parameter. This means that statistical inference in high-dimensional settings with local privacy constraints becomes significantly more challenging, as the effective sample size is substantially reduced and much more data is needed to maintain statistical accuracy.

\cite{PournaderiXiang} introduce a private knockoff-based FDR control method that injects noise to the second-moment of the data matrix. While this privatization method can adapt to any algorithm based on the second-moment matrix and demonstrate reasonable power with FDR control, there remains limited understanding of the precise trade-off between the strength of the privacy constraint and the statistical power of the knockoff procedure. Additionally, a significant limitation of privatizing the second-moment matrix by noise injection is that the resulting matrix is not guaranteed to be positive semi-definite (PSD), which poses challenges for constructing valid knockoffs and ensuring reliable inference. A non-PSD private estimate of the second-moment matrix in linear regression leads to a non-convex optimization problem, which imposes much stricter requirements for the convergence of iterative methods such as gradient descent. As discussed in \cite{sheffet2015private} and \cite{XKI11}, and as further demonstrated in our simulation studies, a non-PSD estimate for the second moment matrix can be detrimental in regression applications.

 In this work, we study the FDR control problem with privatized data, and characterize the privacy-power trade-off for model-X knockoffs. Instead of direct noise injection, we apply input perturbation using random projection, guaranteeing PSD output.

\subsection{Johnson--Lindenstrauss Transform (JLT)}

Beyond direct noise injection, the Johnson--Lindenstrauss Transform (JLT) offers an elegant approach to differentially private covariance estimation. By leveraging random projections, JLT obscures small changes in the input data while nearly preserving the covariance structure, thus enabling privacy protection with minimal distortion to the underlying data relationships \citep{blocki2012johnsonlindenstrauss}. The classical JLT \citep{JL84} guarantees that a dataset can be projected into a lower-dimensional space using a random Gaussian matrix while approximately maintaining pairwise distances. This provides a theoretical foundation for using JLT in second-moment matrix approximation. In particular, given data matrix $\bA\in\mbR^{n\times d}$, JLT enables the private release of a covariance matrix
\[
\bSigma^* = \frac{1}{r} \bA^\top\bR^\top\bR\bA\,,
\]
where $\bR\in\mbR^{r\times n}$ is a random projection matrix whose entries are i.i.d. samples from $\mcN(0,1)$ \citep{blocki2012johnsonlindenstrauss}. 

In contrast to noise injection, the Johnson--Lindenstrauss Transform (JLT) guarantees that the released second-moment matrix is PSD, making it well-suited for linear regression tasks that require a well-posed, convex optimization problem and rely solely on the second moment of the data. This technique has also been applied to private Ordinary Least Squares (OLS) estimation and confidence interval construction \citep{pmlr-v70-sheffet17a}. However, there is a lack of analysis in the existing literature regarding penalized linear regression methods such as Lasso when using JL-transformed data. In the knockoff procedure, feature statistics based on penalized linear regression, such as the Lasso Coefficient Difference (LCD), are often used to achieve high power. Therefore, a thorough analysis of Lasso with JL-transformed data is important for making valid inferences when knockoffs are privatized using the JLT.

\subsection{Our Contribution and Overview}

In this work, we introduce the use of the Johnson--Lindenstrauss Transform (JLT) to achieve differential privacy in the Model-X knockoff procedure. Our primary focus is to systematically investigate the interplay between privacy constraints and the statistical performance of the Model-X knockoff procedure when privatized using JLT. Specifically, we aim to understand how increasing privacy requirements impact both the power and false discovery rate (FDR) of knockoff-based inference. To this end, we employ a novel feature statistic derived from a debiased version of the Lasso estimator, inspired by \cite{JM14}. Leveraging this debiased approach, we provide an asymptotic characterization of both the statistical power and FDR for the privatized Model-X knockoff.

Our theoretical analysis demystifies the role of different factors, such as the JLT’s dimension reduction ratio, signal-to-noise ratio, differential privacy parameters, sample size and feature dimension, in shaping the privacy-power trade-off. In summary, our main contributions are as follows:
\begin{itemize}
\item We propose the Johnson--Lindenstrauss Transform as a principled method for privatizing the Model-X knockoff procedure, ensuring a PSD second-moment matrix and well-posed optimization for regression-based feature selection.

\item We propose a novel class of private feature statistics based on the debiased private Lasso estimator, encompassing popular methods such as Lasso Coefficient Difference (LCD) \citep{Barber_2015}.

\item We provide a rigorous power analysis of the privatized knockoff procedure, enabling practical predictions of power, and a deeper understanding of the privacy-power trade-off.

\item We demonstrate that the Johnson--Lindenstrauss Transform (JLT) outperforms the Gaussian Mechanism, which lacks guarantees for PSD outputs and can lead to ill-posed optimization problems in linear regression applications. Furthermore, We show that JLT outperforms the Gaussian Mechanism with projection to the PSD cone as well.
\end{itemize}

Our work bridges a critical gap in the literature by offering both theoretical and empirical guidance for practitioners seeking to balance privacy and statistical power in high-dimensional inference with Model-X knockoffs with data-level privacy.

The paper is organized as follows: Section~\ref{sec:prelim} provides a concise introduction to the Model-X knockoff framework, the fundamentals of differential privacy, and the Johnson--Lindenstrauss Transform (JLT). Section~\ref{sec:method-JL} details the JLT-based privatization procedure, introduces the debiased private Lasso estimator, and outlines the private feature selection process using Model-X knockoffs. Section~\ref{sec:utility} presents a technique for characterizing the distribution of the private debiased Lasso estimator, along with a thorough analysis of the power and false discovery rate (FDR) of the proposed method. Section~\ref{sec:baseline} discusses the baseline approach, specifically the application of the Gaussian mechanism to the private knockoff procedure. Finally, section~\ref{sec:simulation} showcases the simulation results that illustrate the empirical performance of these methods. 

\section{Preliminary}\label{sec:prelim}
\textbf{Notation.} For matrix norms, we denote by $|\cdot|_\infty$ the entry-wise maximum absolute value: for $\bA\in\mbR^{m\times n}$, $|\bA|_\infty = \max_{i\in[m],j\in[n]}|\bA_{ij}|$. Further, $\|\bA\|_2$ and $\|\bA\|_F$ respectively denote the operator norm and the Frobenius norm of $\bA$. For vector norms, we denote by $\|\cdot\|_1$, $\|\cdot\|_2$,$\|\cdot\|_\infty$ the $\ell_1$-norm, $\ell_2$-norm, and the max norm, respectively. Throughout, we use boldface letters to indicate matrices, while vectors and scalars are represented by non-bold letters.
For asymptotic notations, we use $O_p(\cdot)$, $o_p(\cdot)$ for order in probability. For example, $X_n=o_p(a_n)$ means that the probability of $|X_n/a_n| > \varepsilon$ goes to 0 for any positive $\varepsilon$.
  
\subsection{FDR control and Model-X Knockoffs}

Consider a linear model
\begin{equation}\label{eq:linear-model}
    y= \bX\theta_0+\xi\,,
\end{equation}
where we observe the design matrix $\bX\in\mbR^{n\times p}$ and the response $ y\in\mbR^{n\times 1}$. In addition, $\xi\in\mbR^{n\times 1}$ is a vector of i.i.d., zero-mean random variables. The model parameter $\theta_0=(\theta_{0,1},\dotsc,\theta_{0,p})^\top$ is unknown. 

To simulate the scenario of feature selection, we assume that only a small number of $\theta_{0,j}$ are non-zero. We denote the sparsity set as $S_0 = \{j\in[p]:\theta_{0,j}\neq 0\}$, and the sparsity level as $s_0 = |S_0|$. In statistical literature, a feature selection procedure is treated as a multiple hypothesis testing problem, where the null hypothesis is 
\[H_{0,j}:\theta_{0,j}=0,\phantom{sss}j\in[p]\,.\]
 A zero feature/variable is also called a null feature/variable, which will be used interchangeably throughout. Given $(\bX,y)$,  a feature selection procedure outputs $\hS_0\subseteq[p]$ as an estimate for $S_0$. The performance of the procedure can be evaluated by \emph{False Discovery Proportion (FDP)} and \emph{Power}, defined as the following:
\begin{equation}\label{eq:Power-FDP}
    \FDP := \frac{|S_0^c\cap\hS_0|}{|\hS_0|\vee 1}, \phantom{sssss}\Power := \frac{|S_0\cap\hS_0|}{|S_0|}\,.
\end{equation}
An ideal feature selection procedure should maximize power by identifying as many non-zero features from $S_0$ as possible, while maintaining a low FDP by minimizing the selection of zero features. False Discovery Rate (FDR) is the expected value of FDP over the randomness of $\bX$ and $y$,
\begin{equation}\label{eq:FDR}
    \FDR := \Expected(\FDP)\,.
\end{equation}
The study of FDR control is to design feature selection procedures such that $\FDR\leq q$, for a user-specified tolerance level $q\in(0,1)$.

The Lasso estimator is given by the solution of the following $\ell_1$-regularized loss 
\begin{equation}\label{eq:lasso-obj}
    \hth = \arg\min_{\theta\in\mbR^p}\frac{1}{2n}\|\bX \theta -y\|^2_2+\lambda\|\theta\|_1\,,
\end{equation}
and has been arguably the most common estimator in feature selection procedures for two decades because it is relatively easy to solve and tends to give a sparse solution. An intuitive selection procedure based on Lasso is given by
\[
\hS_0 := \{j\in[p]: \hth_j\neq 0\}\,.
\]
$|\hth_j|$ can be considered as the feature statistics for this selection procedure. 
Model selection (support recovery) of the Lasso estimator has been studied in the literature, typically by bounding $\mathbb{P}(\hS_0\neq S_0)$; see e.g \cite{zhao2006model,wainwright2009sharp,BulmannvandeGeer11,NIPS2013_a223c6b3}. However, correct support recovery depends, in a crucial way, on the so-called irrepresentability condition of \cite{zhao2006model}. 
 
\cite{Barber_2015} introduced the knockoff filter, a variable selection method for controlling the false discovery rate (FDR) in homoscedastic linear models. This method achieves exact FDR control in finite samples without making specific assumptions about the covariates or the magnitudes of the unknown regression coefficients. However, it is limited to the regime where $n\ge p$. This approach was extended by \citep{candes2017panning}, where Model-X knockoffs provide valid inference even when the conditional distribution of the response is arbitrary and completely unknown, and regardless of the number of covariates. As the name suggests, the Model-X knockoff framework requires the covariates to be random (i.e., independent and identically distributed rows) with a distribution that is known or can be approximated sufficiently well. We also refer to \cite{javanmard2014hypothesis} for alternative methodologies for testing hypotheses about individual parameters in high-dimensional models.

The core idea of the knockoff framework is to generate an auxiliary null variable, called a knockoff variable, for each original variable (covariate). These knockoff variables are designed to mimic the correlation structure of the original variables, ensuring that if an original variable is truly null, its feature statistic will be statistically indistinguishable from that of its knockoff copy. Leveraging this property, we can estimate the false discovery proportion and control the FDR using a data-driven threshold. The present work focuses on the Model-X knockoffs regime introduced by \cite{candes2017panning}. 
\begin{definition}[Model-X knockoffs]\label{def:knockoffs}
     Given a family of random variables $X_1,...,X_p$ and a response $y$, the model-X knockoffs for $X_1,...,X_p$ are a new family of random variables $\widetilde{X}_1,...,\widetilde{X}_p$ such that 
    \begin{enumerate}
        \item (exchangeable) for any $S\subset\{1,...,p\}$, 
        \[
        (X_1,...,X_p,\widetilde{X}_1,...,\widetilde{X}_p) \stackrel{d}{=}(X_1,...,X_p,\widetilde{X}_1,...,\widetilde{X}_p)_{\mbox{swap}(S)}\,;
        \]
        \item (conditionally independent) $(\widetilde{X}_1,...,\widetilde{X}_p)\mid(X_1,...,X_p)$ is independent of $y$.
    \end{enumerate}
    We obtain $(X_1,...,X_p,\widetilde{X}_1,...,\widetilde{X}_p)_{\mbox{swap}(S)}$ from $(X_1,...,X_p,\widetilde{X}_1,...,\widetilde{X}_p)$ by swapping $X_j$ and $\widetilde{X}_j$ for every $j\in S$.
\end{definition}
When the distribution of the original features $(X_1,...,X_p)$ are known, the knockoffs can be constructed by consecutively sampling from the law of $X_j\mid(X_{-j},\widetilde{X}_1,...,\widetilde{X}_{j-1})$ for $j =1,...,p$, where $X_{-j}$ denotes the family of all the original features except $X_j$ \citep{candes2017panning}. In this work, we assume that the entries of the design matrix $\bX$ are sampled independently from a known distribution $\mathcal{D}$. This is a special case where the corresponding knockoff matrix $\bXT$ is then generated as an independent copy of $\bX$. It is worth emphasizing that such assumption is made for the power analysis, later detailed in section~\ref{sec:utility}. As later detailed in Section~\ref{sec:method-JL}, for our private knockoff method, the finite-sample FDR control guarantee (Theorem~\ref{thm:FDR}) and the privacy guarantee (Theorem~\ref{thm:privacy}) still hold with dependent variables as long as the knockoff copies satisfy Definition~\ref{def:knockoffs}.

After constructing the knockoff matrix, we combine it with the original design matrix to form the augmented matrix $[\bX\ \bXT]$. To identify the non-zero features, we define a statistic $W([\bX\ \bXT],y)=(W_1,\dotsc,W_p)$ for each feature. As discussed in \cite{candes2017panning}, for the null features, the signs of all $W_j$'s behave like independent coin tosses, and large values of $W_j$ provide stronger evidence against the null hypothesis $H_{0,j}$. For example, Lasso Coefficient Difference (LCD) is a common choice of feature statistics. Specifically, we first compute the Lasso solution with the augmented matrix:
\begin{equation}\label{eq:lasso-knockoff}
    \hth(\lambda) = \arg\min_{\theta\in\mbR^{2p}}\frac{1}{2n}\|[\bX\ \bXT]\theta-y\|_2^2+\lambda\|\theta\|_1\,.
\end{equation}
Define LCD statistics to be
\[
W_j = |\hth_j|-|\hth_{j+p}|, \phantom{ss} j\in[p]\,.
\]
If $\theta_{0,j}=0$, then $\hth_j$ and $\hth_{j+p}$ are statistically indistinguishable, implying that $W_j$ has a symmetric distribution around zero. Conversely, if $\theta_{0,j}\neq 0$, the magnitude of $|\theta_{0,j}|$ typically stands out relative to null features, causing $W_j$ to be a large positive value and providing evidence against the null hypothesis $\theta_{0,j}=0$. Besides LCD, any statistics exhibiting this coin-toss behavior can be applied. In our work, we utilize feature statistics based on the debiased Lasso method proposed by \cite{JM14}, detailed in Section~\ref{sec:method-JL}.

Lastly, given a target FDR level $q\in(0,1)$, we compute a data-dependent threshold,
\[
\wht = \min\left\{t\in\mathcal W: \widehat{\FDP}(t) \leq q\right\},\quad \widehat{\FDP}(t) = \frac{1+|\{j\in[p]:W_j\leq -t\}|}{1\vee|\{j\in[p]:W_j\geq t\}|}\,,
\]
where $\wht=+\infty$ when the above set is empty and $\mathcal{W}=\{|W_j|:j\in[p]\}\setminus{\{0\}}$. It has been shown by \cite{candes2017panning} that the following selection satisfies the FDR control guarantee, satisfying $\FDR\leq q$:
\[
\hS_0 = \left\{j\in[p]:W_j\geq \wht\,\right\}\,.
\]
\begin{remark}
    One of the important advantages of the model-X knockoffs is that it controls FDR without any assumptions on the underlying model $y\mid\bX$. In this work, we impose a linear model assumption only for the purpose of power analysis in Section~\ref{sec:utility}, where we characterize the distribution of the privatized feature statistics. This assumption is therefore not intrinsic to the validity of the private knockoff procedure itself. In particular, the exchangeability property between $\bX$ and $\bXT$ in Definition~\ref{def:knockoffs} is preserved after applying the JL transform. Consequently, our FDR control guarantee (Theorem~\ref{thm:FDR}) can be extended to non-linear models. Similarly, by the post-processing property of differential privacy, our privacy guarantee (Theorem~\ref{thm:privacy}) remains valid for non-linear models. From a utility perspective, we expect the proposed JLT-based privatization mechanism to remain effective for a broader class of models that rely on the second moment of the data matrix such as generalized linear model. However, extending our power analysis results to non-linear models would require new technical arguments with potentially stricter regularity conditions.
\end{remark}

\subsection{Differential Privacy and Johnson--Lindenstrauss Transform (JLT)}\label{sec:prelim-jlt}
Differential privacy is a mathematically rigorous framework that enables the release of statistical information about datasets while protecting the privacy of individual data subjects. Suppose a dataset $A_1\in \mathcal{X}$ consists of data entries: $A_1 = (a_1,a_2,\dotsc,a_m)$. We call $A_1, A_2$  \emph{neighboring datasets} if $A_1$ and $A_2$ only differ by one data entry. For example, $A_2 = (a_1,\dotsc,a_{k-1},a'_k,a_{k+1},\dotsc,a_m)$. Suppose a randomized algorithm takes input from $\mathcal{X}$. The goal of differential privacy is to ensure that the output of the algorithm is almost independent of any single entry in the data. This is formally captured by the definition of differential privacy, given below.
\begin{definition}
    (Differential privacy) For $\epsilon>0,\delta\in(0,1)$, a randomized algorithm $\mathcal{M}:\mathcal{X}\to\mathcal{F}$ is $(\epsilon,\delta)$-differentially private if given any $F\subseteq\mathcal{F} $, $A_1,A_2\in\mathcal{X}$ such that $A_1,A_2$ differ by at most one data entry,
    \[
    \Prob(\mathcal{M}(A_1)\in F) \leq e^\epsilon\Prob(\mathcal{M}(A_2)\in F)+\delta\,.
    \]
\end{definition}

Post-processing safety is a key property of differential privacy, meaning that any function applied to the output of a differentially private algorithm does not weaken its privacy guarantees \citep[Proposition 2.1]{DworkRoth}.

\begin{lemma}\label{lem:DP-post}
    (Post-processing) If an algorithm $\mathcal{M}:\mathcal{X}\to\mathcal{F}$ is $(\epsilon,\delta)$-differentially private, for arbitrary random function $f:\mathcal{F}\to\mathcal{F}'$, $f\circ\mathcal{M} :\mathcal{X}\to \mathcal{F}'$ is $(\epsilon,\delta)$-differentially private.
\end{lemma}

The Johnson--Lindenstrauss Transform (JLT) provides a powerful and elegant approach for differentially private data release, particularly for second-moment (covariance) matrix estimation. Unlike direct noise injection, JLT guarantees that the privatized second-moment matrix remains PSD, which is essential for downstream applications such as regression.

Given a data matrix $\bA\in\mbR^{n\times d}$, the goal is to privately release the second-moment matrix $\bA^\top \bA$. Rather than adding noise directly to $\bA^\top \bA$, JLT privatizes the data through a random projection while preserving positive semidefiniteness. To reach the desired privacy level, the input matrix needs minimum singular value larger than a threshold $w>0$. When the minimum singular value is unknown, we ensure it by augmenting the rows of $\bA$ by a diagonal matrix $w\bI_d$. Specifically, define
\[
\bA' := \begin{bmatrix}\bA \\ w\bI_d\end{bmatrix}\,,
\]
and left-multiply $\bA'$ by a random projection matrix $\bR\in\mbR^{r \times (n+d)}$, where the entries of $\bR$ are i.i.d. Gaussian random variables with mean zero and variance $\frac{1}{r}$. The released private approximation is formed from the projected matrix $\bR\bA'$. See Algorithm~\ref{algo:JLT} for details.

The transformed matrix $(\bR\bA')^\top \bR\bA'$ serves as a randomized, PSD estimate of the original second moment, and closely approximates $\bA^\top\bA + w^2\bI$ if $\bR^\top\bR\approx \bI$ due to the Restricted Isometry Property (RIP); see e.g., \cite{candestao05}. Properly scaled Gaussian matrices satisfy RIP with high probability \citep{vershynin11}. As for the privacy guarantee, \cite{blocki2012johnsonlindenstrauss} and \cite{sheffet2015private} show that the JLT output is $(\epsilon,\delta)$-differentially private if all singular values of $\bA$ are above $w>0$, where
\begin{align}\label{eq:w-2}
w^2=\frac{4B^2}{\epsilon}\left(\sqrt{2r\log(4/\delta)}+\log(4/\delta)\right)\,.
\end{align}

The formal privacy guarantee is as stated below. We refer to \cite[Theorem A.1]{pmlr-v70-sheffet17a} for the proof.
\begin{theorem}[\cite{pmlr-v70-sheffet17a}]
    Algorithm \ref{algo:JLT} is $(\epsilon,\delta)$-differentially private.
\end{theorem}

\begin{algorithm}\label{algo:JLT}
    \caption{Johnson-Lindenstrauss Transform Privatization \citep{pmlr-v70-sheffet17a}}
    \textbf{Input}: data matrix $\bA\in \mathbb{R}^{n\times d}$ and upper-bound $B>0$ on the $\ell_2$-norm of every row of $\bA$; privacy parameters $\epsilon>0,\ \delta\in(0,1/e)$; number of rows in the projection matrix, $r$.
    \begin{enumerate}[label=\arabic*.]
        \item Compute $w^2=\frac{4B^2}{\epsilon}\left(\sqrt{2r\log(4/\delta)}+\log(4/\delta)\right).$
        \item Append a scaled identity matrix to $\bA$, i.e. $\bA' := \begin{bmatrix}
            \bA \\
            w\bI_d
        \end{bmatrix}$ 
        \item Generate random matrix $\bR\in\mbR^{r\times (n+d)}$ by sampling entries from $\mcN(0,\frac{1}{r})$.
        \item \textbf{Return} $\bR\bA'$.
    \end{enumerate}
\end{algorithm}

\section{Methodology}\label{sec:method-JL}
\subsection{JL Privatization}
Consider the linear model as in (\ref{eq:linear-model}), where we observe $\bX \in \mathbb{R}^{n\times p}$, $y \in \mbR^n$, and our goal is to design an $(\epsilon, \delta)$-differentially private feature selection procedure that releases an estimate $\hS_0$ of the model support, $S_0$. The target FDR level is $q\in(0,1)$. Suppose that the entries of $\bX$ are sampled independently from a known, bounded distribution $\mathcal{D}$, and generate the knockoff matrix $\bXT$ by sampling from $\mathcal{D}$ independently. We construct an augmented matrix $\bA\in\mbR^{n\times (2p+1)}$ by concatenating $\bX$, $\bXT,$ and $y$:
\[
\bA := [\bX\ \bXT\ y]\,.
\]
Choose $r$ as the number of rows for the Gaussian matrix $\bR$, which will be further addressed in section \ref{sec:utility}. Assume the $\ell_2$ norm for each row of $\bA$ is bounded by $B$. Then perform JLT as in Algorithm (\ref{algo:JLT}) with the input $\bA,B,\epsilon,\delta$, (and so $d=2p+1$ in this case) which returns
\[
\bA^* := \bR\begin{bmatrix}
    \bA\\
    w\bI_{2p+1}
\end{bmatrix}\,, 
\]
where $\bR\in\mbR^{r\times(n+2p+1)}$ consists of entries sampled from $\mcN(0,1/r)$, and $w$ is defined as in Algorithm (\ref{algo:JLT}). In this work we use asterisk `$*$' to indicate private releases.

By splitting $\bA^*\in\mbR^{r\times(2p+1)}$, we obtain privatized data $\bX^*,\bXT^*$, $y^*$:
\[
\bA^* = [\phantom{.}\underbrace{[\bX^*\ \bXT^*]}_{\in\mbR^{r\times 2p}}\phantom{ss} \underbrace{y^*}_{\in\mbR^{r\times 1}}]\,.
\]
To see the correspondence between non-privates and privates, consider a split of $\bR$ according to the dimensions of $[\bX\ \bXT]$ and $y$:
\begin{align*}
    \bR\begin{bmatrix}
    \bA\\
    w\bI_{2p+1}
\end{bmatrix} 
&=[\underbrace{\bR_1}_{\mbR^{r\times n}}\ \underbrace{\bR_2}_{\mbR^{r\times 2p}}\ \underbrace{\bR_3}_{\mbR^{r\times 1}}]
\begin{bmatrix}
    [\bX\ \bXT]\phantom{sss}y\\
    \phantom{s}w\bI_{2p+1}
\end{bmatrix}\\
&=[\underbrace{\bR_1[\bX\ \bXT]+w\bR_2}_{\in\mbR^{r\times 2p}}\phantom{ss}
    \underbrace{\bR_1y+w\bR_3}_{\in\mbR^{r\times 1}}] \,.
\end{align*}
In other words, we privatize the data matrices $[\bX\ \bXT]$ and $y$ by left-multiplying them with the same random matrix $\bR_1\in\mbR^{r\times n}$, and then adding independent Gaussian noise $w\bR_2$ and $w\bR_3$, respectively.

\subsection{Feature Statistics: Debiased Private Lasso}\label{sec:feature-statistics}
Inspired by \cite{JM14}, our choice of feature statistics are based on a debiased version of the Lasso estimator (see also \cite{javanmard2018debiasing,javanmard2019false,javanmard2013nearly,javanmard2020flexible,deshpande2023online} for extensions of this idea). To obtain the debiased estimator, we first solve a standard Lasso optimization using the private data $[\bX^*\ \bXT^*],y^*$. Specifically, for some penalty parameter $\lambda>0$, we compute:
\begin{equation}\label{eq:lasso-JL}
    \hth^*(\lambda) := \arg\min_{\theta\in\mbR^{2p}}\frac{1}{2n}\left\|[\bX^*\ \bXT^*]\theta - y^*\right\|_2^2+\lambda\|\theta\|_1\,.
\end{equation}
Note that by Lemma~\ref{lem:DP-post}, every subsequent computation based on $[\bX^*\ \bXT^*],y^*$ is $(\epsilon,\delta)$-differentially private. The debiased private Lasso estimator is then defined as:
\begin{equation}\label{eq:lasso-debias}
    \hth^u := \hth^*(\lambda) +\frac{1}{n}[\bX^*\ \bXT^*]^\top(y^*-[\bX^*\ \bXT^*]\hth^*(\lambda))+\frac{w^2}{n }\hth^*(\lambda) \,.
\end{equation}
Note that the parameter $w$ used in the debiased estimator above is the same parameter used in the privacy mechanism and defined by~\eqref{eq:lasso-debias}. As discussed before~\eqref{eq:lasso-debias} it ensures that the matrix $A^*$ has full column rank with a quantitative positive lower bound on its singular values.
Using this debiased estimator, we define feature statistics based on a general family of functions. Namely, for some $f:\mbR\to\mbR_{\ge 0}$,
\begin{equation}\label{eq:W-feat-stat}
    W_j = f(\hth^u_j)-f(\hth^u_{j+p})\,,
\end{equation} 

Lemma~\ref{lem:iid-coin-toss} below shows that the defined statistic $W_j = f(\hth^u_j)-f(\hth^u_{j+p})$ exhibits the i.i.d. coin-toss symmetry property, thus controlling the FDR, which is formally stated in Theorem~\ref{thm:FDR}. The proof of Lemma~\ref{lem:iid-coin-toss} is analogous to that of \cite[Lemma 3.3]{candes2017panning} and \cite[Lemma 1]{Barber_2015}, but must be adapted to account for the random matrix $\bR$ and the appended matrix $w\bI_{2p+1}$. We defer the proof to Appendix~\ref{proof:lem:iid-coin-toss}.
\begin{lemma}\label{lem:iid-coin-toss} 
    Let $W$ be defined as in~\eqref{eq:W-feat-stat}, and let $c=(c_1,\dotsc,c_p)$ be an independent random vector where $c_j=1$ for all $j\in S_0$ and $c_j\stackrel{i.i.d.}{\sim}Unif(\{-1,1\})$ for all $j \in S_0^c$. Then
    \[
    W\stackrel{d}{=}(c_1W_1,\dotsc,c_pW_p)\,.
    \]
\end{lemma}
Note that the i.i.d. coin-toss symmetry property of $W_j$ holds for general choice of $f:\mbR\to\mbR_{\ge 0}$, which eventually controls the FDR as later shown in Section~\ref{sec:FDR-control}. However, for the knockoff procedure to have strong power in detecting non-null variables, we need extra assumptions on $f$. For example, $f(x)$ goes to infinity when $|x|$ goes to infinity. We defer the rigorous assumptions of $f$ needed for power analysis to Section~\ref{sec:utility}. A straightforward example of $f$ is $f(\hth^u_j) = |\hth^u_j|$, hence the statistic $W_j$ is the absolute-value difference of the debiased coefficients: $W_j = |\hth^u_j|-|\hth^u_{j+p}|$. Indeed, the commonly-used Lasso Coefficient Difference (LCD) statistic falls within this general framework, as clarified by Lemma~\ref{lem:function-of-debiased} below.
\begin{lemma}\label{lem:function-of-debiased}
    Let $\hth^*$ and $\hth^u$ respectively denote the private Lasso estimator and the debiased estimator given by  (\ref{eq:lasso-JL}) and (\ref{eq:lasso-debias}). Then for all $j\in[2p]$,
    \begin{equation}\label{eq:function-of-debiased}
        \hth^*_j(\lambda) =\frac{1}{1+ \frac{w^2}{  n}}\cdot \sign(\hth^u_j)\cdot\left(|\hth^u_j|-\lambda\right)_+\,,
    \end{equation}
    where $(a)_+ = \max\{a,0\}$ for $a\in\mbR$.
\end{lemma}
The proof of Lemma~\ref{lem:function-of-debiased} is provided in Appendix~\ref{pf:function-of-debiased}. Notice that when we choose $f(\hth^u_j)$ according to \eqref{eq:function-of-debiased} in Lemma~\ref{lem:function-of-debiased}, we recover the LCD statistic: 
\[
W_j = |\hth^*_j|-|\hth^*_{j+p}|\,.
\]

\subsection{FDR Control}\label{sec:FDR-control}
We next choose a data-dependent threshold for FDR budget $q>0$,
\[
\wht = \min\left\{t\in\mathcal{W}: \frac{1+|\{j\in[p]:W_j\leq-t\}|}{1\vee|\{j\in[p]:W_j\geq t\}|}\leq q\right\}\,,
\]
where $\wht = +\infty$ when the above set is empty and $\mathcal{W}=\{|W_j|:j\in[p]\}\setminus{\{0\}}$. We then select variables 
\begin{equation}\label{eq:S0}
    \hS_0=\{1\leq j\leq p: W_j\geq \wht\ \}\,.
\end{equation}
By Lemma~\ref{lem:DP-post}, the post-processing property of differential privacy, $\hS_0$ is $(\epsilon,\delta)$-differentially private.

Finally, we claim that JL privatization method does not affect the FDR control property of the knockoff procedure. 
\begin{theorem}[FDR control]\label{thm:FDR}
    The selection $\hS_0$ defined in \eqref{eq:S0} controls the FDR at level $q$, i.e.
    \[
    \Expected\left(\frac{|S_0^c\cap\hS_0|}{|\hS_0|\vee 1}\right)\leq q\,.
    \]
    The expectation is taken over the randomness in noise $\xi$, knockoff variable $\bXT$ and privatization matrix $\bR$.
\end{theorem}
The proof of the FDR control theorem follows from the fact that the signs of $W_j$'s for all null variables $j\notin S_0$ are i.i.d. coin tosses, formally stated in Lemma \ref{lem:iid-coin-toss}. The proof of Lemma \ref{lem:iid-coin-toss} is given in Appendix~\ref{proof:lem:iid-coin-toss}, and we refer to \cite[Theorem 3.4]{candes2017panning} and \cite[Theorem 2]{Barber_2015} for the remainder of the proof.

\subsection{Privacy}
Once the augmented data matrix $[\bX \ \bXT \ y]$ has been privatized, all subsequent steps of the knockoff procedure can be viewed as post-processing of that privatized output. For example, computing feature statistics, forming $W_j$ applying the knockoff threshold, and reporting selected variables are deterministic or randomized functions of the released private object, but they do not require any further access to the raw data. By Lemma~\ref{lem:DP-post}, applying arbitrary data-independent transformations to the output of a DP mechanism does not weaken its privacy guarantee. 
\begin{theorem}[Privacy]\label{thm:privacy}
    The selection $\hS_0$ defined in \eqref{eq:S0} is $(\epsilon,\delta)$-differentially private.
\end{theorem}

The exact construction of model-X knockoffs requires a known distribution of the variables $\bX$. A known distribution of $\bX$ is not a concern for privacy since it is public information. The distribution of $\bX$ is bounded by assumption, so are the rows of $\bXT$ by exchangeability property. Hence the rows of $[\bX\ \bXT]$ is bounded. In the exact knockoffs construction, the rows of the knockoff matrix only depend on the the same row of the original design matrix. Therefore, we can ensure the privacy by doubling the sensitivity. For example, if the rows of $\bX$ is bounded by $B_X$, then the rows of $[\bX\ \bXT]$ is bounded by $2B_X$. 

When the distribution of $\bX$ is unknown and knockoffs are approximated, the privacy budget allocation largely depends on the method of approximation. For example, when knockoffs are approximated by matching the first and second moment, a private release of the first and second moment is needed, costing extra privacy budget. This cost can be alleviated by using auxiliary public covariate data to learn the feature distribution instead. Model-X knockoffs require knowledge of the marginal distribution of $X$, not a model for $Y\mid X$, so covariate-only data are directly relevant for constructing knockoffs. For example, \cite{barber2020robust} discuss estimating the distribution or conditional distributions of X from unlabeled data and then constructing approximate knockoffs, with robustness guarantees depending on the quality of these estimated conditionals.

\section{Analysis of Power and FDR for JL Knockoffs}\label{sec:utility}
It is trivial to make an algorithm private. For example, an algorithm whose output is completely independent of its input is perfectly private, but offers no value for data analysis. The real challenge in differential privacy is to maximize the utility of the algorithm while satisfying privacy constraints. In the context of FDR control, utility is most commonly measured by power, the proportion of true selections to the total number of true features, which is a crucial metric alongside FDR itself (See \eqref{eq:Power-FDP}  and \eqref{eq:FDR} for formal definition of FDR and power). Achieving a favorable balance between privacy and utility, particularly maintaining high power while controlling FDR, is a central goal in the design of differentially private algorithms. 

This section analyzes power and FDR of the JLT-privatized Model-X knockoff procedure in an asymptotic regime by characterizing the distribution of the debiased JLT-privatized Lasso estimate, $\hth^u$, as defined in (\ref{eq:lasso-debias}). To this end, in Section \ref{sec:dist-char-lasso} we further analyze the debiased private Lasso estimate, inspired by \cite{JM14} which will be used to characterize the distribution of Lasso estimate $\hth^*$. In Section \ref{sec:Power-FDR-analysis} we give an asymptotic power and FDR prediction formulae using the distribution of private Lasso. In Section \ref{sec:suff-cond-eg} we further discuss the conditions in our theory and illustrate them through several examples.

Throughout this section, in order to apply privatization methods such as Gaussian mechanism and JLT, we make several assumptions on the linear model \eqref{eq:linear-model} to impose a bounded support of the data.
\begin{assumption}\label{ass1}
    The distribution $\mathcal{D}$, from which the entries of $\bX$ are sampled independently, is centered, normalized and supported on $[-b,b]$. In particular, for any $i\in[n]$, $j\in[p]$, 
    \begin{equation}\label{eq:X-def}
        |\bX_{ij}|\leq b ,\phantom{sss}\Expected(\bX_{ij}) = 0\,, \phantom{sss} \Var(\bX_{ij}) = 1\,.
    \end{equation}
\end{assumption}
\begin{assumption}\label{ass2}
    The noise terms $(\xi_1,\cdots,\xi_n)$ are i.i.d. $\mcN(0,\sigma^2)$ for some unknown $\sigma>0$.
\end{assumption}
\begin{assumption}\label{ass3} We assume that $\|\theta_0\|_2$ is of order one. Without loss of generality and by normalization we assume $\|\theta_0\|_2=1$.
\end{assumption}
Writing $y_i = x_i^{T} \theta_0 + \xi_i$, under Assumptions~\ref{ass1} and~\ref{ass3}, we have $\Var(x_i^\top \theta_0) = \|\theta_0\|_2^2 = 1$ which is comparable to $\Var(\xi_i) = \sigma^2$. If $\|\theta_0\|_2$ grows at a different order, the task of feature selection becomes either trivial or impossible.

Under such assumptions, the rows of $\bA=[\bX\,\bXT\, y]$ has $\ell_2$ sensitivity bounded by $B:= \sqrt{b^2(2p+s_0)}=O(\sqrt{p})$, since $s_0\le p$. To simplify the presentation, we define the sensitivity of $\bA$ solely with respect to the design matrix only and treat $y$ as a function of $\bX$ to avoid the unboundedness of the Gaussian noise $\xi$. In practice, one can assume the Gaussian noise is bounded by $O(\sigma\log(n))$ with high probability and proceed the analysis with $B= O(\sigma\sqrt{\log(n)}+\sqrt{p})$ under a bounded noise assumption.

Additionally, in our asymptotic theory we assume that the dimensions $n,p,r\to\infty$. We say a probabilistic event happens \emph{``with high probability''}, if its probability converges to one asymptotically. 
\begin{remark}
    Assumption~\ref{ass1}, \ref{ass2}, and~\ref{ass3} are made only for power analysis. Our results from previous sections do not rely on these assumptions. Although the model-X knockoffs framework controls FDR without restrictions on the model, we do need to make necessary assumptions for power analysis to precisely characterize the distribution of the output statistics. 
\end{remark}

\subsection{Distributional Characterization of Debiased Private Lasso}\label{sec:dist-char-lasso}
This section is devoted to the distributional characterization of the debiased private Lasso estimator. To start with, $\bX$ has i.i.d entries and $\bXT$ are i.i.d. copies of each other, we can focus on the problem with $(\bX,y)$ other than $([\bX\ \bXT],y)$. We can then use the obtained theory by replacing $p$ with $2p$ and $\theta_0$ with $[\theta_0;{{\bf{0}}}_{p\times 1}]$. 

As defined in section \ref{sec:method-JL}, given $w$ and $r$, the privatized data $\bX^*\in\mbR^{r\times p},y^*\in\mbR^{r\times 1}$ can be written as
\begin{equation*}
    [\bX^*\ y^*] = \bR\begin{bmatrix}
        \bX\phantom{ss} y\\
        w\bI_{p+1}
    \end{bmatrix}
\end{equation*}
for a random matrix $\bR\in\mbR^{r\times (n+p+1)}$ whose entries are i.i.d. drawn from $\mathcal{N}(0,1/r)$. On the other hand, considering the split 
\begin{equation}\label{eq:R-split}
    \bR = [\underbrace{\bR_1}_{\in\mbR^{r\times n}}\ \underbrace{\bR_2}_{\in\mbR^{r\times p}}\ \underbrace{\bR_3}_{\in\mbR^{r\times 1}}]\,,
\end{equation} 
we can express $\bX^*,y^*$ separately as
\begin{equation}\label{eq:def-private-X-y-alt}
    \bX^* = \bR_1\bX+w\bR_2\,, \phantom{sss} y^* = \bR_1y+w\bR_3\,.
\end{equation}
Define the private Lasso estimate:
\begin{equation}\label{eq:lasso-JL-simpl}
    \hth^*(\lambda) := \arg\min_{\theta\in\mbR^{p}}\frac{1}{2n}\|\bX^*\theta - y^*\|_2^2+\lambda\|\theta\|_1\,,
\end{equation}
and the debiased private Lasso estimate: 

\begin{equation}\label{eq:lasso-debias-simpl}
    \hth^u := \hth^*(\lambda) +\frac{1}{n}\bX^{*\top}(y^*-\bX^*\hth^*(\lambda))+\frac{w^2}{n }\hth^*(\lambda) \,.
\end{equation}
We will show that the debiased estimator, as the name suggests, is asymptotically unbiased under specific conditions. In particular, we show that its distribution converges to a multivariate Gaussian with mean $\theta_0$. Our next theorem decomposes the debiased estimator into three components: its mean, the bias, and the (Gaussian) noise term. We refer to Appendix \ref{pf:debiased-rewrite} for the proof of Theorem~\ref{thm:debiased-rewrite}.
\begin{theorem}\label{thm:debiased-rewrite}
    Suppose $\bR\in\mbR^{r\times (n+p+1)}$ is an independent random matrix whose entries are drawn i.i.d. from $\mcN(0,1/r)$. Suppose $\bX\in \mbR^{n\times p}$ is a random matrix whose entries are i.i.d. samples from $\mathcal{D}$ as described in \eqref{eq:X-def}, and $y$ is defined as in the linear model \eqref{eq:linear-model}. Let $\bR_1\in\mbR^{r\times n},\bR_2\in\mbR^{r\times p},\bR_3\in\mbR^{r\times 1}$ be a partition of $\bR$ following the split in \eqref{eq:R-split}, and define $\bX^*$, $y^*$ as in \eqref{eq:def-private-X-y-alt}. In addition, let $\hth^*$ and $\hth^u$ be given by \eqref{eq:lasso-JL-simpl} and \eqref{eq:lasso-debias-simpl}. Then, conditioning on $\bX$ and $\bR$, we have
    \begin{equation}\label{eq:debias-rewrite}
        \hth^u = \theta_0+Z+\Delta\,.
    \end{equation}
    The noise term $Z$ is zero mean and has covariance $\sigma^2\bQ$, where
    \[
    \bQ=\Expected\left(\frac{1}{\sigma^2}ZZ^\top\mid\bX,\bR\right)=\frac{1}{n^2}\bX^{*\top}\bR_1\bR_1^\top\bX^*\,,
    \]
    and the bias term $\Delta$ is given by
    \begin{align}\label{eq:Delta-thm}
    \Delta = \left(\frac{1}{n }\bX^\top \bX-\bI_p\right)(\theta_0-\hth^*)
        +\frac{1}{n }[\bX^\top\ w\bI_p\ {\bf 0}](\bR^\top \bR-\bI_{n+p+1})
        \begin{bmatrix}
        \bX(\theta_0-\hth^*)\\ -w\hth^*\\ w
        \end{bmatrix}\,.
    \end{align}
\end{theorem}
We are interested in regimes where the bias term is much smaller than other terms. In the next lemma, we compute the expectation of the covariance matrix $\bQ$, thus obtain the asymptotic order of the noise term. This characterization becomes useful later when we derive the limit of empirical distribution of the feature statistics. We refer to Appendix \ref{pf:rho_n} for the proof of Lemma~\ref{lem:rho_n}.
\begin{lemma}\label{lem:rho_n}
    Assume the setting in Theorem \ref{thm:debiased-rewrite}. Taking the expectation over $\bX$ and $\bR$, we have
    \begin{equation}\label{eq:rho_n}
        \Expected\left(\bQ\right)  = \rho_n^2\bI_p\,,\quad\rho_n^2 = \frac{n+r+1}{n r}+\frac{w^2}{nr}\,.
    \end{equation}
    Moreover, if we let $n,r\to\infty$, for any $\varepsilon>0$, we have
    \[
    \lim_{n,r\to\infty}\Prob\left(\left|\frac{\bQ}{\rho_n^2}-\bI_p\right|_{\infty}>\varepsilon\right)=0\,.
    \]
\end{lemma} 
In the next subsection, we discuss sufficient conditions for the bias to become asymptotically negligible compared to the noise, i.e. $\|\Delta\|_\infty = o_p(\sigma\rho_n)$.

\subsubsection{Sufficient Condition for Debiasing}
We continue by establishing sufficient conditions under which $\|\Delta\|_\infty=o(\sigma\rho_n)$. In other other words, the bias term in the debiased estimator $\hth^u$ becomes asymptotically negligible compared to the size of the noise $Z$. 
Consider the input parameters, $n,p,s_0,r,\epsilon,\delta$.  We treat every parameter as a function of $n$, i.e., $n,p(n),s_0(n),r(n),\epsilon(n),\delta(n)$. Then a sufficient condition for $\|\Delta\|_\infty=o_p(\sigma\rho_n)$ can be written as a parameter tuple in the following collection:
\begin{align}\label{eq:op}
\mathcal{P} = \left\{(n,p(n),s_0(n),r(n),\epsilon(n),\delta(n)):\quad \forall \varepsilon>0, \; \lim_{n\to\infty}\Prob\left(\frac{\|\Delta\|_\infty}{\sigma\rho_n}>\varepsilon \right) = 0\right\}\,.
\end{align}

The next theorem establishes a set of sufficient conditions for~\eqref{eq:op} to hold.
\begin{theorem}\label{thm:suff-cond-debiasing}
Consider a sequence of data matrices $\bX\in\mbR^{n\times p}$, as defined in (\ref{eq:X-def}), random projection matrices $\bR\in\mbR^{r\times n}$ whose entries are i.i.d. $\mcN(0,1/r)$, with dimensions $n\to\infty$, $p(n)\to\infty$, $r(n)\to \infty$. Define the privacy parameter $w>0$ as in Algorithm \ref{algo:JLT}. Define the biased term $\Delta$ as in (\ref{eq:debias-rewrite}). There exists an absolute constant $C_\lambda$ such that 
\[\|\Delta\|_\infty = o_p(\sigma\rho_n)\,,\] 
if \textbf{all} of the following conditions are satisfied:
    \begin{enumerate}
    \item $\frac{\kappa_n^2}{\lambda^2\phi_w}=o\left(\sqrt{\frac{n}{\log p}}\cdot\min\left\{1,\sqrt{\frac{r}{n}},\frac{\sqrt{rn}}{w^2}\right\}\right)$\label{item:cond-comp}
    \item $\sqrt{s_0} = O\left(\frac{\max\big\{\frac{1}{\sqrt{n}},\frac{\log p }{\sqrt{r}}\big\}}{\max\big\{\frac{w\log p}{\sqrt{nr}},\frac{w^2}{n}\sqrt{\frac{\log p}{r}}\big\}}\right)$\label{item:cond-lambda}
    \item $\frac{1}{\sigma\rho_n}\sqrt{s_0} = o\bigg(\min \bigg\{\frac{n}{w^2}\sqrt{\frac{r}{\log p}},\frac{\sqrt{nr}}{w\log p}\bigg\}\bigg)$.\label{item:cond-sparsity}
    \item $\frac{1}{\sigma\rho_n}\frac{4\kappa_n^2}{\lambda\phi_w} = o\bigg(\min\bigg\{\sqrt{\frac{n}{\log p}},\sqrt{\frac{r}{\log p}},\frac{n}{w^2}\sqrt{\frac{r}{\log p}},\frac{\sqrt{nr}}{w\log p}\bigg\}\bigg)$.\label{item:cond-error}
    \end{enumerate}
     where $\rho_n^2 = (n+r+1)/(nr)+w^2/(nr)$, $\phi_w = 1+w^2/n$, $\lambda =C_\lambda\cdot\sigma\cdot\max\bigg\{\sqrt{(\log p)/n},\sqrt{(\log p)^3/r}\bigg\}$,  $\kappa_n = 2\lambda\sqrt{s_0}+w^2/n$.
\end{theorem}
The conditions in Theorem~\ref{thm:suff-cond-debiasing} ensure that the debiased private Lasso estimator is asymptotically dominated by its Gaussian fluctuation rather than by residual bias. Condition~\ref{item:cond-comp} guarantees that the privatized design $\bX^*$ satisfies a suitable restricted eigenvalue condition so that the private Lasso remains stable. Condition~\ref{item:cond-lambda} ensures that the regularization parameter $\lambda$ dominates the stochastic score terms arising from the privatized data. The remaining two conditions control the two sources of error that appear in the expansion of $\Delta$: terms involving the $\ell_1$ estimation error $||\widehat{\theta}^*-\theta_0||_1$ and the $\ell_1$ norm $||\widehat{\theta}^*||_1$ of the private Lasso estimator. By triangular inequality, it suffices to bound $||\theta_0||_1$ and $||\widehat{\theta}^*-\theta_0||_1$ by condition~\ref{item:cond-sparsity} and condition~\ref{item:cond-error}, respectively. 
 
The proof of Theorem \ref{thm:suff-cond-debiasing} is deferred to Appendix \ref{pf:suff-cond-debiasing}. The proof relies on several concentration inequalities to control the asymptotic behavior of various quantities in the bias, along with the optimality condition of the Lasso estimate $\hth^*$ as the solution to the optimization~\eqref{eq:lasso-obj}. 

We are now ready to characterize the asymptotic distribution of the residuals $(\hth^u_i-\theta_{0,i})$. Theorem~\ref{thm:suff-cond-debiasing} along with the decomposition in Theorem~\ref{thm:debiased-rewrite} and Lemma~\ref{lem:rho_n} naturally suggest to consider the scaled residual $(\hth^u_i-\theta_{0,i})/(\sigma \rho_n)$. In the next  lemma we consider a slightly more general scaling, where we replace $\sigma$ by a consistent estimator $\widehat{\sigma}$. The proof of Lemma~\ref{lem:marginal} can be found in \ref{pf:marginal}.
\begin{lemma}\label{lem:marginal}
    Consider the linear model~\eqref{eq:linear-model} and let $\hth^u$ be the debiased estimator of the private Lasso, given by~\eqref{eq:lasso-debias-simpl}. Let $\widehat{\sigma} = \widehat{\sigma}(\bX^*,y^*)$ be a consistent estimator of $\sigma$,  satisfying for any $\varepsilon>0$,
    \[
    \lim_{n\to\infty} \sup_{\|\theta_0\|_{0}\le s_0} \mathbb{P}\left(\left|\frac{\widehat{\sigma}}{\sigma} - 1 \right|\ge \varepsilon\right) = 0\,.
    \]
    Under the setting of Theorem \ref{thm:suff-cond-debiasing} for any $j\in [p]$, we have 
    \[
    \lim_{n\to\infty}\sup_{\|\theta_0\|_{0}\le s_0}\left|\Prob\left(\frac{\hth^u_j-\theta_{0,j}}{\widehat{\sigma}\rho_n}\leq x\right)-\Phi(x)\right|=0\,,
    \]
    where $\Phi$ is the cdf of standard Gaussian.
\end{lemma}

A simple generalization of the above result is to derive asymptotic distribution of $\hth^u_G:= (\hth^u_j)_{j\in G}$ for finite set $G\subseteq[p]$.
\begin{lemma}\label{lem:marginal-gen}
 Consider the setting of Theorem~\ref{thm:suff-cond-debiasing} and let $G=G(n)$ be a sequence of sets $G(n)\subseteq[p]$ with $|G(n)| = k$ fixed as $n,p\to\infty$.  Then, for all $x = (x_1,\dotsc, x_k)\in\mathbb{R}^k$ we have
\[
\lim_{n\to\infty}\sup_{\|\theta_0\|_{0}\le s_0}\left|\Prob\left(\frac{\hth^u_G-\theta_{0,G}}{\widehat{\sigma}\rho_n}\leq x\right)-\Phi_k(x)\right|=0\,,
\]
where $(a_1,\dotsc, a_k)\le (b_1,\dotsc, b_k)$ indicates that $a_i\le b_i$ for $i\in[k]$ and $\Phi_k(x) = \Phi(x_1)\dotsc \Phi(x_k)$.
\end{lemma}
In words, Lemma~\ref{lem:marginal-gen} provides distributional characterization for low-dimensional projections of $\hth^u$. The proof of Lemma~\ref{lem:marginal-gen} is similar to Lemma~\ref{lem:marginal} and hence we omit it.

\subsection{Power and FDR Analysis}\label{sec:Power-FDR-analysis}

Under the knockoff setting, we replace $p$ by $2p$ and $\theta_0$ by $[\theta_0;{{\bf0}}_{p\times 1}]$ because of the additional knockoff variables. Recall the feature statistics from \eqref{eq:W-feat-stat}, $W_j = f(\hth^u_j)-f(\hth^u_{j+p})$, $j\in[p]$, with $\hth^u$ given by~\eqref{eq:lasso-debias}. As discussed in Sections~\ref{sec:feature-statistics} and \ref{sec:FDR-control}, the FDR control holds for general $f$. However, to obtain significant power in detecting non-null variables, we need to make assumptions so that a large value of $W_j$ provides evidence against the null hypothesis that $\theta_{0,j}=0$. Furthermore, since $f$ could depend on parameters $n$, $p$, etc., we introduce a second argument $a$ to encode this dependence. 
\begin{assumption}\label{ass:f-feat-stat}
    Given $\hth^u$ given by~\eqref{eq:lasso-debias}, for $j\in[p]$, the feature statistics are defined as $W_j:=f(\hth^u_j;a)-f(\hth^u_{j+p};a)$, for some $f:\mbR^2\to\mbR_{\ge 0}$. For any $a\in \mbR$, $f(\cdot;a)$ satisfies the following
\begin{enumerate}
    \item $f(\cdot;a)$ is $L$-Lipschitz for a fixed constant $L>0$.
    \item $k\cdot f(x;a)=f(kx;ka)$ for any $k>0$.
    \item $f(x;a)=f(-x;a)$.
    \item $f(\cdot;a)$ is non-decreasing on $[0,\infty)$.
    \item $f(x;a)\ge c_ax$ when $x\ge c_x$ for some $c_a>0$, $c_x>0$.
\end{enumerate}
\end{assumption}

We next construct an estimator of the false discovery proportion (FDP) as a function of the threshold $t$:
\begin{equation}\label{eq:FDP-hat}
    \widehat{\FDP}(t) = \frac{1+|\{j\in[p]:W_j\leq -t\}|}{|\{j\in[p]:W_j\geq t\}|\vee 1} \,.
\end{equation}
The selection threshold $\wht$ is chosen in a data-dependent way. Namely, given the target False Discovery Rate (FDR) level $q\in(0,1)$, we set:
\[
    \widehat{t} := \min \left\{t\in \mathcal W: \widehat{\FDP}(t)\leq q\right\},\phantom{ss}\hS_0=\left\{j\in[p]:W_j\geq\wht\, \right\}\,,
\]
where $\wht=+\infty$ when the above set is empty and $\mathcal{W}=\{|W_j|:j\in[p]\}\setminus{\{0\}}$. Consequently, the power and FDP for knockoff procedure are data-dependent as well.

To analyze the data-dependent procedure, we first isolate the randomness induced by the data-dependent threshold from the randomness in power. The ultimate goal is to characterize the power at the data-dependent threshold $\wht$. However, because $\wht$ is defined implicitly through the process $\widehat{\mathrm{FDP}}(t)$, a direct analysis of Power($\,\wht\,$) is cumbersome. We therefore begin with a fixed-threshold analysis: for deterministic $t$, we derive asymptotic characterizations of the power and of the estimated FDP as functions of $t$. These results are then used to control the location of the data-dependent threshold $ \wht$. Finally, substituting this control into the fixed-threshold power characterization yields the desired power analysis for the knockoff procedure with data-dependent threshold.

We use the distributional characterization of the debiased private Lasso to derive analytical predictions for FDR and power as a function of threshold $t$, namely

\begin{equation}\label{eq:FDP_t}
    \FDP(t) = \frac{|\{j\in[p]:\theta_{0,j}=0,W_j\geq t\}|}{|\{j\in[p]:| W_j\geq t\}|}\,;
\end{equation}
\begin{equation}\label{eq:power_t}
    \Power(t) = \frac{|\{j\in[p]:\theta_{0,j}\neq0,W_j\geq t\}|}{s_0}\,.
\end{equation}
It is worth noting that $\widehat{\FDP}(t)$ is an observable random variable, while $\FDP(t)$ is not observable (since $\theta_0$ is unknown).

In the asymptotic setting and under the conditions of Theorem \ref{thm:suff-cond-debiasing}, we have $\hth^u \approx \theta_0+\sigma\rho_n Z$, where $Z\sim \mcN({{\bf{0}}},\bI_{2p})$. Since $f$ is known, we have
\begin{align*}
W_j = f(\hth^u_j;a)-f(\hth^u_{j+p};a)\approx f(\theta_{0,j}+\sigma\rho_nZ_j;a)- f(\sigma\rho_nZ_{j+p};a)\,.
\end{align*}
Such approximation helps us to predict the power and FDP of a selection even when the power and FDP are not observable. The next theorem formally provides these predictions by characterizing the asymptotic behavior of the power and FDP. By assuming $p,s_0,r,\epsilon,\delta$ are all functions of $n$, we have the following result.
\begin{theorem}\label{thm:utility-limit}
    Consider the conditions of Theorem~\ref{thm:suff-cond-debiasing}. Recall $\rho_n$ from~\eqref{eq:rho_n}. Assume $s_0(n)/p(n)\to c_0$ for some constant $c_0\in(0,1)$, and for all $j\in S_0$, $\theta_{0,j}\geq \mu_n$ for some positive sequence $\{\mu_n\}_{n\in\mathbb N}$ that $\lim_{n\to\infty}\mu_n/(\sigma\rho_n)=\mu_0$. Suppose the test statistics $W_j = f(\hth^u_j;\sigma\rho_n)-f(\hth^u_{j+p};\sigma\rho_n)$ for some $f$ satisfies Assumption~\ref{ass:f-feat-stat}. Then for any constant $t_0>0$, the following bounds hold true with high probability:
    \begin{align}\label{eq:FDP-hat-limit}
        &\underset{n\to\infty}{{\rm limsup}}\;\widehat{\FDP}(\sigma\rho_nt_0) \leq \widehat{\alpha}(\mu_0,t_0)\,, \text{ where } \\
        &\widehat{\alpha}(\mu_0,t_0):= \frac{c_0\Prob(f(\mu_0+Z;1)-f(Z';1)\leq -t_0)+(1-c_0)\Prob(f(Z;1)-f(Z';1)\leq -t_0)}{c_0\Prob(f(\mu_0+Z;1)-f(Z';1)\geq t_0)+(1-c_0)\Prob(f(Z;1)-f(Z';1)\geq t_0)}\nonumber\,.
    \end{align}
    In addition,
    \begin{align}\label{eq:FDP-limit}
        &\underset{n\to\infty}{{\rm limsup}}\;\FDP(\sigma\rho_nt_0) \leq \alpha(\mu_0,t_0), \text{ where }\\
        &\alpha(\mu_0,t_0):=  \frac{(1-c_0)\Prob(f(Z;1)-f(Z';1)\geq t_0)}{c_0\Prob(f(\mu_0+Z;1)-f(Z';1)\geq t_0)+(1-c_0)\Prob(f(Z;1)-f(Z';1)\geq t_0)}\nonumber\,.
    \end{align}
    For the power, we have 
    \begin{equation}\label{eq:power-limit}
        \underset{n\to\infty}{{\rm liminf}}\;\Power(\sigma\rho_nt_0) \geq \beta(\mu_0,t_0),\quad\quad 
        \beta(\mu_0,t_0):= \Prob(f(\mu_0+Z;1)-f(Z';1)\geq t_0)\,,
    \end{equation}
    where $Z,Z'$ are independent Gaussian $\mcN(0,1)$.
\end{theorem}
The proof of Theorem \ref{thm:utility-limit} is given in \ref{pf:utility-limit}. Note that $\alpha(\mu_0,t_0)$, $\widehat{\alpha}(\mu_0,t_0)$ and $\beta(\mu_0,t_0)$ are deterministic functions, where in the notation we made the dependence on the signal strength $\mu_0$, the sparsity level $c_0$ and the threshold $t_0$ explicit. The privacy parameters $(\epsilon, \delta)$ appear in the above bounds through $\rho_n$, given by~\eqref{eq:rho_n}, which involves $w$ that depends on $(\epsilon, \delta)$ as in~\eqref{eq:w-2}.

The assumption $s_0/p \to c_0 \in (0,1)$ in Theorem~\ref{thm:utility-limit} is not part of the debiasing requirement in Theorem~\ref{thm:suff-cond-debiasing}. The debiasing conditions become weaker when the model is sparser. The quantities involving $\|\theta_0\|_1 \le \sqrt{s_0}$, and the Lasso error bound all decrease as $s_0$ decreases. Thus, allowing $s_0/p \to 0$ does not create a problem for the debiasing step; if anything, it makes the debiasing conditions easier to satisfy. The reason Theorem~\ref{thm:utility-limit} assumes $s_0/p \to c_0 \in (0,1)$ is instead to obtain non-degenerate fixed-threshold limits for FDP and $\widehat{\FDP}$. In the expressions for $\alpha(\mu_0,t_0)$ and $\widehat{\alpha}(\mu_0,t_0)$, the signal contribution is weighted by $c_0$, while the null contribution is weighted by $1-c_0$. If $c_0=0$, the signal term disappears from the denominator, and by the symmetry of the null knockoff statistic the upper bounds reduce to 1, making them trivial. Therefore, the constant-sparsity-proportion assumption in Theorem~\ref{thm:utility-limit} is imposed for the FDP and $\widehat{\FDP}$ asymptotics, not for debiasing. When $s_0/p \to 0$, one would need a different, more refined analysis, to obtain nontrivial FDP bounds.

Corollary~\ref{cor:power1} discusses conditions under which the knockoff procedure can achieve asymptotic power one and additional conditions to achieve zero false discovery proportion. 

\begin{corollary}\label{cor:power1}
Consider the conditions of Theorem~\ref{thm:suff-cond-debiasing}. Recall $\rho_n$ from~\eqref{eq:rho_n}. Assume $s_0(n)/p(n)\to c_0$ for some constant $c_0\in(0,1)$, and for all $j\in S_0$, $|\theta_{0,j}|\geq \mu_n$ for some positive sequence $\{\mu_n\}_{n\in\mathbb N}$. Suppose the test statistics $W_j = f(\hth^u_j;\sigma\rho_n)-f(\hth^u_{j+p};\sigma\rho_n)$ for some $f$ satisfies the conditions in Assumption~\ref{ass:f-feat-stat}. If $\liminf_{n\to\infty}\mu_n/(\sigma\rho_n)=\infty$, we have the following:  
\begin{enumerate}
    \item For any $q\in(0,1)$, the sequence of selection threshold is defined as 
    \[t_n:=\min\{t\in\mathcal{W}:\widehat{\FDP}(t)\le q\}\,,\] 
    where $\wht=+\infty$ when the above set is empty and $\mathcal{W}=\{|W_j|:j\in[p]\}\setminus{\{0\}}$. Then $t_n$ satisfies that 
    \begin{enumerate}
        \item For any $n$, $\FDR(t_n)\le q$.
        \item $\lim_{n\to\infty} \Power(t_n) =1$
    \end{enumerate}
    \item For any sequence of thresholds $\{t_n\}_{n\in \mathbb{N}}$ such that \[\liminf_{n\to\infty}\frac{t_n}{\sigma\rho_n} = \infty\quad \text{ and }\quad  \lim_{n\to\infty}\frac{t_n}{\mu_n}=0,\]
    \textcolor{black}{with high probability, }we have
    \[
    \lim_{n\to\infty} \FDP(t_n) =0,\quad\quad 
    \lim_{n\to\infty} \Power(t_n) =1\,.
    \]
\end{enumerate}
\end{corollary}

The proof of Corollary~\ref{cor:power1} is given in Appendix~\ref{pf:power1}. Note that the selection threshold $t_n$ is data-dependent, which is designed to control the FDR under a pre-determined level $q$, as shown by Theorem~\ref{thm:FDR}.

Lastly, we consider a special case of the feature statistics class---Lasso Coefficient Difference (LCD). Let $\hth^* $ be the private Lasso estimate defined in \eqref{eq:lasso-JL}. Denoted by $\eta$ the mapping recovering the private Lasso estimate $\hth^*$ from the debiased private Lasso $\hth^u$ in Lemma~\ref{lem:function-of-debiased}. Namely, for any $j\in[2p]$,
\[
\hth^*_j = \eta(\hth^u_j;\lambda):=\frac{1}{1+ \frac{w^2}{n}}\cdot\sign(\hth^u_j)\cdot\left(|\hth^u_j|-\lambda\right)_+\,.
\]
Recall that under the setting of Theorem~\ref{thm:suff-cond-debiasing}, $\lambda$ is of the order $\sigma\max\{\sqrt{\log p/n},\sqrt{(\log p)^3/r}\}$. The following lemma shows that under the setting of Theorem~\ref{thm:suff-cond-debiasing}, $\lambda$ dominates $\sigma\rho_n$ asymptotically. The proof of Lemma~\ref{lem:lambda-bigger-than-rho_n} is given in Appendix~\ref{pf:lambda-bigger-than-rho_n}.
\begin{lemma}\label{lem:lambda-bigger-than-rho_n}
    Consider the conditions of Theorem 4.2. Recall that 
    \[
    \lambda = C_\lambda\cdot\sigma\cdot\max\left\{\sqrt{\frac{\log p}{n}},\sqrt{\frac{(\log p)^3}{r}}\right\}\,,\quad\rho_n^2 = \frac{n+r+1}{n r}+\frac{w^2}{nr}\,.
    \]
    for some constant $C_\lambda>0$. Assume $ \lim_{n\to\infty}s_0(n)/p(n) = c_0$ for some $c_0\in(0,1)$ and $\sigma=\Theta(1)$, then we have $w^2=o(n)$ and $\sigma\rho_n=o(\lambda)$.
\end{lemma}
As discussed in the lemma above, $\lambda$ does not scale with $\sigma\rho_n$, hence $\eta(\cdot; \lambda)$ does not fit into the regime of Theorem~\ref{thm:utility-limit} where the function is of the form $f(\hth^u_j;\sigma\rho_n)$. Nonetheless, in the next corollary we show that by adopting the proof of Theorem~\ref{thm:utility-limit}, one can obtain sufficient conditions for the power of the test to reach asymptotically perfect power and zero FDP.

\begin{corollary}\label{cor:utility-LCD}
    Consider the conditions of Theorem 4.2. Assume $s_0(n)/p(n)\to c_0$ for some constant $c_0\in(0,1)$, and for all $j\in S_0$, $\theta_{0,j}\geq \mu_n$ for some positive sequence $\{\mu_n\}_{n\in\mathbb N}$. Suppose the test statistics $W_j = |\hth^*_j|-|\hth^*_{j+p}|$ where $\hth^*$ is the private Lasso estimate defined in \eqref{eq:lasso-JL}. For any positive sequence $\{t_n\}_{n\in\mathbb N}$ such that $\liminf_{n\to \infty} t_n/(\sigma\rho_n)>0$, then
    \begin{align*}
        \lim_{n\to\infty}\widehat \FDP(t_n)=0\,,\quad\lim_{n\to\infty}\FDP(t_n) = 0\,.
    \end{align*}
    In addition, if there exists $\varepsilon>0$ such that $ \liminf_{n\to\infty}\frac{\mu_n-(1+\varepsilon)t_n}{\lambda}>1$,
    \[
    \lim_{n\to\infty}\Power(t_n) = 1
    \]
\end{corollary}
The proof of Corollary~\ref{cor:utility-LCD} is given in~\ref{pf:utility-LCD}.

\subsection{Examples: Sufficient Conditions}\label{sec:suff-cond-eg}
Here we provide some examples of the parameter settings that satisfy the sufficient conditions described in Theorem \ref{thm:suff-cond-debiasing}. Recall from Algorithm~\ref{algo:JLT} that $w^2 = \frac{4B^2}{\epsilon}\left(\sqrt{2r\log(4/\delta)}+\log(4/\delta)\right) = O\left(\frac{p\sqrt{r\log(1/\delta)}}{\epsilon}\right)$, assuming $r>\sqrt{\log(4/\delta)}$. We provide the validation of the sufficient conditions for the examples in Appendix~\ref{pf:examples}.
\begin{example}
For high-dimensional setting, consider $n=p=r$, $\epsilon = p^{3/4}\sqrt{\log p}$, $\delta = p^{-2}$, $\sqrt{s_0} = o\left(\frac{p^{1/8}}{\log p}\right)$, $\sigma=1$. 
\end{example}

\begin{example} 
For the case when $s_0 $ is of the same order as $p$, let $n=p^3(\log p)^6$, $r=p^2(\log p)^5$, $\epsilon=1$, $\delta=p^{-4}$, $s_0 = c_0\cdot p$ for some $c_0\in(0,1)$, $\sigma=1$. 
    
\end{example}

In the first example, we have $n$ and $p$ of the same order and the privacy loss $\epsilon = \tilde{O}(p^{3/4})$. In the second example, $n = \tilde{O}(p^3)$ and $\epsilon = O(1)$. 
As discussed in Section~\ref{sec:input-output}, the minimax bound for linear regression under local differential privacy constraint reduces the effective sample size from $n$ to $\epsilon^2 n/p$~\citep{duchiminimax}. This suggests that in order to achieve descent performance in (local) differentially private setting, either $\epsilon$ should be set large or $n\gg p$. 

We next study a more general class of parameter regimes that satisfy the conditions stated in Theorem~\ref{thm:suff-cond-debiasing}, where we focus on the case of constant sparsity level, namely $s_0 = c_0 p$, for some constant $c_0\in(0,1)$. Consider $n,s_0,r,\epsilon,\delta$ as functions of $p$, and let $\mathcal{P}$ denotes the collection of tuples $(n,p,s_0,r,\epsilon,\delta)$ such that the conditions in Theorem \ref{thm:suff-cond-debiasing} are satisfied. Suppose $s_0 = c_0p$, for some constant $c_0\in (0,1)$. What are the proper orders of $n,\epsilon$, for which there exists a choice of $r$ as a function of $p$ such that the tuple $(n,p,c_0p,r,\epsilon,\delta)\in\mathcal{P}$?

To answer this question, we look at a simpler version of it: let $n=p^\alpha,r=p^\beta,\epsilon=p^\gamma$, find the set of $(\alpha,\gamma)\in\mathbb{R}_+^2$ such that there exists a $\beta$ for which the parameters are in $\mathcal{P}$. We assume $\log(1/\delta)=O(\log(p))$. With this simplification, we can find sufficient conditions in $\alpha, \beta,\gamma$ for each of the conditions in Theorem \ref{thm:suff-cond-debiasing}.
\begin{enumerate}
    \item $\max\{1,2+\beta-2\alpha-2\gamma+\min\{\alpha,\beta\}\}<\max\{0,1+\frac{1}{2}\beta-\alpha-\gamma\}+\min\{\frac{1}{2}\beta,\frac{1}{2}\alpha,\alpha+\gamma-1\}$
    \item $\max\{1-\frac{1}{4}\beta-\frac{1}{2}\alpha-\frac{1}{2}\gamma,\frac{3}{2}-\alpha-\gamma\}<\max\{-\frac{1}{2}\alpha,-\frac{1}{2}\beta\}$
    \item $1+\min\{\alpha,\beta,\alpha+\gamma+\frac{1}{2}\beta-1\}<\min\{2\alpha+2\gamma-2,\frac{1}{2}\beta+\alpha+\gamma-1\}$
    \item 
        $\min\{\frac{1}{2}\alpha,\frac{1}{2}\beta,\frac{1}{2}\alpha+\frac{1}{2}\gamma-\frac{1}{2}+\frac{1}{4}\beta\}+\max\{1-\min\{\alpha,\beta\},2+\beta-2\alpha-2\gamma\}<\\
        -\frac{1}{2}\min\{\alpha,\beta\}+\max\{0,1+\frac{1}{2}\beta-\alpha-\gamma\}+\min\{\frac{1}{2}\alpha,\frac{1}{2}\beta,\alpha+\gamma-1,\frac{1}{4}\beta+\frac{1}{2}(\alpha+\gamma)-\frac{1}{2}\}$
\end{enumerate}
To answer the question we can fix $\beta$ and find the range of $(\alpha,\gamma)$, as presented in Figure \ref{fig:ineq}. Here are some observations:
\begin{enumerate}
    \item We must have $\alpha>2$ and $\beta>2$ for above inequalities to have a nonempty intersection.
    \item If $\alpha+\gamma>3$, then there exists $\beta>2$ such that the above inequalities have a nonempty intersection.
    \item If $\gamma>3/4$, then there exists $\beta>2$ such that the above inequalities have a nonempty intersection.
\end{enumerate}

\begin{figure}[h]
    \centering
    \includegraphics[width=0.95\textwidth]{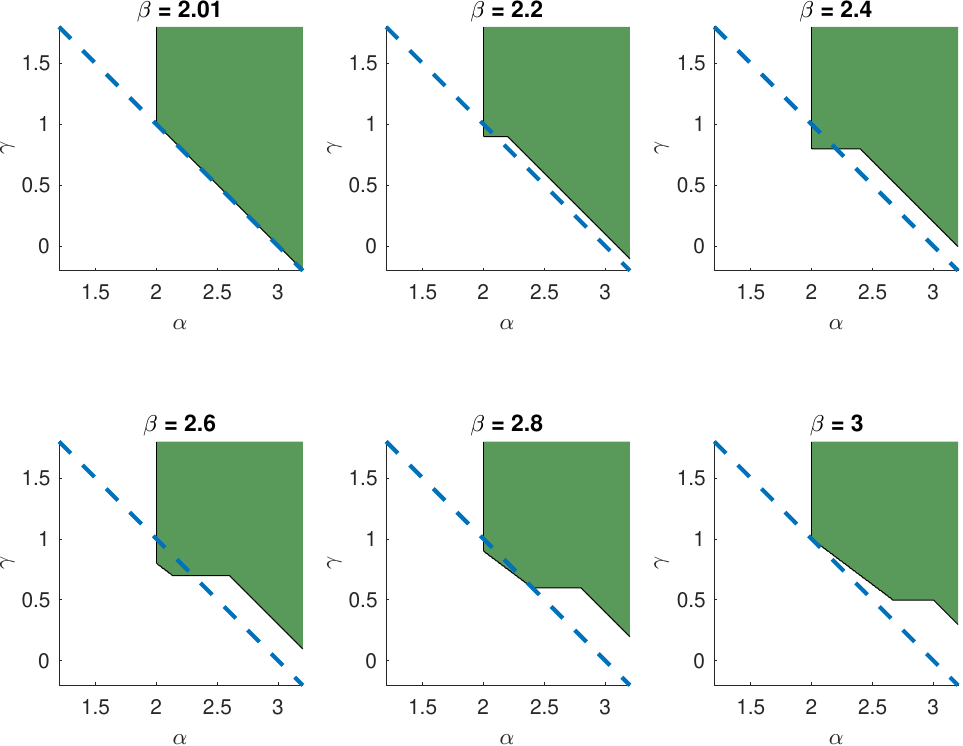} 
    \caption{Sufficient conditions for debiasing. Assume $n=p^\alpha,r=p^\beta,\epsilon=p^\gamma,s_0=c_0p$ for some $c_0\in(0,1)$. The shaded region indicates the choice of $\alpha$ and $\gamma$ such that the debiasing conditions are satisfied for different $\beta$. Dotted line indicates the equality $\alpha+\gamma=3$.}
    \label{fig:ineq}
\end{figure}

With the insights provided by this simplification, we have the following corollary on the choice of parameters, and refer to Appendix~\ref{proof:cor:sample-privacy-tradeoff} for its proof. 
\begin{corollary}\label{cor:sample-privacy-tradeoff}
    If $s_0 = c_0p$ for some $c_0 \in (0,1), n=p^{\alpha},\epsilon=p^\gamma$ for $\alpha>2,\gamma+\alpha>3$, $\log(\delta^{-1})=O(\log(p))$, then there exists $r$ such that $(n,p,s_0,r,\epsilon,\delta)\in\mathcal{P}$. 
    
\end{corollary}

Note that Corollary~\ref{cor:sample-privacy-tradeoff} states a trade-off between sample complexity and the privacy loss, through the condition $\alpha+\gamma>3$. For example, if we want constant privacy loss, i.e. $\epsilon=O(1)$, then we need $n\gg p^3$. If we do not have enough sample size, say $n=O(p^{2.5})$, we must sacrifice the privacy so that $\epsilon\gg p^{0.5}$. The next Corollary \ref{cor:privacy-upper-bound} states that once $\gamma> 3/4$, then $\alpha>2$ suffices and we do not see this trade-off anymore in order to ensure $(n,p,s_0,r,\epsilon,\delta)\in\mathcal{P}$.

\begin{corollary}\label{cor:privacy-upper-bound}
    If $s_0 = c_0p$ for some $c_0 \in (0,1), n=p^{\alpha},\epsilon=p^\gamma$ for $\alpha>2,\gamma>\frac{3}{4}$, $\log(\delta^{-1})=O(\log(p))$, then there exists $r$ such that $(n,p,c_0p,r,\epsilon,\delta)\in\mathcal{P}$. 
\end{corollary}
The proof of this corollary is similar, therefore omitted.

\section{Baseline Method: Gaussian Mechanism}\label{sec:baseline}
To be self-contained,  we include key implementation details for privatizing the second-moment matrix using the Gaussian mechanism \citep{analyzegauss}, though this is not the primary focus of our work. The Gaussian mechanism---a foundational differential privacy technique---ensures algorithmic privacy by adding calibrated Gaussian noise to data.

\begin{theorem}
    (Gaussian mechanism) Given $\epsilon,\delta\in(0,1)$ and a function $f:\mathcal{X}\to\mbR^k$, if $\sigma\geq \sqrt{2\log(1.25/\delta)}\Delta_2f/\epsilon$, then $f+Z$ is $(\epsilon,\delta)$-differentially private, where $Z\sim\mcN(0,\sigma^2\bI_k)$ and $\Delta_2f$ denotes the $\ell_2$-sensitivity of $f$, defined as:
    \[
    \Delta_2f = \max\{\|f(x)-f(x')\|_2:x,x'\text{ differ by at most one entry}\}\,.
    \]
\end{theorem}
In the same spirit, given a data matrix $\bA\in\mbR^{n\times d}$, \cite{analyzegauss} proposed Algorithm (\ref{algo:gauss}) that outputs a private estimate for the second-moment matrix $\bA^\top \bA$ by adding a symmetric Gaussian matrix. The sensitivity of matrix $\bA^\top \bA$ is measured in Frobenius norm where for a matrix $\bA$, the Frobenius norm $\|\bA\|_F = \sqrt{\sum_{i\in[n],j\in[d]}\bA_{ij}^2}$. Assuming that the row norm of $\bA$ is bounded by $B>0$, the $\ell_2$-sensitivity of $\bA^\top\bA$ with respect to replacing one row can be obtained by the following derivation. Assuming a neighboring data matrix $\bA'\in\mbR^{n\times d}$ replaces one row of $\bA\in\mbR^{n\times d}$, $v^\top\in\mbR^d$, by $v'^\top\in\mbR^d$, we have that 
\begin{align*}
    \|\bA^\top \bA - \bA'^\top\bA'\|_F&=\|vv^\top-v'v'^\top\|_F\\
        &= \sqrt{\|v\|_2^4+\|v'\|_2^4-2(v^\top v')^2}\\
        &\le \sqrt{\|v\|_2^4+\|v'\|_2^4}\\
        &\leq \sqrt{2}B^2\,.
\end{align*}
So the Gaussian mechanism suggests that if we add a Gaussian noise with standard deviation greater than or equal to $2\sqrt{\log(1.25/\delta)}B^2/\epsilon$, then the release of the second-moment matrix is $(\epsilon, \delta)$-differentially private. Since the second-moment matrix is symmetric, we only have to generate noise for the upper-triangular part and release the lower triangular entries same as their correspondent upper-triangular entries.

\begin{algorithm}\label{algo:gauss}
    \caption{Gaussian Mechanism for Second-Moment Matrix \citep{analyzegauss}}
    \textbf{Input}: data matrix $\bA\in \mathbb{R}^{n\times d}$ and upper-bound $B>0$ on the $l_2$-norm of every row of $\bA$; privacy parameters $\epsilon>0,\ \delta\in(0,1/e)$.
    \begin{enumerate}
        \item Compute the standard deviation of noise $\sigma = 2\sqrt{\log(1.25/\delta)}B^2/\epsilon$.
        \item Generate a symmetric random matrix $\bZ\in\mbR^{d^2}$ by sampling the upper-triangle entries from $\mcN(0,\sigma^2)$ independently. The lower triangular entries are copies of the corresponding upper-triangular entries, $\bZ_{ij} = \bZ_{ji}$.  
        \item \textbf{Return} $\bA^\top\bA+Z$.
    \end{enumerate}
\end{algorithm}

By rewriting the objective function in \eqref{eq:lasso-knockoff}, it is clear that Lasso solution only depends on the second moment of the data, $[\bX\ \bXT\ y]^\top[\bX\ \bXT \ y]$:
\[
\hth(\lambda) = \arg\min_{\theta\in\mbR^{2p}}\frac{1}{2n}\left(\theta^\top[\bX\ \bXT]^\top [\bX\ \bXT]\theta-2\theta^\top[\bX\ \bXT]^\top y\right)+\lambda\|\theta\|_1\,.
\]
Applying Algorithm~\ref{algo:gauss} by taking $\bA = [\bX\ \bXT\ y]$, the Gaussian Mechanism produces private version of $[\bX\ \bXT]^\top [\bX\ \bXT]$ and $[\bX\ \bXT]^\top y$ as follows:
\[
[\bX\ \bXT\ y]^\top[\bX\ \bXT \ y]+\bZ=
\begin{bmatrix}
    [\bX\ \bXT\ ]^\top[\bX\ \bXT \ ]+\bZ_1 & [\bX\ \bXT]^\top y+\bZ_2\\
    y^\top [\bX\ \bXT]+\bZ_2^\top & y^\top y+\bZ_3
\end{bmatrix}\,,
\]
for symmetric random Gaussian matrix $\bZ\in\mbR^{(2p+1)^2}$, and $\bZ_1\in\mbR^{(2p)^2},\bZ_2\in\mbR^{2p},\bZ_3\in\mbR$ are sub-matrices of $\bZ$ partitioned  based on the dimension of $[\bX\ \bXT]^\top [\bX\ \bXT]$ and $[\bX\ \bXT]^\top y$. Then by replacing the non-private second-moment matrix by the private one, we obtain a private solution for lasso:
\[
\hth^*_G(\lambda) = \arg\min_{\theta\in\mbR^{2p}}\frac{1}{2n}\left(\theta^\top([\bX\ \bXT]^\top [\bX\ \bXT]+\bZ_1)\theta-2\theta^\top([\bX\ \bXT]^\top y+\bZ_2)\right)+\lambda\|\theta\|_1\,,
\]
and the corresponding debiased Lasso estimate:
\[
\hth^u_G(\lambda) = \hth_G^*(\lambda) +\frac{1}{n}\left([\bX\ \bXT]^\top y+\bZ_2-([\bX\ \bXT]^\top[\bX\ \bXT]+\bZ_1)\hth^*(\lambda))\right)+\frac{w^2}{n }\hth_G^*(\lambda) \,.
\]
The remaining parts follow the knockoff procedure detailed in Section~\ref{sec:method-JL}.

\section{Simulation Studies}\label{sec:simulation}

This section presents numerical simulations regarding the JLT-privatized Model-X knockoffs framework. We first validate the power/FDP theoretical prediction regime from Section~\ref{sec:Power-FDR-analysis} by plotting theoretical values against empirical results, simultaneously demonstrating the convergence of prediction error to zero as sample size $n\to \infty$ and revealing an inverse relationship between privacy strength ($\epsilon$) and accuracy. We then systematically quantify the privacy-utility trade-off in Section~\ref{sec:privacy-utility-tradeoff}, through parameter sweeps, establishing operational boundaries for practical deployment. In Section~\ref{sec:compare-Gaussian} we conduct direct performance comparisons between JLT and Gaussian noise injection mechanisms for Model-X knockoff procedure. In Section~\ref{sec:correlated}, we examine how the dependence structure of the design matrix affects the power and FDR. Finally, in Section~\ref{sec:nonDP}, we compare the JLT-privatized knockoff procedure with the non-private Model-X knockoff procedure in terms of power, FDR, and runtime.

Throughout this section, we use \textcolor{black}{independent covariates except in Section~\ref{sec:correlated}} and Lasso Coefficient Difference (LCD) as the feature statistics. Concretely, for each $j\in[p]$, 
\[
W_j := |\hth^*_j|-|\hth^*_{j+p}|\,,
\]
where $\hth^*$ is the private Lasso solution defined as in \eqref{eq:lasso-JL}. The regularization parameter $\lambda$ in the Lasso estimator is set by cross-validation on the privatized data, and so by post-processing property of differential privacy does not affect the privacy guarantee. By Lemma~\ref{lem:function-of-debiased}, the private Lasso estimate can be viewed as a coordinate-wise function of the debiased private Lasso estimate. As discussed in Corollary~\ref{cor:utility-LCD}, the probability distribution of LCD can be characterized through such mapping. Namely, let $\hth^u$ be the debiased private Lasso estimate defined in \eqref{eq:lasso-debias}, then for $j\in[2p]$,
\[
|\hth^*_j| = f(\hth^u_j;\lambda):= \frac{1}{1+ \frac{w^2}{n}}\left(|\hth^u_j|-\lambda\right)_+\approx\frac{1}{1+ \frac{w^2}{n}}\left(|\theta_{0,j}+\sigma\rho_nZ|-\lambda\right)_+\,,
\]
where $Z\sim\mcN(0,1)$.

\subsection{Evaluating Theoretical Predictions with Numerical Experiments}\label{sec:thmVsim}
In Section~\ref{sec:Power-FDR-analysis} we derived theoretical prediction of the power and FDP for the JLT-privatized knockoffs procedure. In this section, we conduct numerical simulation to corroborate our theoretical curves. First, as discussed in Section~\ref{sec:Power-FDR-analysis}, equation \eqref{eq:FDP_t} and \eqref{eq:power_t}, we consider power and FDP as functions of selection threshold $t$. 
Recall that in the actual knockoff procedure, the selection threshold $\wht>0$ is data-dependent. Namely $\wht$ equals the smallest value in the nonzero $|W_j|$'s such that $\widehat{\FDP}(\wht)\leq q$. Theorem \ref{thm:utility-limit} provides a prediction formula (\ref{eq:FDP-hat-limit}) for $\widehat{\FDP}(t)$. Therefore, to estimate $\hat t$, we can solve the following equation for $t$:
\begin{equation}\label{eq:solve-t}
    \frac{c_0\Prob(f(\mu+Z;\lambda)-f(Z';\lambda)\leq -t)+(1-c_0)\Prob(f(Z;\lambda)-f(Z';\lambda)\leq -t)}{c_0\Prob(f(\mu+Z;\lambda)-f(Z';\lambda)\geq t)+(1-c_0)\Prob(f(Z;\lambda)-f(Z';\lambda)\geq t)}=q\,.
\end{equation}
We compute the solution to above equation by estimating the probabilities involved via simulation of Gaussian variables $Z, Z'\sim\mcN(0,\sigma^2\rho_n^2)$ independently. 

\begin{figure}
        \centering
        \includegraphics[width=1\linewidth]{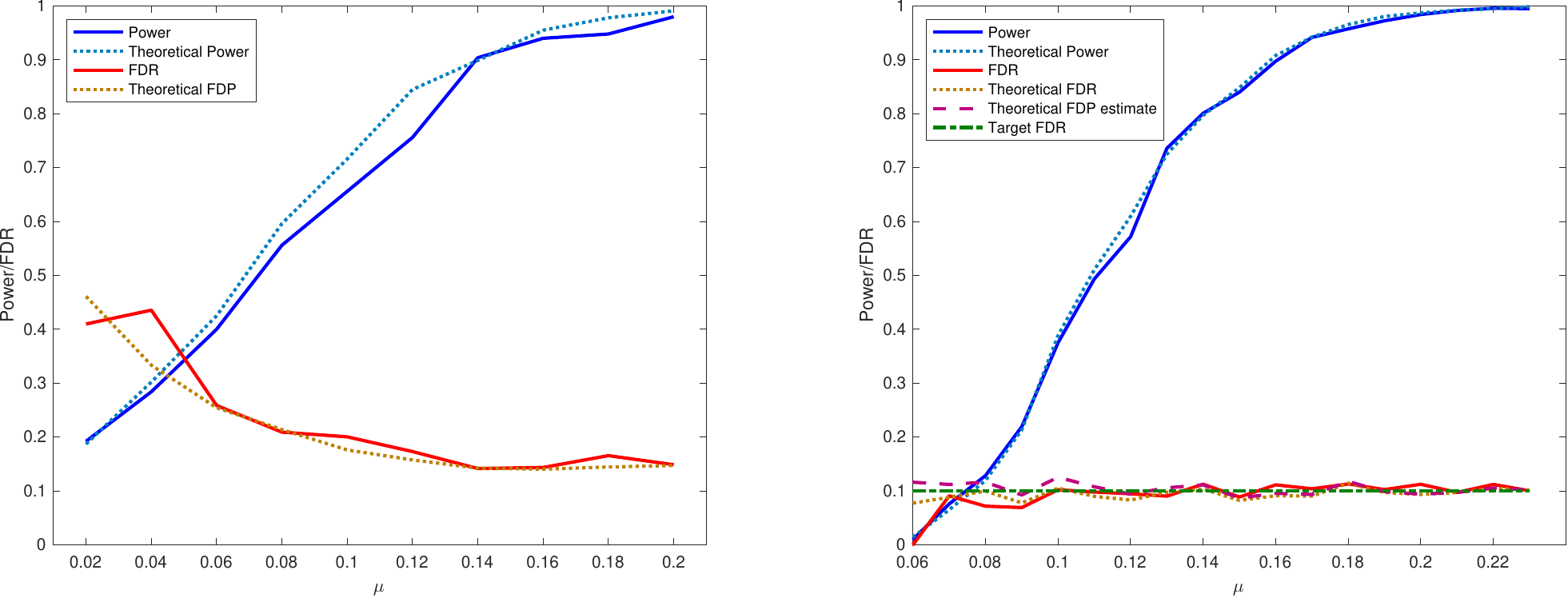}
        \caption{Signal Magnitude ($\mu$) v.s. (Theoretical) Power/FDR. Left plot is fixed-t. Right plot is estimated-t. Parameter used: $n=10^6, p = 50, s_0 = 12,\sigma^2=1, r =1500,\epsilon =1, \delta =0.01, \lambda = 0.03$, $t = 0.025$ (for fixed-t). Power and FDR are estimated by averaging over 100 knockoff-procedure iterations.}
        \label{fig:fixed_theory_threshold}
\end{figure}

In the first experiment, we investigate the theoretical prediction with a fixed value of $t$, whose result is shown on the left of Figure~\ref{fig:fixed_theory_threshold}. Then we run a similar experiment with $t$ chosen by solving \eqref{eq:solve-t}, whose result is shown on the right of Figure~\ref{fig:fixed_theory_threshold}. For simplicity, we let the non-zero coordinates of feature vector $\theta_0$ to have a constant value, $\theta_{0,j}=\mu$, for all $j\in S_0$. The theoretical prediction is evaluated by plotting power/FDR against the magnitude of $\mu$. As Figure \ref{fig:fixed_theory_threshold} reveals, with proper choice of parameters, the theoretical prediction is consistent with the simulation results for both fixed $t$ and estimated $t\approx\wht$. Note that for both simulations, the threshold $t$ is not data-dependent, hence different from the actual knockoff procedure. In the next experiment, we evaluate the accuracy of power/FDR prediction using the solution for $t$ to (\ref{eq:solve-t}) as the selection threshold.

\begin{figure}[h]
        \centering  
        \includegraphics[width=1\linewidth]{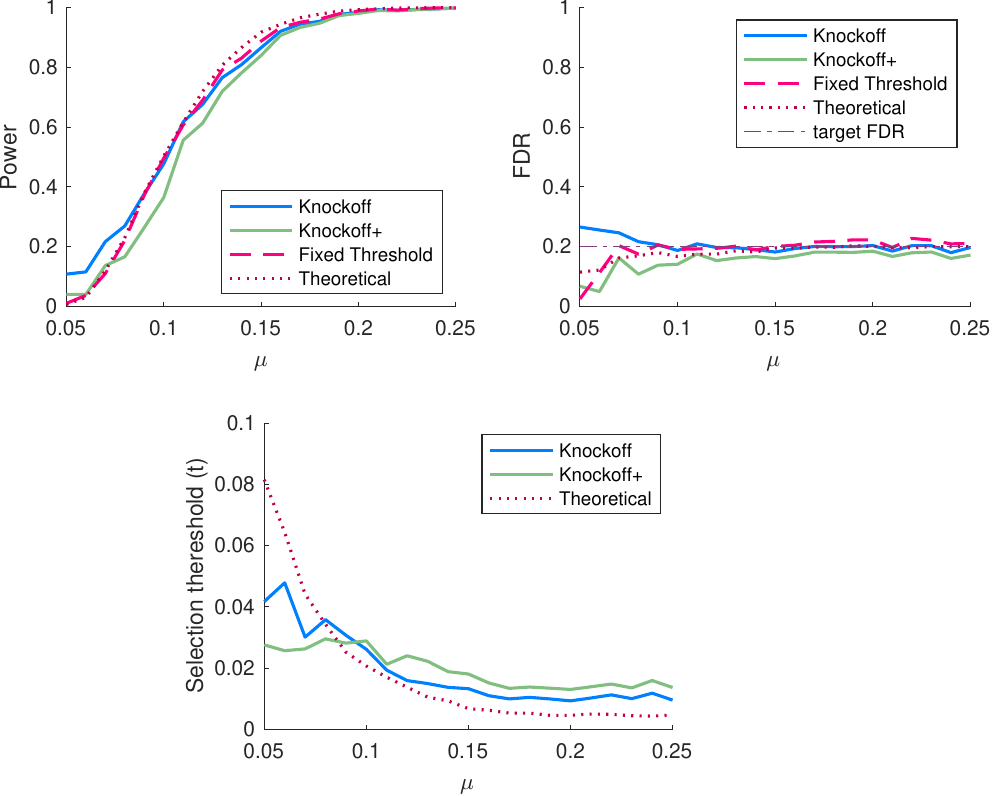}
        \caption{First row: Power(left)/FDR(right) against signal magnitude ($\mu$). Second row: average selection threshold ($t$) against signal magnitude ($\mu$). Parameter used: $n = 10^6, p = 100, s_0 = 25,\sigma^2=1, r =1500, \epsilon=1, \delta =0.01, \lambda = 0.025$, $q=0.2$. The power values of knockoff, knockoff+, and fixed threshold methods are estimated by averaging over 100 knockoff-procedure iterations.}
        \label{fig:data-t}
\end{figure}
\begin{figure}[h]
        \centering  
        \includegraphics[width=0.75\linewidth]{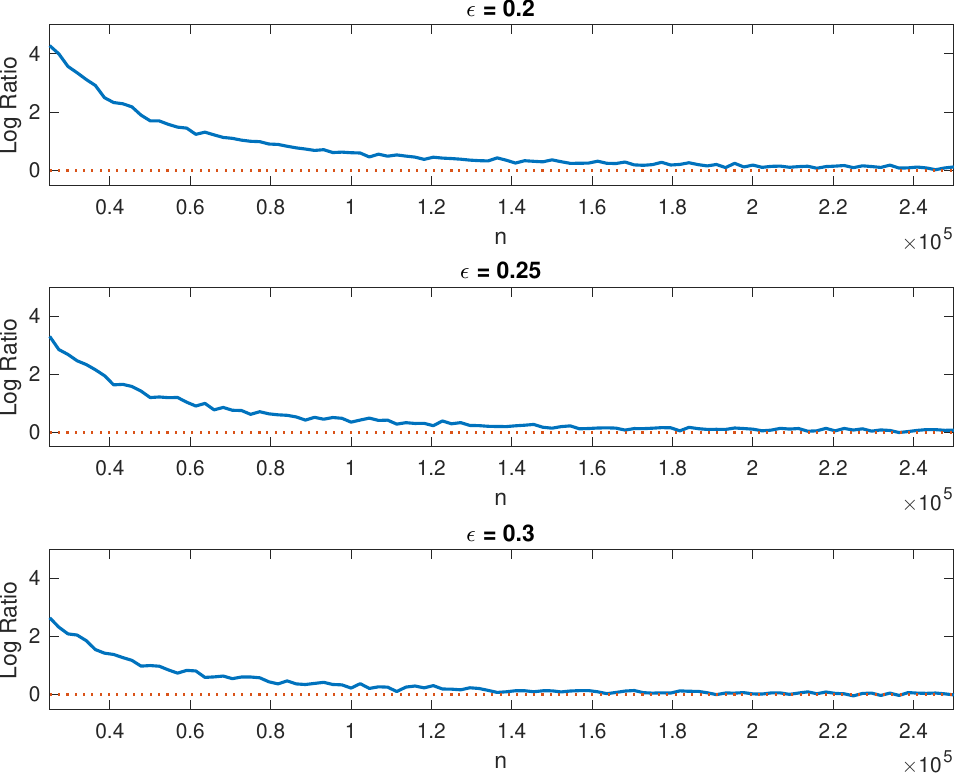}
        \caption{The log ratio of experiment power to theoretical power. Parameter used: $p = 80, s_0 = 20,\sigma^2=1, r =1000, \mu = 0.13, \delta =0.01, \lambda = 0.03$. Power and FDR are estimated by averaging over 150 knockoff-procedure iterations. Theoretical values are estimated by Monte-Carlo Methods. }
        \label{fig:error-conv}
\end{figure}
\begin{figure}[h]
        \centering
        \includegraphics[width=0.9\linewidth]{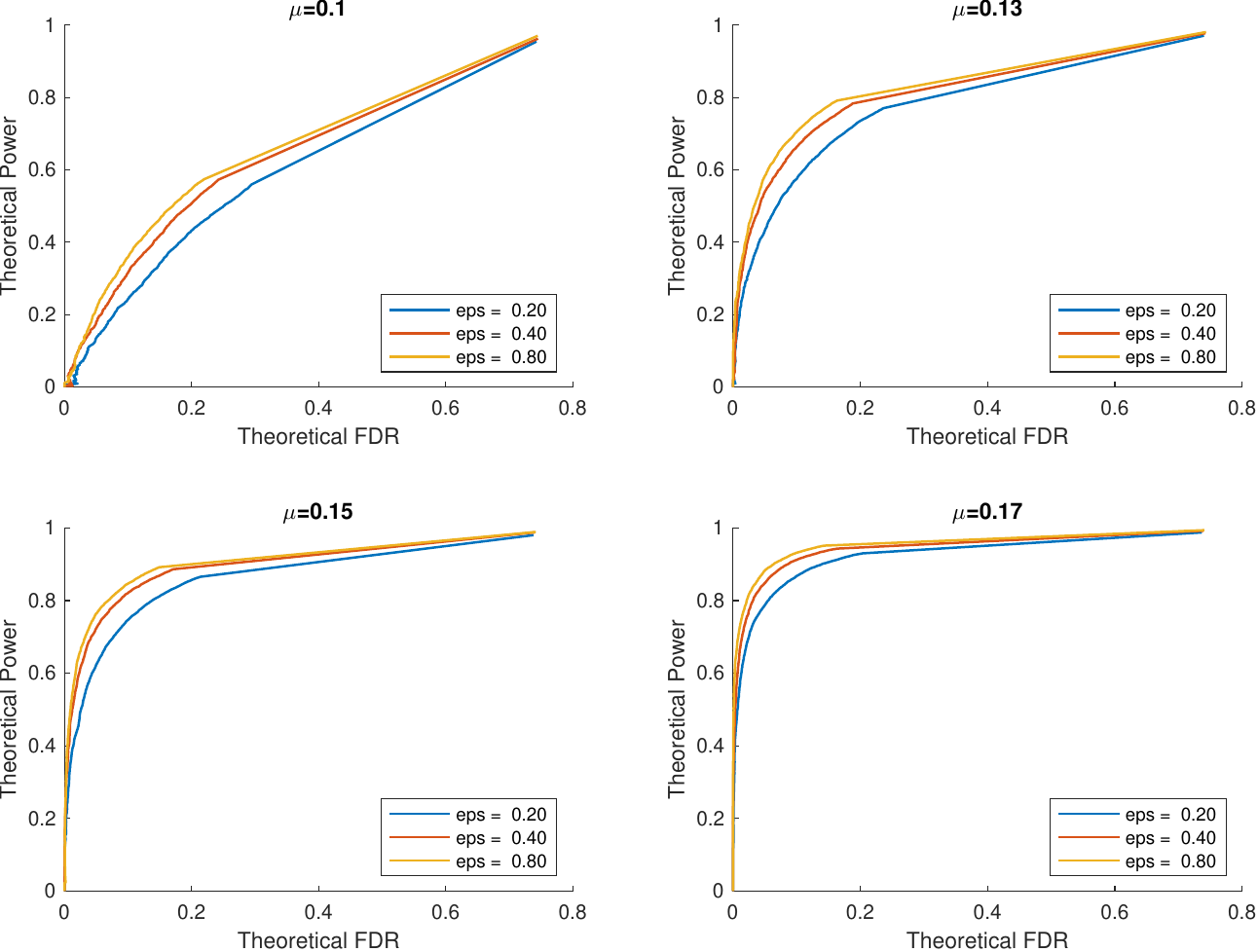}
        \caption{Power/FDR trad-eoff for different privacy loss $\epsilon$. Parameters used: $n=10^6$, $p = 80$, $s_0 = 20$, $\sigma^2=1$, $r =1500$, $\delta =0.01$, $\lambda = 0.03$. }
        \label{fig:tradeoff}
\end{figure}

The theoretical estimation provided in (\ref{eq:power-limit}) and (\ref{eq:FDP-limit}) does not access data, therefore the selection threshold is estimated by solving (\ref{eq:solve-t}) for $t$. We verify the consistency of this regime by comparing it with simulations of Model-X knockoffs procedure. Note that, as introduced in \cite{candes2017panning}, there are two definitions of $\widehat{\FDP}$, one leads to exact FDR control, which is same as the present work uses (\ref{eq:FDP-hat}), the other one is less conservative. In particular, it does not add an extra count to the estimation of false discoveries:
\begin{equation}\label{eq:FDP-hat0}
    \widehat{\FDP}_0(t) = \frac{|\{j:W_j\leq -t\}|}{|\{j:W_j\geq t\}|\vee 1} \,.
\end{equation}
With such choice, the selection threshold will be smaller to selection more features. Following original names, we use knockoff$_+$ and knockoff to indicate whether an extra count is add to the estimation of the false discoveries or not, respectively. We plot the results using both of the methods, along with the theoretical estimation in Figure~\ref{fig:data-t}. Meanwhile, the plot on the second row of Figure~\ref{fig:data-t} shows the value of selection threshold as $\mu$ increases. We see that the theoretical prediction shows promising precision for the given set of parameters.  

In the third experiment, we investigate the precision of the theoretical estimation responding to the change in sample size and privacy. We measure the precision of the estimation by log ratio of simulation power to theoretical power:
\[
\log\left(\frac{\Power(\text{experiment})}{\Power(\text{theory})}\right)\,.
\]
As Figure~\ref{fig:error-conv} shows, the quality of theoretical prediction increases as sample size increases and $\epsilon$ increases. There is a clear trend that the prediction error converges to zero as sample size $n\to\infty$, at a faster rate for larger $\epsilon$.

\subsection{Privacy-Utility Trade-off}\label{sec:privacy-utility-tradeoff}
In this section, we study the privacy-utility trade-off using our theoretical prediction of the power. We explore the privacy-power trade-off by plotting power against the False Discovery Rate (FDR). Regardless of privacy constraints, there is an inherent trade-off between power and FDR: minimizing FDR can be achieved by selecting nothing, while maximizing power requires selecting everything. An ideal power-FDR curve exhibits high power at a relatively low FDR. Figure~\ref{fig:tradeoff} visualizes the trade-off between privacy and power. As $\epsilon$ increases (indicating weaker privacy), the curve shifts toward the top-left corner, meaning greater power is achieved at the same FDR level. We also observe that this improvement is not linear—loosening the privacy constraint yields diminishing returns in power gains as privacy decreases.

\subsection{Comparison with the Gaussian Mechanism}\label{sec:compare-Gaussian}
 In this section, we compare JLT to Gaussian mechanism on model-X knockoffs. As mentioned in Section~\ref{sec:prelim-jlt}, one major shortcoming of the Gaussian mechanism is that it does not guarantee the output matrix will be positive semi-definite. A non-PSD estimate of the second moment matrix can cause iterative methods, such as those used in solving the Lasso, to fail to converge. This issue is illustrated in Figure~\ref{fig:compare}, where we observe a smooth, steadily increasing power curve as the sample size grows for the JLT method, in contrast to the abrupt jump in power seen with the Gaussian mechanism. This behavior aligns with our analysis: when the sample size is small, the minimum singular value of the second moment matrix $\bX^\top\bX$ is small relative to the magnitude of the noise $\mathbf{Z}$. As a result, the privatized second moment matrix $\bX^\top\bX + \mathbf{Z}$ is almost always non-PSD in this regime. Consequently, the Lasso algorithm fails to converge until the sample size is sufficiently large for the privatized second moment matrix to become PSD. Furthermore, even if we project the matrix $\bX^\top\bX + \mathbf{Z}$ to the PSD set, it still underperforms compared to JLT, as shown in Figure~\ref{fig:compare}. 

 \begin{figure}[]
    \centering
    \includegraphics[width=0.9\textwidth]{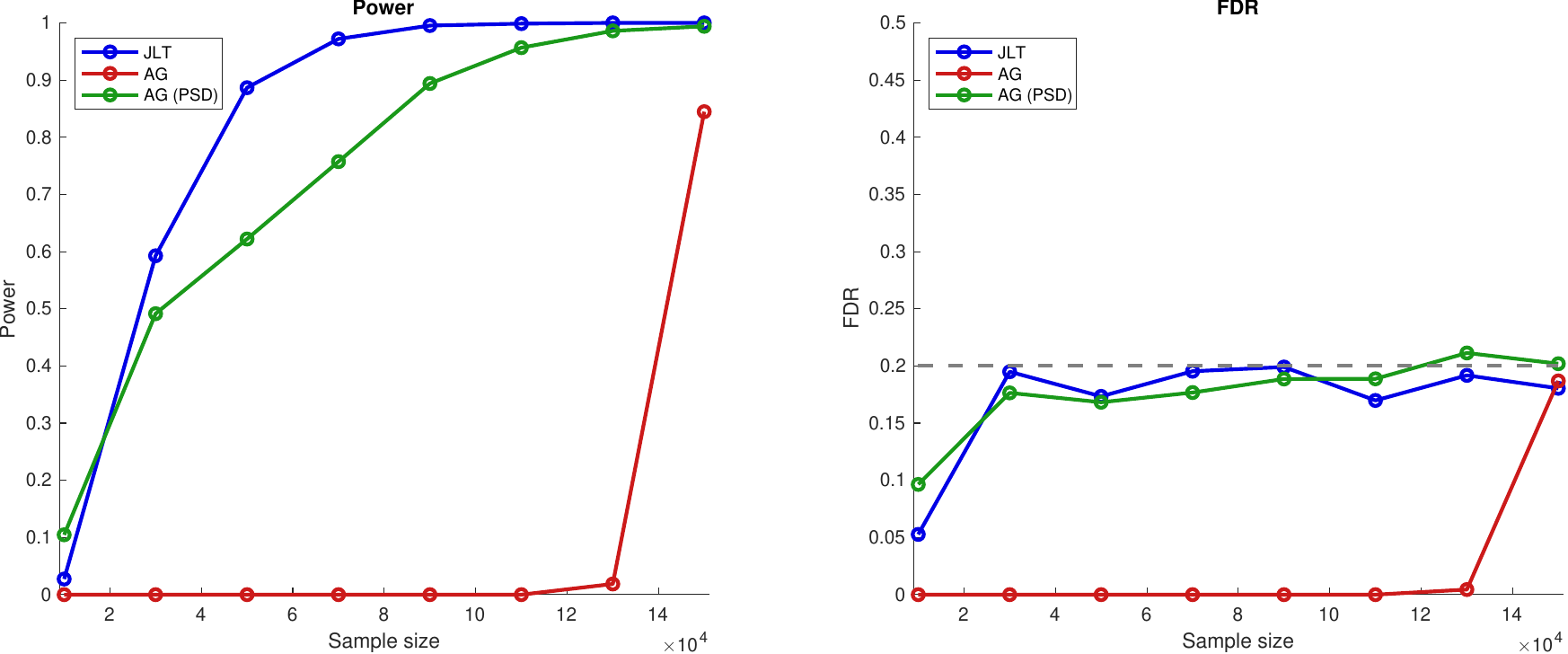} 
    \caption{Power and FDR comparison: Johnson--Lindenstrauss Transform vs Gaussian Mechanism. Legend explanation: JLT is the power/FDR of JLT-privatized knockoffs; AG uses Gaussian Mechanism on the second moment matrices $\bX^\top \bX$ and $\bX^\top y$; AG (PSD) projects the privatized matrix to the PSD set. Parameter used: $\mu=2,p = 50, s_0=15,r=10^4,\epsilon=0.2,\delta=0.01, \lambda =0.025, q =0.2$.}
    \label{fig:compare}
\end{figure}

\subsection{Correlated Design}\label{sec:correlated}

We next examine how the dependence structure of the design matrix affects the finite-sample performance of the proposed procedure. Specifically, we generate the covariates from a correlated Gaussian design with covariance matrix
\[
    \mathrm{Cov}(X_i,X_j) = \rho^{|i-j|},
\]
where $\rho \in \{0,0.2,0.4,0.6\}$. This autoregressive covariance structure creates increasingly strong local dependence among nearby covariates as $\rho$ grows. For each value of $\rho$, we vary the signal magnitude and report the empirical power and FDR.

The left panel of Figure~\ref{fig:compare_cov} shows that correlation has a substantial impact on power. When the design is independent or only weakly correlated, corresponding to $\rho=0$ and $\rho=0.2$, the power increases rapidly with the signal magnitude and approaches one for strong signals. For moderate correlation, $\rho=0.4$, the power still improves as the signal becomes stronger, but the curve is uniformly lower and appears to plateau below the weakly correlated cases. Under stronger correlation, $\rho=0.6$, even for large signal magnitudes, the procedure only detects a relatively small fraction of the true signals.

This behavior is expected. As the correlation among covariates increases, it becomes harder to distinguish the contribution of an active variable from that of its correlated neighbors. Consequently, the knockoff statistics become less separated between null and non-null variables, leading to fewer discoveries at the same FDR target. In other words, stronger design correlation reduces the effective signal strength available for variable selection.

The right panel of Figure~\ref{fig:compare_cov} shows that the empirical FDR remains close to the nominal target level across all values of $\rho$ and signal magnitudes. Thus, the main effect of stronger correlation is a reduction in power rather than a loss of FDR control. These results suggest that the procedure remains conservative or approximately calibrated under correlated designs, but its ability to recover true signals can degrade substantially when the design becomes highly collinear.

\begin{figure}[]
    \centering
    \includegraphics[width=0.9\textwidth]{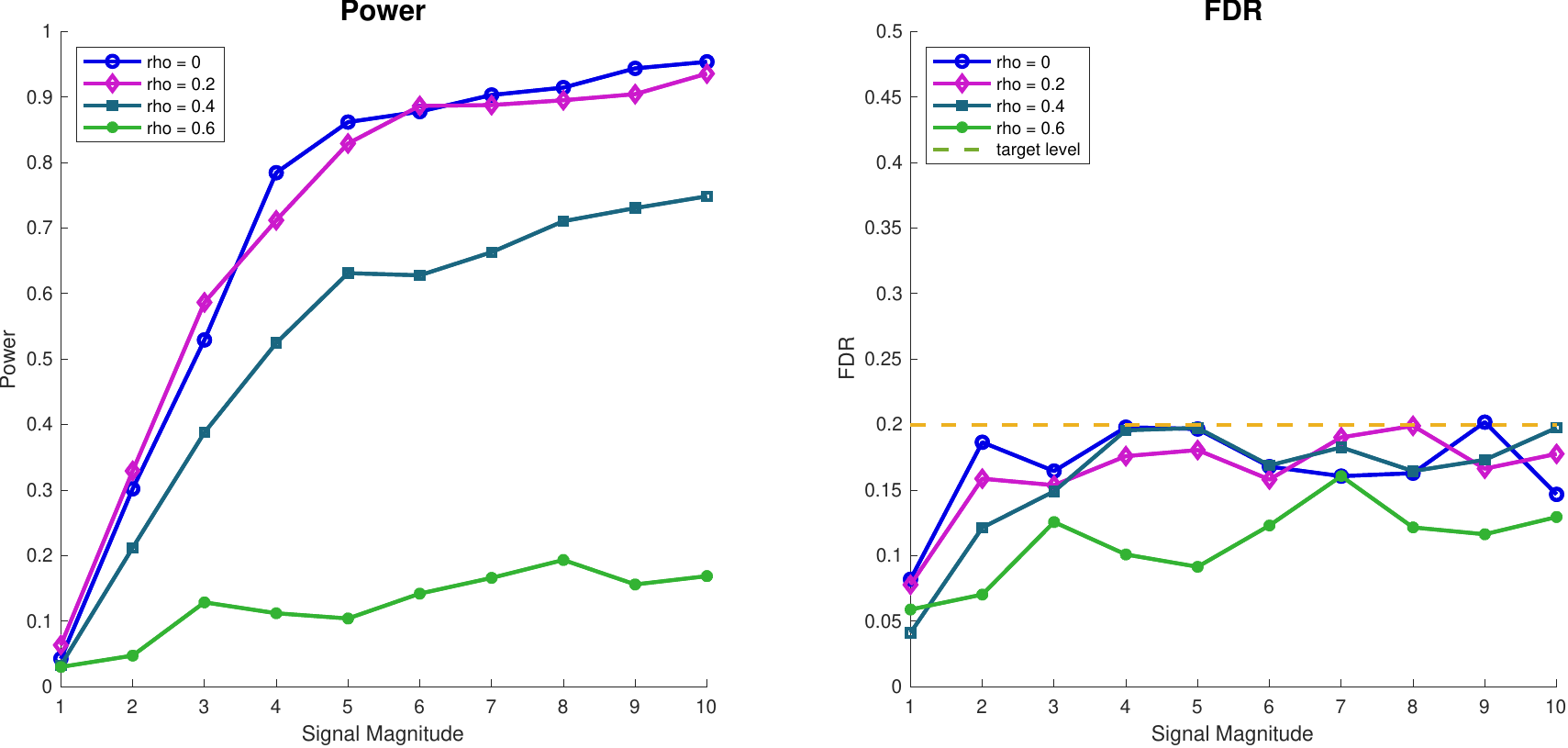} 
    \caption{Power and FDR comparison for different levels of correlated design. Parameter used: $n=2500,p = 50, s_0=15, r=10^5,\epsilon=1,\delta=0.01, \lambda =0.025, q =0.2$.}
    \label{fig:compare_cov}
\end{figure}

\subsection{Comparison with Non-Private Model-X Knockoffs and Runtime Analysis}\label{sec:nonDP}

We compare the JLT-privatized knockoff procedure with the non-private Model-X knockoff procedure in terms of both statistical accuracy and computational cost. The non-private method serves as a benchmark in which the knockoff statistics are computed directly from the original augmented design \([X\,\widetilde X]\) and response \(y\), whereas the private method first applies the Johnson--Lindenstrauss transform to the augmented data matrix \([X\ \widetilde X\ y]\) and proceeds with privatized data.

Figure~\ref{fig:compare_nonDP} compares power and FDR as the signal magnitude varies. As expected, the non-private procedure achieves higher power because it does not suffer from the additional random projection and privacy-induced noise. Nevertheless, the private procedure still attains high power once the signal is sufficiently large. Moreover, both procedures control the FDR below the target level throughout the experiment, confirming that the privacy transformation mainly affects power rather than the validity of FDR control.

We also compare the runtime of the two methods in Figure~\ref{fig:compare_runtime}. Let \(d=2p+1\) denote the number of columns in the augmented matrix
\[
A=[X,\widetilde X,y]\in\mathbb R^{n\times d}.
\]
Ignoring the common cost of constructing the knockoff variables, the dominant cost of the non-private procedure is solving the Lasso problem on the \(n\times 2p\) augmented design. If \(I_n\) denotes the number of Lasso iterations, this cost is
\(
O(I_nnp).
\)
For the JLT-private procedure, the dense Gaussian projection step computes
\(
A^*
=
R_1A+wR_2,
\)
which costs
\(
O(rnd)=O(rnp).
\)
After projection, the private Lasso is solved on an \(r\times 2p\) design. If \(I_r\) denotes the corresponding number of Lasso iterations, this costs
\(
O(I_rrp).
\)
Therefore, the total inference cost of the JLT-private procedure is
\[
O(rnp+I_rrp),
\]
compared with
\[
O(I_nnp)
\]
for the non-private procedure. Thus, the JLT method introduces an additional projection cost, but reduces the sample size in the subsequent Lasso optimization from \(n\) to \(r\). In particular, when \(r\ll n\), the Lasso step can be substantially cheaper, while for larger choices of \(r\), the projection cost may dominate.

The runtime experiment reflects this tradeoff. When \(r=n/4\), the private method has runtime comparable to the non-private method, while \(r=n\) leads to a moderate increase in runtime. In contrast, taking \(r=4n\) substantially increases the computational cost, as the dense projection step and the enlarged projected Lasso problem both become more expensive. These results suggest that the computational overhead of privacy depends strongly on the projection dimension \(r\), and that choosing \(r\) below \(n\) can partially offset the additional cost of JLT privatization.

\begin{figure}[]
    \centering
    \includegraphics[width=0.9\textwidth]{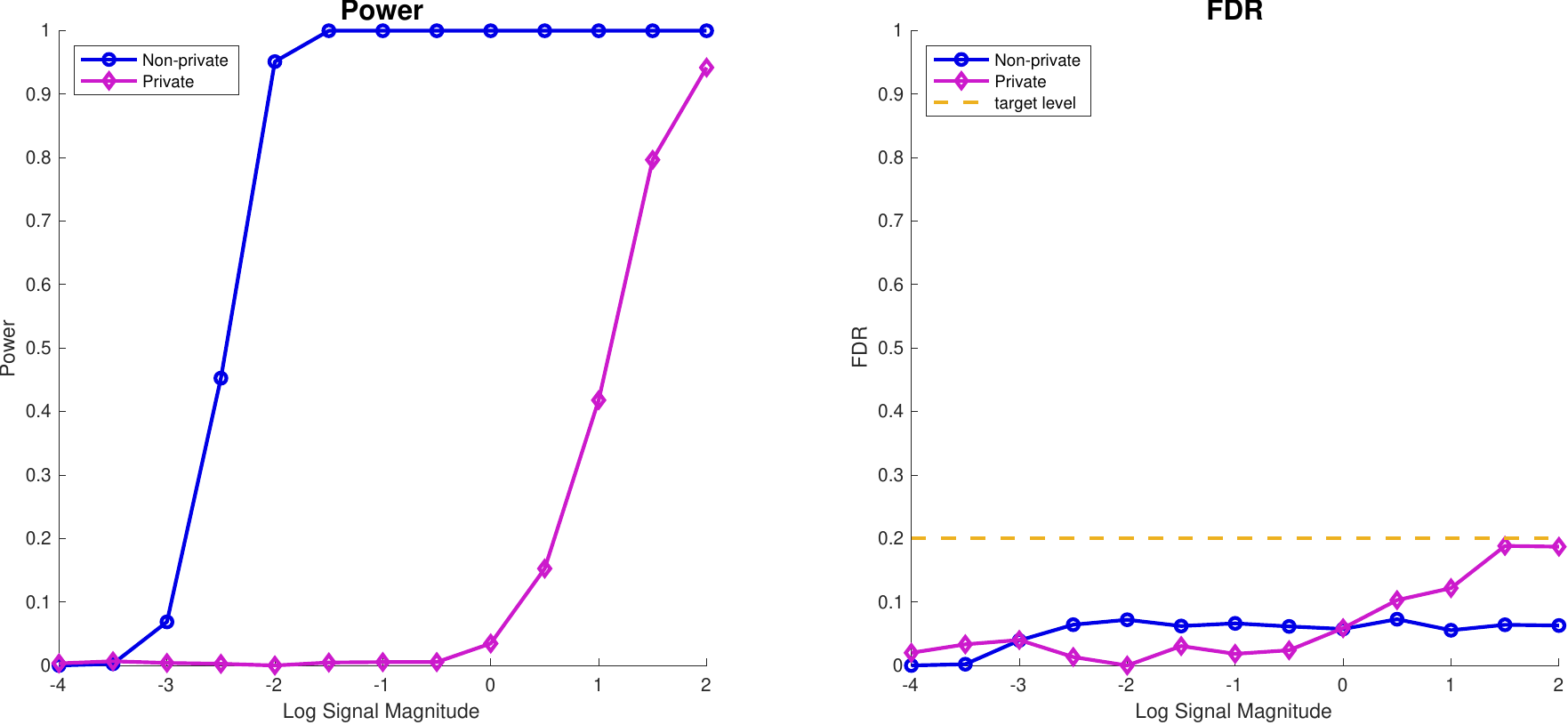} 
    \caption{Power and FDR comparison with non-private model-X knockoffs. Parameter used: $n=2500,p = 50, s_0=15, r=10^5,\epsilon=1,\delta=0.01, \lambda =0.025, q =0.2$.}
    \label{fig:compare_nonDP}
\end{figure}

\begin{figure}[]
    \centering
    \includegraphics[width=0.5\textwidth]{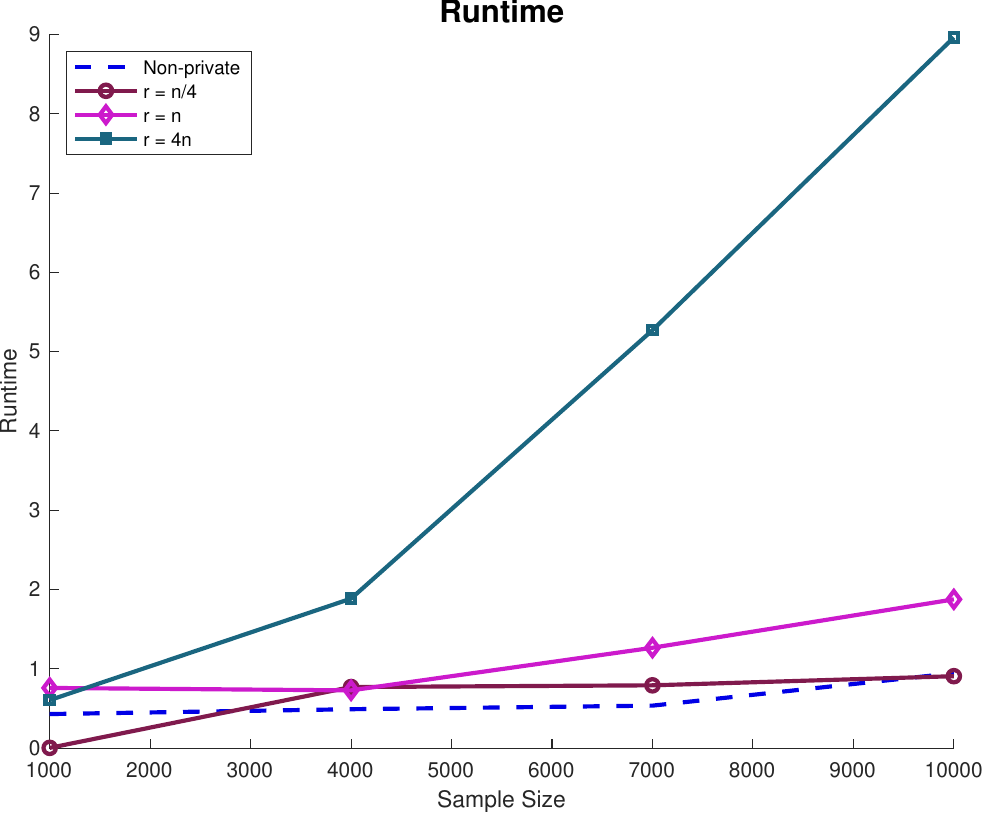} 
    \caption{Runtime comparison with non-private model-X knockoffs. Parameter used: $\mu =1, p = 50, s_0=15, \epsilon=1,\delta=0.01, \lambda =0.025, q =0.2$.}
    \label{fig:compare_runtime}
\end{figure}

\section*{Acknowledgments}
AJ was partially supported by the Sloan fellowship in mathematics, the NSF CAREER Award
DMS-1844481, the NSF Award DMS-2311024, an Amazon Faculty Research Award, an Adobe Faculty Research Award and an iORB grant from USC Marshall School of Business.

\bibliographystyle{apalike}
\bibliography{references}
\newpage

\appendix
\begin{center}
{\bf \LARGE{Supplementary material for\\ ``Differentially Private Model-X Knockoffs via Johnson-Lindenstrauss Transform"}}
\end{center}
\section{Proof of Lemma~\ref{lem:iid-coin-toss}}\label{proof:lem:iid-coin-toss}
Consider $W$ as a function of the data and the randomization matrix: $W=W([\bX\ \bXT],y,\bR)$. Given $S\subseteq[p]$, define $[\bX\,\bXT]_{Swap(S)}$ to be the matrix obtained by swapping the $j$-th column of original variable with its corresponding knockoffs variable at $(j+p)$-th column for all $j\in S_0$. Define
\[
W_{Swap(S)}=W([\bX\ \bXT]_{Swap(S)},y,\bR)\,.
\]
Let $S' = \{j\in [p]:c_j=-1\}\subseteq S_0^c$. To prove the lemma, it suffices to show that\\
(a)$W_{Swap(S')}\stackrel{d}{=}W$ and (b) $W_{Swap(S')}\stackrel{d}{=}(c_1W_1,\dotsc,c_pW_p)$. 

Since $\bR$ is independent from the data, so (a) follows directly from the fact that
\[
[\bX\ \bXT]|y\stackrel{d}{=}[\bX\ \bXT]_{Swap(S')}|y\,,
\]
when $S'\subseteq S_0^c$. This a property for Model-X Knockoffs and we refer to \citep[Lemma 3.2]{candes2017panning} for the proof. 

As for (b), consider the Lasso estimator as a function of $\lambda$, $[\bX\ \bXT]$, $y$, and $\bR$:
\[
\hth^*(\lambda;[\bX\ \bXT],y,\bR) = \arg\min_{\theta\in\mbR^{2p}}\frac{1}{2n}\left\|[\bX^*\ \bXT^*]\theta - y^*\right\|_2^2+\lambda\|\theta\|_1\,.
\]
The privatized design-knockoff matrix can be re-written as 
\[
[\bX^*\ \bXT^*]= \bR_1[\bX\ \bXT]+w\bR_2,\ y^*=\bR_1y+w\bR_3\,.
\]
Now consider replacing $[\bX\ \bXT]$ by $[\bX\ \bXT]_{Swap(S'
)}$, we notice that swapping columns in $[\bX\ \bXT]$ is equivalent to swapping the corresponding coordinates in $\theta$, denoted as $\theta_{Swap(S')}$:
\[
(\bR_1[\bX\ \bXT]_{Swap(S')}+w\bR_2)\theta = (\bR_1[\bX\ \bXT]+w(\bR_2)_{Swap(S')})\theta_{Swap(S')}\,.
\]
Since $\bR_2$ has i.i.d. entries, $\bR_2\stackrel{d}{=}(\bR_2)_{Swap(S')}$. And $\|\cdot\|_1$ is invariant to swapping coordinates, so for any $\theta\in\mbR^{2p}$, $\lambda>0$,
\begin{align*}
    &\frac{1}{2n}\left\|(\bR_1[\bX\ \bXT]_{Swap(S')}+w\bR_2)\theta - (\bR_1y+w\bR_3)\right\|_2^2+\lambda\|\theta\|_1\\
    \stackrel{d}{=}&
\frac{1}{2n}\left\|(\bR_1[\bX\ \bXT]+w\bR_2)\theta_{Swap(S')} - (\bR_1y+w\bR_3)\right\|_2^2+\lambda\|\theta_{Swap(S')}\|_1\,.
\end{align*}
Therefore, the optimal solutions to the equations above also have the same distribution:
\[
\hth^*(\lambda;[\bX\ \bXT]_{Swap(S')},y,\bR) \stackrel{d}{=}(\hth^*(\lambda;[\bX\ \bXT],y,\bR))_{Swap(S')}\,.
\]
So the effect of swapping columns in $[\bX \ \bXT]$ on the debiased Lasso estimator $\hth^u$ is:
\begin{align*}
    \hth^u([\bX\ \bXT]_{Swap(S')},y)& \stackrel{d}{=} \hth^*_{Swap(S')} +\frac{1}{n}([\bX^*\ \bXT]_{Swap(S')})^\top(y^*-[\bX^*\ \bXT]_{Swap(S')}\hth^*_{Swap(S')})\\
    &\quad+\frac{w^2}{n }\hth^*_{Swap(S')}\\
        &=\hth^*_{Swap(S')} +\frac{1}{n}([\bX^*\ \bXT]_{Swap(S')})^\top(y^*-[\bX^*\ \bXT]\hth^*)+\frac{w^2}{n }\hth^*_{Swap(S')}\\
        &=\hth^*_{Swap(S')} +\frac{1}{n}\left([\bX^*\ \bXT]^\top(y^*-[\bX^*\ \bXT]\hth^*)\right)_{Swap(S')}+\frac{w^2}{n }\hth^*_{Swap(S')}\\
        &=\left(\hth^* +\frac{1}{n}[\bX^*\ \bXT]^\top(y^*-[\bX^*\ \bXT]\hth^*)+\frac{w^2}{n }\hth^*\right)_{Swap(S')}\\
        &=\left(\hth^u([\bX\ \bXT],y)\right)_{Swap(S')}\,.
\end{align*}
As a result, for $j\in S'$, 
\begin{align*}
    W_j([\bX\ \bXT]_{Swap(S')},y) &\stackrel{d}{=}  f\left(\left(\hth^u_{Swap(S')}\right)_{j}\right) - f\left(\left(\hth^u_{Swap(S')}\right)_{j+p}\right)\\
    &= f\left(\hth^u_{j+p}\right) - f\left(\hth^u_{j}\right)\\
    &= -W_j([\bX\ \bXT],y)
\end{align*}
Therefore, $W_{Swap(S')}$ has the following property:
\[
W_{Swap(S')}\stackrel{d}{=}(c_1'W_1,\dotsc,c_p'W_p),\ c_j'=\begin{cases}
    1 & \mbox{if }j\notin S'\\
    -1 & \mbox{if }j\in S'\,.
\end{cases}
\]
Note that $c=(c_1',\dotsc,c_p')$ which completes the proof of (b).
\section{Proof of Lemma \ref{lem:function-of-debiased}}\label{pf:function-of-debiased}

Define $\eta:\mbR^{2p}\to\mbR^{2p}$
\[
\eta(t) := \arg\min_{\theta\in\mathbb{R}^{2p}}\frac{ 1}{2}\|t-\theta\|^2+\lambda\|\theta\|_1+\frac{w^2}{2n}\|\theta\|_2^2 \,.
\]
By optimality of $\eta  (t)$, 
\[
    \theta = \eta  (t) \iff  \exists g_\theta \in sg(\theta)\,,\ -(t-\theta) + \lambda g_\theta +\frac{w^2}{n}\theta = 0\,,
\]
where $sg(\theta)$ is the sub-gradient of  $\|\theta\|_1$. In particular, $v\in sg(\theta)$ if and only for any $j\in [2p]$
\[
    v_i \in \begin{cases}
        \{1\} & \text{if } \theta_j>0\\
        \{-1\} & \text{if } \theta_j<0\\
        (-1, 1) & \text{if } \theta_j = 0\,.
    \end{cases}
\]
Recall that 
\[
\hth^*(\lambda) = \arg\min_{\theta\in\mbR^{2p}}\frac{1}{2n}\|y^* -[\bX^*\ \bXT^*]\theta  \|_2^2+\lambda\|\theta\|_1\,.
\]
By the optimality of $\hth^*$, there is a $g_{\hth^*} \in sg(\hth^*)$:
\[    
\lambda g_{\hth^*} = \frac{1}{n}[\bX^*\ \bXT^*]^\top (y^*-[\bX^*\ \bXT^*]\hth^*)\,.
\]
Now consider our definition of $\hth^u$:
\[
\hth^u = \hth^* +\frac{1}{ n}[\bX^*\ \bXT^*]^\top (y^*-[\bX^*\ \bXT^*]\hth^*)+\frac{w^2}{ n}\hth^*\,.
\]
Rearranging the terms in the above equation we have
\[
- (\hth^u -  \hth^*)+\frac{1}{n}[\bX^*\ \bXT^*]^\top (y^*-[\bX^*\ \bXT^*]\hth^*)+\frac{w^2}{n}\hth^*=0\,.
\]
Therefore, $\hth^* = \eta  (\hth^u)$. Now it remains to show that for any $t=(t_1,\dotsc,t_{2p})^\top$, $j\in[2p]$,
\[\left(\eta  (t)\right)_j= \frac{1}{1+ \frac{w^2}{ n}}\cdot \sign(t_j)\cdot\left(|t_j|-\lambda\right)_+\,.\]
Let $\theta = \eta  (t)$, by optimality, there exists $g_\theta \in sg(\theta)$, for all $j\in[2p]$, 
\[    
\left(1+\frac{w^2}{  n}\right)\theta_j=t_j-\lambda (g_\theta)_j\,.
\]
By setting $\theta>0$, $<0$, and $=0$, we have the following,
\begin{align*}
     (\eta  (t))_j &= \frac{1}{1+ \frac{w^2}{  n}}\cdot\begin{cases}
        t_j-\lambda & \text{if } t_j>\lambda\\
        t_j+\lambda & \text{if } t_j<-\lambda\\
        0 & \text{otherwise}\\
    \end{cases}\\
    &= \frac{1}{1+ \frac{w^2}{  n}}\cdot \sign(t_j)\cdot\left(|t_j|-\lambda\right)_+\,.
\end{align*}
\section{Proof of Theorem \ref{thm:debiased-rewrite}}\label{pf:debiased-rewrite}

    First, rewrite the unbiased estimator in the following way:
\small
\begin{align}\label{eq:decompo-Z}
    \hth^u &=  \hth^*+\frac{1}{n }\bX^{*\top}(y^*-\bX^*\hth^*)+\frac{w^2}{n }\hth\nonumber\\
        &= \hth^*+\frac{1}{n }\begin{bmatrix}\bX^\top & w\bI_p & \bf{0}\end{bmatrix} \bR^\top \bR \left(\begin{bmatrix}
            y\\
            \bf0\\
            w
        \end{bmatrix}-\begin{bmatrix}
            \bX\\
            w\bI_p\\
            \bf{0}^\top
        \end{bmatrix} \hth^*\right)+\frac{w^2}{n }\hth^*\nonumber\\
        &=  \hth^*+\frac{1}{n }\begin{bmatrix}\bX^\top & w\bI_p & \bf{0}\end{bmatrix} \begin{bmatrix}
            y-\bX\hth^*\\
            -w\hth^*\\
            w
        \end{bmatrix}+\frac{w^2}{n }\hth^*+\frac{1}{n }\begin{bmatrix}\bX^\top & w\bI_p & \bf{0}\end{bmatrix} (\bR^\top \bR-\bI_n)\begin{bmatrix}
            y-\bX\hth^*\\
            -w\hth^*\\
            w
        \end{bmatrix}\nonumber\\
        &=  \hth^*+\frac{1}{n }\begin{bmatrix}\bX^\top & w\bI_p & \bf{0}\end{bmatrix}\begin{bmatrix}\bX(\theta_0-\hth^*)+\xi\\ -w\hth^*\\ w\end{bmatrix}+\frac{w^2}{n }\hth^*\nonumber\\
        & \quad+\frac{1}{n }\begin{bmatrix}\bX^\top & w\bI_p & \bf{0}\end{bmatrix}(\bR^\top \bR-\bI_n)\begin{bmatrix}\bX(\theta_0-\hth^*)+\xi\\ -w\hth^*\\ w\end{bmatrix}\nonumber\\
        &= \theta_0 + (\hth^*-\theta_0) + \frac{1}{n }\bX^\top \bX(\theta_0-\hth^*) +\frac{1}{n }\bX^\top \xi\nonumber\\
        & \quad+\frac{1}{n }\begin{bmatrix}\bX^\top & w\bI_p & \bf{0}\end{bmatrix}(\bR^\top \bR-\bI_{n+p+1})\begin{bmatrix}\bX(\theta_0-\hth^*)\\ -w\hth^*\\ w\end{bmatrix}+\frac{1}{n }(\bX^\top (\bR_1^\top \bR_1-\bI_n)+w\bR_2^\top \bR_1)\xi \nonumber\\
        &= \theta_0 + \underbrace{\left(\frac{1}{n}\bX^\top \bX-\bI_p\right)(\theta_0-\hth^*)
        +\frac{1}{n}[\bX^\top\ w\bI_p\ {\bf 0}](\bR^\top \bR-\bI_{n+p+1})
        \begin{bmatrix}
        \bX(\theta_0-\hth^*)\\ -w\hth^*\\ w
        \end{bmatrix}}_\Delta\nonumber\\
        &\quad+\underbrace{\frac{1}{n }\bX^{*\top}\bR_1\xi}_Z\,.
\end{align}
\normalsize
Recall that $\xi$ is i.i.d noise with $\Expected(\xi\xi^\top)=\sigma^2\bI$ and $\Expected(\xi)=\bf{0}$, which completes the proof of theorem. 
\section{Proof of Lemma \ref{lem:rho_n}}\label{pf:rho_n}
Then we compute the expected value of the covariance matrix: 
\begin{align*}
    \Expected\left(\frac{1}{n^2}\bX^{*\top}\bR_1\bR_1^\top\bX^*\right) &= \frac{1}{n^2} \Expected_{\bX}\left(\Expected_{\bR}\left(\bX^{*\top}\bR_1\bR_1^\top\bX^*\right)\right)\\
        &= \frac{1}{n^2}\Expected_{\bX}(\Expected_{\bR}((\bX^\top \bR_1^\top+w\bR_2^\top)\bR_1\bR_1^\top(\bR_1X+w\bR_2)))\\
        &=  \frac{1}{n^2}\Expected_{\bX}(\bX^\top \Expected_{\bR}(\bR_1^\top \bR_1 \bR_1^\top \bR_1)\bX+w^2\Expected_{\bR}(\bR_2^\top \bR_1 \bR_1^\top \bR_2))\\
        &=  \frac{1}{n^2}\Expected_{\bX}(\bX^\top \Expected_{\bR_1}(\bR_1^\top \bR_1 \bR_1^\top \bR_1)\bX+w^2\Expected_{\bR_2}(\bR_2^\top \Expected_{\bR_1}(\bR_1 \bR_1^\top) \bR_2))\,.
\end{align*}
To compute $\Expected_R(\bR_1^\top \bR_1 \bR_1^\top \bR_1)$, note that 
\begin{equation}\label{eq:rho-eq1}
    \bR_1^\top \bR_1 = \frac{1}{r}\sum_{k =1}^rW_kW_k^\top\,,\quad W_1,\dotsc,W_r\stackrel{i.i.d.}{\sim}\mcN({{\bf{0}}},\bI_n)\,.
\end{equation}
Then we have that
\begin{align*}
    \Expected(\bR_1^\top \bR_1 \bR_1^\top \bR_1) &=  \frac{1}{r^2}\Expected\left(\left(\sum_{k=1}^r W_kW_k^\top\right)\left( \sum_{\ell=1}^rW_lW_l^\top\right)\right)\\
        &= \frac{1}{r^2}\sum_{k=1}^r\sum_{\ell=1}^r\Expected\left(W_kW_k^\top W_{\ell} W_{\ell}^\top\right)\\
        &= \frac{1}{r^2}\left(r\Expected(W_1W_1^\top W_1W_1^\top)+r(r-1)\Expected(W_1W_1^\top W_2W_2^\top)\right)\,,
\end{align*}
where the last equation holds because $W_1,\dotsc,W_r$ are i.i.d and we only have to consider the cases of $k=\ell$ and $k\neq \ell$. Taking the expectation entry-wise, we have
\begin{align}\label{eq:rho-eq2}
    \Expected\left(W_1W_1^\top W_1W_1^\top\right)_{ij} &= \Expected\left(\sum_{k=1}^n (W_{1,i}\cdot W_{1,k}^2\cdot W_{1,j})\right)\nonumber\\
        &= \begin{cases}
            0 & \text{if }i\neq j\\
            \Expected(W_{1,i}^4)+\sum_{k\neq i}\Expected(W_{1,i}^2)\Expected(W_{1,k}^2) & \text{if }i=j
        \end{cases}\nonumber \\
        &= \begin{cases}
            0 & \text{if }i\neq j\\
            n+2 & \text{if }i=j\,,
        \end{cases}
\end{align}
and $\Expected\left(W_1W_1^\top W_2W_2^\top\right) = \Expected\left(W_1W_1^\top\right)\Expected\left( W_2W_2^\top\right)= \bI_n$. Therefore, 
\begin{align}\label{eq:rho-eq3}
    \Expected(\bR_1^\top \bR_1 \bR_1^\top \bR_1) &=  \frac{1}{r^2}\left(r(n+2)\bI_n+r(r-1)\bI_n\right)\nonumber\\
        &=  \frac{r+n+1}{r}\bI_n\,.
\end{align}
On the other hand,
\begin{align*}
    \Expected_{\bR_2}(\bR_2^\top \Expected_{\bR_1}(\bR_1 \bR_1^\top) \bR_2) =  \Expected_{\bR_2}\left(\bR_2^\top \left(\frac{n}{r}\bI_r\right) \bR_2\right)
    = \frac{n}{r}\bI_p\,.
\end{align*}
The expected value of the covariance matrix follows after plugging in these computation results:
\begin{align*}
    \Expected\left(\frac{1}{n^2}\bX^{*\top}\bR_1\bR_1^\top\bX^*\right) &= \frac{1}{n^2}\Expected_{\bX}(\bX^\top \Expected_{\bR_1}(\bR_1^\top \bR_1 \bR_1^\top \bR_1)\bX+w^2\Expected_{\bR_2}(\bR_2^\top \Expected_{\bR_1}(\bR_1 \bR_1^\top) \bR_2))\\
        &= \frac{1}{n^2 }\Expected_{\bX}\left(\frac{r+n+1}{r}\bX^\top\bX+\frac{nw^2}{r}\bI_p\right)\\
        &=  \left(\frac{r+n+1}{ nr}+\frac{w^2}{nr}\right)\bI_p\,,
\end{align*}
where in the last line, we used the properties of $\bX$ described by~\eqref{eq:X-def}.

We proceed to prove that $|\bQ/\rho_n^2-\bI_p|_\infty=o_p(1)$. First we prove that $|\bX^{\top}(\bR_1^\top\bR_1)^2\bX/\mu^2-\bI_n|_\infty=o_p(1)$, where we let $\mu^2\bI_n:=\Expected(\bX^\top(\bR_1\bR_1)^2\bX)$. As shown in \eqref{eq:rho-eq3}, $\mu^2 = n(r+n+1)/r$. It suffices to show that,
\[
\left|\frac{\bX^{\top}(\bR_1^\top\bR_1)^2\bX}{\mu^2}-\frac{\mbox{tr}((\bR_1^\top\bR_1)^2)}{\mu^2}\bI_p\right|_\infty=o_p(1)\,,\quad\left|\frac{\mbox{tr}((\bR_1^\top\bR_1)^2)}{\mu^2}-1\right|=o_p(1)\,.
\]
The standard Gaussian concentration gives that for any $t_1\ge 0$,
\[
\Prob\left(\sigma_{\max}(\bR_1)\le \left(1+\sqrt{n/r}+t_1/\sqrt{r}\right)_+\right)\leq \exp(-t_1^2/2)\,.
\]
This high probability bound on $\sigma_{\max}(\bR_1)$ implies that with probability $1-\delta_1$,
\[
\|(\bR_1^\top\bR_1)\|_2^2\le \left(1+\sqrt{n/r}+\sqrt{\frac{\log(1/\delta_1)}{2r}}\right)^4\,.
\]
Denote by $\bX_j$ the $j$th column of $\bX$. By Hanson-Wright inequality, for any $i,j$ for any $t_2>0$, there exists a constant $c>0$ such that
\[
\Prob\left(|\bX_i^\top (\bR_1^\top \bR_1)^2\bX_j-\Expected(\bX_i^\top (\bR_1^\top \bR_1)^2\bX_j)|>t_2|\bR\right)\le 2\exp\left(\frac{ct_2}{\|(\bR_1^\top\bR_1)^2\|_2}\right)
\]
\[
\implies\Prob\left(\left|\frac{\bX_i^\top (\bR_1^\top \bR_1)^2\bX_j}{\mu^2}-\frac{\mbox{tr}((\bR_1^\top \bR_1)^2)}{\mu^2}\mathbb \delta_{ij}\right| >\frac{t_2}{\mu^2} \bigg |\bR\right)\le 2\exp\left(-\frac{ct_2}{\|(\bR_1^\top\bR_1)^2\|_2}\right)\,.
\]
By union bound, we have that
\[
\Prob\left(\left|\frac{\bX^{\top}(\bR_1^\top\bR_1)^2\bX}{\mu^2}-\frac{\mbox{tr}((\bR_1^\top\bR_1)^2)}{\mu^2}\bI_p\right|_\infty >\frac{t_2}{\mu^2} \bigg |\bR\right)\le 2p^2\exp\left(-\frac{ct_2}{\|(\bR_1^\top\bR_1)^2\|_2}\right)\,.
\]
Let $t_2 = \log(2p^2/\delta_2)\cdot \|(\bR_1^\top\bR_1)^2\|_2/c$, then as $n,r \to \infty$, with probability $1-\delta_2$, 
\begin{align*}
    \left|\frac{\bX^\top (\bR_1^\top \bR_1)^2\bX}{\mu^2}-\frac{\mbox{tr}((\bR_1^\top \bR_1)^2)}{\mu^2}\bI_p\right|_\infty &\le \frac{\log(2p^2/\delta_2)\cdot \|(\bR_1^\top\bR_1)^2\|_2/c}{n(r+n+1)/r}\\
    &= O\left(\frac{\log(2p^2/\delta_2)\left(1+\sqrt{n/r}+\sqrt{\frac{\log(1/\delta_1)}{2r}}\right)^4}{n(r+n+1)/r}\right)\\
    &= o(1)
\end{align*}
As for $|\mbox{tr}((\bR_1^\top \bR_1)^2)/\mu^2-1|$, we continue to use the notation in \eqref{eq:rho-eq1}, let $\bR_1^\top \bR_1 = \frac{1}{r}\sum_{k =1}^rW_kW_k^\top$. Then
\begin{align*}
    \mbox{tr}\left(\left(\bR_1^\top\bR_1\right)\right) &=\|\bR_1\bR_1^\top\|_F^2\\
        &=\frac{1}{r^2}\left\|\sum_{k=1}^rW_kW_k^\top\right\|_F^2\\
        &=\frac{1}{r^2}\left(\sum_{k=1}^r\|W_k\|_2^4+2\sum_{1\le k<l\le r}(W_k^\top W_l)^2\right)
\end{align*}
Let $f(W_{11},\dotsc,W_{rn})=\|\sum_{k=1}^rW_kW_k^\top\|_F^2$. By Gaussian Poincaré Inequality, 
\[
\Var(f(W_{11},\dotsc,W_{rn}))\leq \Expected\left(\|\nabla f\|_2^2\right)\,.
\]
Then we have that
\begin{align*}
    \Expected\left(\|\nabla f\|_2^2\right) &= \sum_{1\le i, j\le r}\Expected \left(\left(4\sum_{k=1}^rW_i^\top W_k W_{kj}\right)^2 \right)\\
        &= 16nr\sum_{1\le k,l\le r}\Expected\left(W_{kj}W_k^\top W_iW_i^\top W_lW_{lj}\right)\\
        &= 16nr\bigg(\Expected\left(\left(W_i^\top W_i\right)^2W_{ij}^2\right)+\sum_{k=l,k\neq i}\Expected\left(\left(W_i^\top W_k\right)^2W_{kj}^2\right)\\
        &\quad\quad+2\sum_{k\neq i}\Expected\left(W_{ij}\|W_i\|_2^2W_i^\top W_k W_{kj}\right) +\sum_{k,l\neq i,k\neq l}\Expected\left(W_{kj}W_k^\top W_iW_i^\top W_l W_{lj}\right)\bigg)\\
        &=16nr\left((n^2+2n+12)+(r-1)(n+2)+2(r-1)(2n+1)+0\right)\\
        &=16nr\left(n^2+5n+4r-3n+7\right)
\end{align*}
Therefore, 
\begin{align*}
    \Var\left(\frac{\mbox{tr}((\bR_1^\top \bR_1)^2)}{\mu^2}\right)&\le \frac{1}{r^4\mu^4}\cdot 16nr(n^2+5rn-3n+4r+7)\\
    &\le O\left(\max\left\{\frac{1}{n^2},\frac{1}{r^2}\right\}\right)
\end{align*}
Therefore, by Chebyshev inequality,
\[
\Prob\left(\left|\frac{\mbox{tr}((\bR_1^\top\bR_1)^2)}{\mu^2}-1\right|\ge t_3\right)\le t_3^{-2}\cdot O\left(\max\left\{\frac{1}{n^2},\frac{1}{r^2}\right\}\right)
\]
Then with probability $1-\delta_3$, 
\[
\left|\frac{\mbox{tr}((\bR_1^\top\bR_1)^2)}{\mu^2}-1\right|\le \sqrt{\delta_3^{-1}O\left(\max\left\{\frac{1}{n^2},\frac{1}{r^2}\right\}\right)}=o(1)
\]
With a similar manner, we can also prove that 
\[
\left|\frac{w}{n^2}\bX^\top\bR_1^\top\bR_1\bR_1^\top\bR_2\right|_\infty=o_p(1)\,,\quad \left|\frac{\frac{w^2}{n^2}\bR_2^\top\bR_1\bR_1^\top\bR_2}{w^2/nr}-\bI_p\right|_\infty=o_p(1)\,.
\]
Hence, the proof is complete.

\section{Proof of Theorem \ref{thm:suff-cond-debiasing}}\label{pf:suff-cond-debiasing}
By Theorem \ref{thm:debiased-rewrite}, the bias term can be written as
\begin{align}\label{eq:delta-expand}
    \Delta &=  \left(\frac{1}{n}\bX^\top \bX-\bI_p\right)(\theta_0-\hth^*)
        +\frac{1}{n}[\bX^\top\ w\bI_p\ \bf{0}](\bR^\top\bR-\bI_{n+p+1})\begin{bmatrix}\bX(\theta_0-\hth^*)\\ -w\hth^*\\ w\end{bmatrix}\nonumber\\
        &=\frac{1}{ n}\big(\underbrace{(\bX^\top \bX-\bI_p)(\theta_0-\hth^*)}_{(a)}
        +\underbrace{\bX^\top(\bR_1^\top \bR_1-\bI_n)\bX(\theta_0-\hth^*)}_{(b)}
        +\underbrace{w^2(\bR_2^\top \bR_2-\bI_p)\hth^*}_{(c)}\nonumber\\
        &\quad\quad+\underbrace{w\bR_2^\top \bR_1\bX(\theta_0-\hth^*)}_{(d)}
        -\underbrace{w\bX^\top \bR_1^\top \bR_2\hth^*}_{(e)}
        +\underbrace{w\bX^\top \bR_1^\top \bR_3}_{(f)}
        +\underbrace{w^2\bR_2^\top \bR_3}_{(g)}\big)\,.
\end{align}
In order to bound $\|\Delta\|_\infty$, we bound the norm of each of the terms (a)-(g) separately. In doing that we use the inequality $\|\bA v\|_\infty \le |\bA|_\infty \|v\|_1$, where for a matrix $\bA$, we define its max-norm as $|\bA|_\infty :=\max_{i,j}\{|\bA_{ij}|\}$. The bulk of the proof consists in controlling the max-norm of several matrices in the decomposition of $\Delta$, and controlling the error term $\|\theta_0-\hth^*\|_1$ and $\|\theta_0\|_1$.
\begin{lemma}\label{lem:concentration-inequalities}
    For $\bX$ generated as~\eqref{eq:X-def} and $\bR\in\mbR^{r\times (n+p+1)}$ a random matrix with i.i.d entries from $\mcN(0,1/r)$, we have
    \begin{align*}        
    &\left|\frac{1}{n}\bX^\top \bX-\bI_p\right|_\infty = O\left(\sqrt{\frac{\log p}{n}}\right)\,,\quad
    |\bX^\top(\bR_1^\top \bR_1-\bI_n)\bX|_\infty = O\left(n\sqrt{\frac{\log p}{r}}\right)\,,\\
    &|\bR_2^\top \bR_2-\bI_p|_\infty =O\left(\sqrt{\frac{\log p}{r}}\right)\,,\quad
    \|\bR_2^\top \bR_3\|_\infty =O\left(\sqrt{\frac{\log p}{r}}\right)\,,\quad
    |\bR_2^\top \bR_1\bX|_\infty =O\left(\log p\sqrt{\frac{n}{r}}\right)\\
    &|\bX^\top \bR_1^\top \bR_2|_\infty =O\left(\log p\sqrt{\frac{n}{r}}\right)\,,\quad 
    \|\bX^\top \bR_1^\top\bR_3\|_\infty =O\left(\log p\sqrt{\frac{n}{r}}\right)\,.
    \end{align*}
\end{lemma}
For the ease of reading, in the statement of the lemma we only present the asymptotic results and refer to Appendix \ref{pf:concentration-inequalities} for finite sample analysis. Notice that

Next, we bound $\|\theta_0-\hth^*\|_1$. The $\ell_1$-norm of the private Lasso estimator, $\|\hth^*\|_1$, can be upper-bound by triangular inequality,
\[
\|\hth^*\|_1\leq\|\theta_0-\hth^*\|_1+\|\theta_0\|_1\,.
\]
For non-private Lasso, there is an abundant literature on this problem \citep{BulmannvandeGeer11}. In the JLT-privatized Lasso problem, there are two differences that require an adaptation: 1) the appearance of random matrix $\bR$; 2) the appearance of appended constant-diagonal matrix $w\bI_{p+1}$. Given that $\bR$ is Gaussian, the randomness can be dealt with using restricted isometry; the introduction of the $w\bI_{p+1}$ is analogous to ridge regression, which makes the optimization problem in (\ref{eq:lasso-JL-simpl}) similar to the Elastic-Net \citep{elasticnet-zou-hastie} because of the mixture of $\ell_1$ and $\ell_2$ regularization. Our technique to bound the error term is inspired by \citep{HvdG11}. 

Before presenting the upper bound, we introduce an important assumption that the private covariance matrix $\frac{1}{n}\bX^{*\top}\bX^*$ is non-singular in a restricted set of directions. This property is also known as the compatibility condition.
\begin{definition}
    (Compatibility condition) Denote $[p]:=\{1,2,\dotsc,p\}$, let $S\subset [p]$, $\bSigma\in\mathbb{R}^{p\times p}$ is symmetric and semi-positive definite. Let
    \begin{equation}\label{def:restricted-set}
        \mathcal{M}_{\kappa_n,\lambda}(S):= \bigg\{\theta\in\mathbb{R}^p\bigg|\; \|\theta_{S^c}\|_1\leq \frac{2\kappa_n}{\lambda}\|\theta_S\|_2\bigg\}\,,
    \end{equation}
    where $\kappa_n = 2\lambda\sqrt{|S|}+w^2/n$, and $(\theta_S)_i =\begin{cases}
        \theta_i & i\in S\\
        0 & i\in S^c\,.
    \end{cases}$
    
    We say \textbf{compatibility condition }$\mathcal{C}(\bSigma,S)$ is satisfied with constant $\phi(\bSigma,S)$ if there exists a constant $\phi(\bSigma,S)>0$ such that for all $\theta \in \mathcal{M}(S)$,
    \[
    \phi(\bSigma,S)\leq\frac{\theta^\top \bSigma \theta}{\|\theta_S\|_2^2}\,,
    \]
\end{definition}
We claim that given the distribution of $\bR$ and $\bX$, $\mathcal{C}\left(\frac{1}{n}\bX^{*\top}\bX^*,S_0\right)$ is satisfied with $\phi_w = \Omega(1+\frac{w^2}{n})$ with high probability if the condition~\ref{item:cond-comp} is true, namely
\[
\frac{\kappa_n^2}{\lambda^2\phi_w}=o\left(\sqrt{\frac{n}{\log p}}\cdot\min\left\{1,\sqrt{\frac{r}{n}},\frac{\sqrt{rn}}{w^2}\right\}\right)\,,\quad \kappa_n = 2\lambda\sqrt{s_0}+w^2/n\,.
\] 
Lemma~\ref{lem:compatibility1} states that when the max-norm of the difference between matrices $\hat{\bSigma}$ and $\bSigma$ is bounded, the compatibility condition $\mathcal{C}(\bSigma,S)$ implies the compatibility condition $\mathcal{C}\left(\hat{\bSigma},S\right)$. The proof can be found in Appendix \ref{pf:compatibility1}.
\begin{lemma}\label{lem:compatibility1}
    Given arbitrary symmetric and positive semi-definite matrices $\bSigma$ and $\hat{\bSigma}$,  assume $\mathcal{C}(\bSigma,S)$ is satisfied with constant $\phi_0>0$, and
    \begin{equation*}
        \left|\hat{\bSigma}-\bSigma\right|_\infty\leq \tilde{t}\,.
    \end{equation*}
    Then for all $\theta\in\mathcal{M}_{\kappa_n,\lambda}(S)$ as defined in \eqref{def:restricted-set},
    \begin{equation*}
        \left|\frac{\theta^\top \hat{\bSigma}\theta}{\theta^\top \bSigma\theta}-1\right|\leq \frac{\tilde{t}\tilde{\kappa}_n^2}{\phi_0}\,,
    \end{equation*}
    where $\tilde{\kappa}_n=5\sqrt{|S|}+\frac{w^2}{n\lambda}$. Moreover, if we have 
    \begin{equation*}
        \frac{\tilde{t}\tilde{\kappa}_n^2}{\phi_0}\leq h\in (0,1)\,,
    \end{equation*}
    then $\mathcal{C}\left(\hat{\bSigma},S\right)$ is satisfied with constant $(1-h)\phi_0$.
\end{lemma}
Consider the expected value of $\frac{1}{n}\bX^{*\top}\bX^*$:
\[
\Expected\left(\frac{1}{n}\bX^{*\top}\bX^*\right)=\frac{1}{n}\Expected((\bR_1\bX+w\bR_2)^\top(\bR_1\bX+w\bR_2))=\left(1+\frac{w^2}{n}\right)\bI_p\,.
\]
Using concentration inequalities similar to those in Lemma~\ref{lem:concentration-inequalities}, we can bound the max-norm of the difference between $\bX^{*\top}\bX^*/n$ and its expected value, namely $|\frac{1}{n}\bX^{*\top}\bX^*-\left(1+\frac{w^2}{n}\right)\bI_p|_\infty$. We formalize this result in Lemma~\ref{lem:compatibility2} whose proof can be found in Appendix \ref{pf:compatibility2}.

\begin{lemma}\label{lem:compatibility2}
    Let $\hat{\bSigma}=\bX^{*\top}\bX^*/n$ and $\bSigma = \Expected(\hat{\bSigma})$. Then for any $t>0$
    \begin{equation*}
        \Prob\left(\left|\hat{\bSigma}-\bSigma\right|_\infty\leq t\right)\geq 1-(P_1 +P_2 +P_3)\,,
    \end{equation*}
    where 
    \begin{align*}
        &P_1 = 2p^2\exp(-C_1 r t^2)\,,\\
        &P_2 = 2p^2\exp\left(-\frac{nt^2}{18}\right)\,,\\
        &P_3 = p^2\exp\left(-C_2\frac{rn^2t^2}{w^4}\right)\,,
    \end{align*}
    for absolute constants $C_1>0, C_2>0$.
\end{lemma}
Note that for $\bSigma=(1+w^2/n)\bI_p$, $\mathcal{C}(\bSigma,S_0)$ is satisfied with $\phi_0:=1+w^2/n$. By Lemma~\ref{lem:compatibility1}, to show that $\mathcal{C}(\hat{\bSigma},S_0)$ is satisfied with $\phi_w=\Omega(1+w^2/n)$, it suffices to show that $|\hat \bSigma-\bSigma|_\infty\le h\phi_0/\tilde{\kappa}_n^2$ for some $h\in(0,1)$. By Lemma~\ref{lem:compatibility2}, let $t=h\phi_0/\tilde\kappa^2_n$
\[
\Prob\left(|\hat\bSigma-\bSigma|_\infty\le t\right)\ge 1-(P_1+P_2+P_3)\,,
\]
We want to show that $P_1$, $P_2$, $P_3$ all go to zero. By condition~\ref{item:cond-comp}, $\tilde{\kappa}_n^2/\phi_0=\Theta(\kappa_n^2/\phi_0\lambda^2)=o(\sqrt{r/\log(p)})$, which implies that 
\begin{align*}
    P_1 &= 2p^2\exp\left(-\frac{C_1rh^2\phi_0^2}{\kappa_n^4}
    \right)\\
    &=2\exp\left(2\log(p)-C_1h^2\frac{r}{\tilde{\kappa_n}^4/\phi_0}\right)\\
    &=2\exp\left(2\log(p)-\omega(\log(p))\right)\\
    &=o(1)
\end{align*}
With a similar approach, we can prove that $P_2$, $P_3$ go to zero. Hence the claim is justified.

Once the compatibility condition is satisfied, the  error term $\|\theta_0-\hth^*\|_1$ can be upper bounded using the optimality of $\hth^*$. The result is formally stated in Lemma~\ref{lem:error-bound}, whose proof can be found in \ref{pf:error-bound}. 
\begin{lemma}\label{lem:error-bound}
    Assume $\|\theta_0\|_2=1$ and let $\hth^*$ be given by \eqref{eq:lasso-JL-simpl}. We also let $\xi^* = y^*-\bX^*\theta_0$. If $\lambda \geq \frac{1}{n}\|\bX^{*\top}\xi^*+w^2\theta_0\|_\infty $, and  $\mathcal{C}\left(\frac{1}{n}\bX^{*\top}\bX^*,S_0\right)$ is satisfied with $\phi_w >0$, then 
    \begin{equation*}
        \|\theta_0-\hth^*\|_1\leq \frac{4\kappa_n^2}{\phi_w\lambda}\,,
    \end{equation*}
    where $\kappa_n = 2\lambda \sqrt{s_0}+\frac{w^2}{n}$. Furthermore,
    \begin{equation*}
        \|\hth^*\|_1\leq\sqrt{s_0} + \frac{4\kappa_n^2}{\phi_w\lambda}\,.
    \end{equation*}
\end{lemma}
Combining the results in Lemma~\ref{lem:concentration-inequalities} and \ref{lem:error-bound}, we have that term (f) and term (g) in \eqref{eq:delta-expand} is dominated by term (e) and term (c), given that $\sqrt{s_0}\gg 1$. In addition, with condition~\ref{item:cond-error}, \ref{item:cond-sparsity}, and the conditions in Lemma~\ref{lem:error-bound}, we conclude that $\|\Delta\|_\infty=o(\sigma\rho_n)$. We have already proved that condition~\ref{item:cond-comp} implies that $\mathcal C(\hat\bSigma,S_0)$ is satisfied with $\phi_w:=1+w^2/n$. 

Lastly, it remains to show that when the features vector $\theta_0$ is as sparse as indicated in condition~\ref{item:cond-lambda}, then with high probability, $\lambda \geq \frac{1}{n}\|\bX^{*\top}\xi^*+w^2\theta_0\|_\infty$ provided $\lambda$ is of the order $\Omega\left(\sqrt{\frac{\log p}{n}},\sqrt{\frac{(\log p)^3}{r}}\right)$.
\begin{lemma}\label{lem:lambda}
    Let $v = \frac{1}{n}(\bX^{*\top}\xi^*+w^2\theta_0)$. Assume that $\|\theta_0\|_2=1$ and
    \[
    \sqrt{s_0} = O\left(\frac{\max\big\{\frac{1}{\sqrt{n}},\frac{\log p }{\sqrt{r}}\big\}}{\max\big\{\frac{\sqrt{p\log p}}{r^{1/4}\sqrt{n\epsilon}},\frac{p}{n\epsilon}\big\}}\right)\,.
    \]Then with high probability, 
    \[\|v\|_\infty=O\bigg(\max\bigg\{\sqrt{\frac{\log p}{n}},\sqrt{\frac{(\log p)^3}{r}}\bigg\}\bigg)\,.\]
\end{lemma}
The proof of Lemma \ref{lem:lambda} can be found in Appendix \ref{pf:lambda}.

\subsection{Proof of Lemma~\ref{lem:concentration-inequalities}} \label{pf:concentration-inequalities}

\begin{definition}[Sub-exponential norm]
    The sub-exponential norm of a random variable $X$, denoted by $\|X\|_{\psi_1}$ is defined as
    \[
    \|X\|_{\psi_1}=\sup_{q\ge 1}q^{-1}(\Expected(|X|^q))^{1/q}
    \]
    The sub exponential norm of $X$, denoted $\|X\|_{\psi_1}$, is defined to be the smallest parameter $K_2$.
\end{definition}
We prove the Lemma~\ref{lem:concentration-inequalities} by proving the following four lemmas:
\begin{lemma}\label{lemma:boundR}
    For any $a >0$, 
    \begin{equation}\label{eq:boundR1}
        \Prob\bigg(\big|\bR_2^\top \bR_2 -\bI_p\big|_\infty\leq a\sqrt{\frac{\log p}{r}}\bigg)\geq 1- 2 p^{2-Ca^2}\,,
    \end{equation}
    \begin{equation}\label{eq:boundR2}
        \Prob\bigg(|\bR_2^\top \bR_3|_\infty\leq a\sqrt{\frac{\log p}{r}}\bigg)\geq 1- 2 p^{1-Ca^2}\,,
    \end{equation}
    for some absolute constant $C>0$.
\end{lemma}
\begin{proof} 
We prove the result for inequality \eqref{eq:boundR1}. Inequality \eqref{eq:boundR2} can be proved in the same manner. Notice that $\bR_{ij}$ are i.i.d draws from $\mcN(0,1/r)$, and therefore $\sqrt{r}\bR_{ij}$ is i.i.d. standard normal. Note that $(\bR^\top\bR-\bI_{n+p+1})_{ij} = \frac{1}{r}\bigg(\sum_{k=1}^r r\bR_{ki}\bR_{kj}-\delta_{ij}\bigg)$, where $\delta_{ij}=\begin{cases}
    0 & i\neq j\\
    1 & i=j
\end{cases}$. 
For any $i,j$, consider the sub-exponential norm of $\bR_{ki}\bR_{kj}$. Since $\bR_{ki}$, $\bR_{kj}$ is i.i.d. when $i\neq j$, we have 
\begin{align*}
\|r\bR_{ki}\bR_{kj}\|_{\psi_1}&= r\sup_{q\ge1} q^{-1} \Expected(|\bR_{ki}\bR_{kj}|^q)^{1/q}\\
&\le r\sup_{q\ge1} q^{-1} \Expected(\bR_{ki}^{2q})^{1/q}\\
&= \|r\bR_{ki}^2\|_{\psi_1}\\
&= \|Z_1^2\|_{\psi_1}
\end{align*}
for standard Gaussian random variables $Z_1$. We let $K = \|Z_1^2\|_{\psi_1}$. Notice that for $k=1,\dotsc,r$, the terms $r\bR_{ki}\bR_{kj} - \delta_{ij}$ are i.i.d. and zero-mean. We apply Bernstein-type inequality \citep[Theorem 2.9.1]{vershynin11}, for any $t>0$,
\[ \Prob\bigg(\frac{1}{r}\bigg|\bigg(\sum_{k=1}^r r\bR_{ki}\bR_{kj}-\delta_{ij}\bigg)\bigg|\geq t\bigg)\leq 2\exp\bigg(-\frac{crt^2}{K^2}\bigg)\]
where $c$ is a positive absolute constant. 
Now, let $t = a\sqrt{\log p/r}$,
\begin{equation}\label{eq:boundR3}
    \Prob\bigg(\frac{1}{r}\bigg|\bigg(\sum_{k=1}^r r\bR_{ki}\bR_{kj}-\delta_{ij}\bigg)\bigg|\geq a\sqrt{\frac{\log p}{r}}\bigg)\leq 2p^{-\frac{a^2c}{K^2}}\,.
\end{equation}
Letting $C=\frac{c}{K^2}$, we have
\[\Prob\bigg(|(\bR_2^\top \bR_2-\bI_p)_{ij}| \geq a\sqrt{\frac{\log p}{r}}\bigg)\leq 2p^{-Ca^2}\]
To finish the proof, we apply union bound. Note that $\bR_2\in\mbR^{r\times p}$:
\begin{align*}
    \Prob\bigg(|\bR_2^\top \bR_2-\bI_p|_\infty \leq a \sqrt{\frac{\log p}{r}}\bigg)&\geq  1- \Prob\bigg(|\bR_2^\top \bR_2-\bI_p|_\infty > a \sqrt{\frac{\log p}{r}}\bigg)\\
        &\geq  1- \sum_{ij}\Prob\bigg(|(\bR_2^\top \bR_2-\bI_p)_{ij}| \geq a\sqrt{\frac{\log p}{r}}\bigg)\\
        &= 1-p^2\cdot 2p^{-Ca^2}\\
        &=  1- 2p^{2-Ca^2}\,.
\end{align*}
\end{proof}
As a consequence, if we choose $a$ large enough (e.g. $a>\sqrt{2/C})$, the event 
\begin{equation}\label{eq:boundR}
    |\bR_2^\top \bR_2 - \bI_p|_\infty = O\bigg(\sqrt{\frac{\log p}{r}}\bigg)\,,
\end{equation}
holds with high probability. A similar argument can be used to bound $|\bR_2^\top \bR_3|_\infty$.

Now take the randomness of $\bX$ into account: conditioning on $\bR$, from inequality \eqref{eq:boundR3}, $|(\bR^\top\bR-\bI_{n+p+1})_{ij}|$ can be bounded with high probability, with proper choice of $t$, and $\bX_{ij}$ are i.i.d. random variables bounded by one. We have the following lemma on the upper bound of the product of the sub-matrices of $\bR^\top\bR-\bI_{n+p+1}$ and $\bX$:

\begin{lemma}\label{lemma:boundXR}
    For any $a,b>0$, 
    \begin{equation}\label{eq:boundXR1}
        \Prob\bigg(\big|\bX^\top \bR_1^\top \bR_2\big|_\infty \leq b\frac{\sqrt{n}\log p}{\sqrt{r}} \bigg)\geq 1-\left(2np^{2-Ca^2}(1-2p^{-\frac{b^2}{2a^2}})+2p^{2-\frac{b^2}{2a^2}}\right)\,,
    \end{equation}
    
    \begin{equation}\label{eq:boundXR2}
        \Prob\bigg(\big|\bX^\top \bR_1^\top \bR_3\big|_\infty \leq b\frac{\sqrt{n}\log p}{\sqrt{r}} \bigg)\geq 1-\left(p^{2-Ca^2}(1-2p^{-\frac{b^2}{2a^2}})+2p^{1-\frac{b^2}{2a^2}}\right)\,,
    \end{equation}
    for some absolute constant $C>0$.
\end{lemma}

\begin{proof} 
We prove inequality (\ref{eq:boundXR1}). Inequality \eqref{eq:boundXR2} can be proved in a similar manner. By Lemma \ref{lemma:boundR}, the entries of $\bR^\top\bR-\bI_{n+p+1}$ are bounded with high probability. Let $(A_1,\dotsc,A_n)^\top$ be a column in the matrix $\bR_1^\top \bR_2\in\mathbb{R}^{n\times p}$. Using inequality (\ref{eq:boundR3}) with $t = a\sqrt{\frac{\log p}{r}}$ for arbitrary $a>0$,  $|A_i|\leq a\sqrt{\frac{\log (p)}{r}}$, with probability at least $\geq 1-2p^{-Ca^2}$. By union bounding over $k\in[n]$, we can bound $|A_k|$'s uniformly:
\begin{equation}\label{eq:boundXR3}
    \Prob\left(\max_{1\leq k\leq n} |A_k|\leq a \sqrt{\frac{\log (p)}{r}}\right)\geq 1-2np^{-Ca^2}\,.
\end{equation}
Let $(B_1,\dotsc,B_n)^\top$ be a column in $\bX\in\mathbb{R}^{n\times p}$. In particular, $|B_i|\leq 1$ and are i.i.d. random variables. So we can examine the upper bound of $|(\bX^\top \bR_1^\top \bR_2)_{ij}|$ by looking at the inner product between $(A_1,\dotsc,A_n)$ and $(B_1,\dotsc,B_n)$. Use Hoeffding inequality by conditioning on $A_k$'s, we get

\[\Prob\bigg(\bigg|\sum_{k=1}^nA_kB_k\bigg|\geq t\bigg|\{A_k\}_{k\in[n]}\bigg)\leq \exp\bigg(-\frac{2t^2}{\sum_{i=k}^n 4|A_k|^2}\bigg)\,.\]
We next define the event $E = \bigg\{\max_{1\leq k\leq n} |A_k|\leq a \sqrt{\frac{\log p}{r}}\bigg\}$ whose probability is lower bounded by \eqref{eq:boundXR3}. We have 
\begin{align*}
    \Prob\bigg(\bigg|\sum_{i=1}^nA_iB_i\bigg|\geq t\bigg) &=  \Prob\bigg(\bigg|\sum_{i=1}^nA_iB_i\bigg|\geq t\bigg|E\bigg)\Prob(E)+\Prob\bigg(\bigg|\sum_{i=1}^nA_iB_i\bigg|\geq t\bigg|E^c\bigg)(1-\Prob(E))\\
        &\le  2np^{-Ca^2}+(1-2np^{-Ca^2})2\exp\left(-\frac{2t^2}{4a^2n(\log p)/r}\right)\,.
\end{align*}
Take $t = b\log p\sqrt{\frac{n}{r}}$ for arbitrary $b>0$,
\[\Prob\bigg(\bigg|\sum_{k=1}^nA_kB_k\bigg|\geq b\frac{\sqrt{n}\log p}{\sqrt{r}}\bigg)\leq  2np^{-Ca^2}(1-2p^{-\frac{b^2}{2a^2}})+2p^{-\frac{b^2}{2a^2}}\,,\]
Applying union bound over all the $p^2$ coordinates of $\bX^\top \bR_1^\top \bR_2$, we obtain
\small
\begin{align*}
    \Prob\left(|\bX^\top \bR_1^\top \bR_2|_\infty\leq b\frac{\sqrt{n}\log p}{\sqrt{r}}\right)\geq & 1-p^2\cdot\Prob\bigg(\bigg|\sum_{k=1}^nA_kB_k\bigg|\geq b\frac{\sqrt{n}\log p}{\sqrt{r}}\bigg)\\
    \geq &1- \left(2np^{2-Ca^2}(1-2p^{-\frac{b^2}{2a^2}})+2p^{2-\frac{b^2}{2a^2}}\right)\,,
\end{align*}
\normalsize
which completes the proof of inequality~\eqref{eq:boundXR1}. 
\end{proof}
As a consequence, assuming $n=O(p^\alpha)$ for some $\alpha>0$, if we pick $a,b$ large enough ($a>\sqrt{C/(2+\alpha)},b>2a$), then the following event holds with high probability,
\begin{equation}\label{eq:boundXR}
    \big|\big(\bX^\top \bR_1^\top \bR_2\big)_\infty\big|=O\bigg(\log p \sqrt{\frac{n}{r}}\bigg)\,,
\end{equation}
for large $p$. A similar argument applies to $\big|\bX^\top \bR_1^\top \bR_3\big|_\infty$. Lastly, to bound $\big|\bX^\top(\bR_1^\top \bR_1-\bI_n)\bX\big|_\infty$, we have the following lemma. 
\begin{lemma}\label{lemma:boundXRX}
    For any $a>0$,
    \begin{equation*}
        \Prob\left(\left|\bX^\top(\bR_1^\top \bR_1-\bI_n)\bX\right|_\infty\leq a\sqrt{\frac{n\log p}{r}}\right)\geq 1-2p^{2-Ca^2}\,,
    \end{equation*}
    for some absolute constant $C>0$.
\end{lemma}
\begin{proof}
Notice that we can decompose $\bR_1^\top \bR_1$ into
\[    \bR_1^\top \bR_1 \stackrel{d}{=} \frac{1}{r}\sum_{k=1}^rZ_kZ_k^\top \,,
\]
where $Z_1,\dotsc,Z_r$ are i.i.d. random vectors sampled from $\mcN({{\bf{0}}},\bI_n)$. Now, if we condition on $\bX$,
\[\bX_i^\top(\bR_1^\top \bR_1-\bI_n)\bX_j \stackrel{d}{=} \frac{1}{r}\sum_{k=1}^r\underbrace{(\bX_i^\top Z_kZ_k^\top \bX_j-\bX_i^\top \bX_j)}_\text{zero mean, i.i.d.}\,.\]
Notice that 
\[|\bX_i^\top Z_kZ_k^\top \bX_j|\leq \max\bigg\{(\bX_i^\top Z_k)^2,(\bX_j^\top Z_k)^2\bigg\}\,.\]
Without lost of the generality, assume $\max\bigg\{(\bX_i^\top Z_k)^2,(\bX_j^\top Z_k)^2\bigg\}=(\bX_i^\top Z_k)^2$. Notice that given $\bX_i$, $\bX_i^\top Z_k\sim\mcN(0,\|\bX_i\|_2^2)$ with $\|\bX_i\|_2^2\le n$. Since $\bX$ is independent from $\bR$, $\bX_i^\top Z_k$ is Gaussian.
By \citep[Lemma 2.7.6]{vershyninHDP}, for some absolute constant $K>0$,
\begin{align*}
    \|\bX_i^\top Z_k Z_k^\top \bX_j\|_{\psi_1}\leq  \|(\bX_i^\top Z_k)^2\|_{\psi_1}
        \leq c_1^2K^2n\,.
\end{align*}
Applying Bernstein-type inequality, we have 
\[    \Prob\bigg(\bigg|\frac{1}{r}\sum_{k=1}^r\bX_i^\top Z_kZ_k^\top \bX_j-\bX_i^\top \bX_j\bigg|\geq t\bigg)\leq 2 \exp\bigg(-c_2\frac{rt^2}{(c_1^2K^2n)^2}\bigg)\,,
\]
for some absolute constant $c_2>0$. Let $C=\frac{c_2}{c_1^4K^4}$, $t=a\frac{n\sqrt{\log p}}{\sqrt{r}}$, we have 
\[ \Prob\left(\bX_i^\top (\bR_1^\top \bR_1-\bI_n)\bX_j\geq a\frac{n\sqrt{\log p}}{\sqrt{r}} \right)\leq 2p^{-Ca^2}\,.\]
By union bounding over $i,j\in[p]$,
\begin{align*}
    \Prob\left( |\bX^\top (\bR_1^\top \bR_1-\bI_n)\bX|_{\infty} \leq a\frac{n\sqrt{\log p}}{\sqrt{r}}\right)\ge  1-p^2\cdot2p^{-Ca^2}
        =1-2p^{2-Ca^2}\,.
\end{align*}
\end{proof}
As a consequence, if $a>\sqrt{2/C}$, the following holds with high probability
\begin{equation}\label{eq:boundXRX}
    |(\bX^\top(\bR_1^\top \bR_1-\bI_n)\bX)_{ij}| = O\bigg(\frac{n\sqrt{\log p}}{\sqrt{r}}\bigg)\,.
\end{equation}
\begin{lemma}\label{lemma:boundX}
    For any $a>0$,
    \begin{equation*}
        \Prob\bigg(\bigg|\frac{1}{n }\bX^\top \bX-\bI_p\bigg|_\infty\leq a\sqrt{\frac{\log p}{  n}}\bigg)\geq 1- 2p^{2-\frac{1}{2} a^2}\,.
    \end{equation*}
\end{lemma}
\begin{proof}
Applying Hoeffiding inequality, we obtain that 
\[
\Prob\bigg(\bigg|\frac{1}{n}\sum_{k=1}^n\bigg(\bX_{jk}\bX_{ik}-\delta_{ij}\bigg)\bigg|\geq t\bigg)\leq 2\exp\left(-\frac{t^2n }{2}\right)
\]
take $t = a\sqrt{\frac{\log p}{  n}}$, then
\[\Prob\bigg(\bigg|\frac{1}{n}\sum_{k=1}^n\bigg(\bX_{jk}\bX_{ik}-\delta_{ij}\bigg)\bigg|\geq a\sqrt{\frac{\log p}{  n}}\bigg)\leq 2p^{-\frac{1}{2} a^2}\,,\]
with union bound,
\[    \Prob\bigg(\bigg|\frac{1}{n }\bX^\top \bX-\bI_p\bigg|_\infty\leq a\sqrt{\frac{\log p}{  n}}\bigg)\geq 1- 2p^{2-\frac{1}{2} a^2}\,.
\]
\end{proof}
\subsection{Proof of Lemma \ref{lem:compatibility1}}\label{pf:compatibility1}

    Given that $\theta \in \mathcal{M}(S)$ and $\big|\hat{\bSigma}-\bSigma\big|_\infty \leq \tilde{t}$,
    \begin{align*}
        \big|\theta^\top \hat{\bSigma}\theta - \theta^\top \bSigma\theta\big| &\leq \big|\sum_{ij}(\hat{\bSigma}-\bSigma)_{ij}\theta_i \theta_j \big| \\
        &\le \sum_{ij} \big|(\hat{\bSigma}-\bSigma)_{ij}\big| |\theta_i| |\theta_j|\\
        &\le \big|\hat{\bSigma}-\bSigma\big|_\infty \|\theta\|_1^2 \leq \tilde{t}\|\theta\|_1^2\,.
    \end{align*}
    By the definition of $\mathcal{M}(S)$, 
    \begin{align*}
        \|\theta\|_1 &=  \|\theta_S\|_1+\|\theta_{S^c}\|_1\\
            &\leq  \|\theta_S\|_1+\frac{2\kappa_n}{\lambda}\|\theta_S\|_2\\
            &\leq  (\sqrt{|S|}+4\sqrt{|S|}+\frac{2w^2}{n\lambda})\|\theta_S\|_2\\
            &=  \tilde{\kappa}_n\|\theta_S\|_2\,.
    \end{align*}
    Therefore,
    \[\big|\theta^\top \hat{\bSigma}\theta - \theta^\top \bSigma\theta\big|\leq \tilde{t}\tilde{\kappa}_n^2\|\theta_S\|_2^2\,.\]
    Since $\mathcal{A}(\bSigma,S)$ is satisfied with constant $\phi_w>0$, then 
    \[\big|\theta^\top \hat{\bSigma}\theta - \theta^\top \bSigma\theta\big|\leq \tilde{t}\tilde{\kappa}_n^2\phi_0^{-1}\cdot \theta^\top \bSigma \theta\,.\]
    As a result,
    \[\left|\frac{\theta^\top \hat{\bSigma}\theta}{\theta^\top \bSigma\theta}-1\right|\leq \frac{\tilde{t}\tilde{\kappa}_n^2}{\phi_0}\,.\]
    Moreover, notice that if $\frac{\tilde{t}\tilde{\kappa}_n^2}{\phi_0}\leq h\in(0,1)$, then for any $\theta\in \mathcal{M}(S)$
    \[        \frac{\theta^\top\hat{\bSigma}\theta}{\|\theta_S\|_2^2}\geq \frac{\theta^\top \bSigma\theta}{\|\theta_S\|_2^2}\left(1-\frac{\tilde{t}\tilde{\kappa}_n^2}{\phi_0}\right)\geq (1-h)\phi_0\,.
\]

\subsection{Proof of Lemma \ref{lem:compatibility2}}\label{pf:compatibility2}
\begin{proof}
    Note that 
    \begin{align*}
        \big|\hat{\bSigma}-\bSigma\big|_\infty &=  \left|\left( \frac{\bX^\top \bR_1^\top \bR_1 \bX}{n}-\bI_p\right)+\frac{w^2}{n}(\bR_2^\top \bR_2-\bI_p)\right|_\infty\\
            &\le \left|\frac{\bX^\top \bR_1^\top \bR_1 \bX}{n}-\bI_p\right|_\infty +\frac{w^2}{n}|\bR_2^\top \bR_2-\bI_p|_\infty\\
            &\leq \underbrace{\left|\frac{1}{n}\bX^\top (\bR_1^\top \bR_1-\bI_n)\bX\right|_\infty}_{T_1}+\underbrace{\left|\frac{1}{n}\bX^\top \bX-\bI_p\right|_\infty}_{T_2}+\underbrace{\frac{w^2}{n}|\bR_2^\top \bR_2-\bI_p|_\infty}_{T_3}\,.
    \end{align*}
    Also notice that 
    \begin{align*}
        \Prob\left(\left|\hat{\bSigma}-\bSigma\right|_\infty\leq t\right)&\geq \Prob(T_1+T_2+T_3\leq t)\\
            &=  1- \Prob(T_1+T_2+T_3>t)\\
            &\geq 1- \bigg(\Prob(T_1>t/3) + \Prob(T_2>t/3)+\Prob(T_2>t/3)\bigg)\,.
    \end{align*}
    The three probabilities on the right hand side of the last inequality can be bounded by lemma \ref{lemma:boundXRX}, \ref{lemma:boundX}, and \ref{lemma:boundR} correspondingly. And the result follows.
\end{proof}

\subsection{Proof of Lemma~\ref{lem:error-bound}}\label{pf:error-bound}
    By the definition of $\hat{\theta}$ in (\ref{eq:lasso-JL-simpl}), the optimality gives
    \begin{align*}
        &\quad\frac{1}{2n}\|\bX^*\hat{\theta}-y^*\|_2^2+\lambda\|\hat{\theta}\|_1\leq \frac{1}{2n}\|\bX^*\theta_0-y^*\|_2^2+\lambda\|\theta_0\|_1\\
        &\Rightarrow  \frac{1}{2n}\|\bX^*(\hat{\theta}-\theta_0)-\xi^*\|_2^2\leq \frac{1}{2n}\|\xi^*\|_2^2+\lambda(\|\theta_0\|_1-\|\hat{\theta}\|_1)\\
        &\Rightarrow  \frac{1}{2n}\|\bX^*(\hat{\theta}-\theta_0)\|_2^2\leq \lambda(\|\theta_0\|_1-\|\hat{\theta}\|_1) + \frac{1}{n} \xi^{*\top}\bX^*(\hat{\theta}-\theta_0)\\
        &\Rightarrow  \frac{1}{2n}\|\bX^*(\hat{\theta}-\theta_0)\|_2^2\leq \lambda(\|\theta_0\|_1-\|\hat{\theta}\|_1) + \frac{1}{n} (\xi^{*\top}\bX^*+w^2\theta_0^\top)(\hat{\theta}-\theta_0)-\frac{w^2}{n}\theta_0^\top(\hat{\theta}-\theta_0)\\
        &\Rightarrow \frac{1}{2n}\|\bX^*(\hat{\theta}-\theta_0)\|_2^2\leq \lambda(\|\theta_0\|_1-\|\hat{\theta}\|_1) + \frac{1}{n}\|\xi^{*\top}\bX^*+w^2\theta_0^\top\|_\infty\|\hat{\theta}-\theta_0\|_1+\frac{w^2}{n}\|\theta_0\|_2\|\hat{\theta}_{S_0}-\theta_0\|_2\,.
    \end{align*}
    By lemma \ref{lem:lambda}, there exists $C_\lambda>0$ such that the probability of the event 
    \begin{equation*}
        \frac{1}{2}C_\lambda\max\bigg\{\sqrt{\frac{\log p}{n}},\sqrt{\frac{(\log p)^3}{r}}\bigg\}\geq \frac{1}{n}\|\xi^{*\top}\bX^*+w^2\theta_0^\top\|_\infty\,,
    \end{equation*}
    converges to one as $p \to \infty$. Let $\lambda = C_\lambda\max\bigg\{\sqrt{\frac{\log p}{n}},\sqrt{\frac{(\log p)^3}{r}}\bigg\}$. So we have
    \begin{align*}
        &\quad\frac{1}{2n}\|\bX^*(\hat{\theta}-\theta_0)\|_2^2\leq \lambda(\|\theta_0\|_1-\|\hat{\theta}\|_1) + \frac{\lambda}{2}\|\hat{\theta}-\theta_0\|_1+\frac{w^2}{n}\|\theta_0\|_2\|\hat{\theta}_{S_0}-\theta_0\|_2\\
        &\Rightarrow  \frac{1}{2n}\|\bX^*(\hat{\theta}-\theta_0)\|_2^2+\frac{\lambda}{2}\|\hat{\theta}-\theta_0\|_1 \leq \lambda(\|\theta_0\|_1-\|\hat{\theta}\|_1) + \lambda\|\hat{\theta}-\theta_0\|_1+\frac{w^2}{n}\|\theta_0\|_2\|\hat{\theta}_{S_0}-\theta_0\|_2\\
        &\Rightarrow  \frac{1}{2n}\|\bX^*(\hat{\theta}-\theta_0)\|_2^2+\frac{\lambda}{2}\|\hat{\theta}-\theta_0\|_1 \leq  2\lambda\|\hat{\theta}_{S_0}-\theta_0\|_1+\frac{w^2}{n}\|\theta_0\|_2\|\hat{\theta}_{S_0}-\theta_0\|_2\\
        &\Rightarrow  \frac{1}{2n}\|\bX^*(\hat{\theta}-\theta_0)\|_2^2+\frac{\lambda}{2}\|\hat{\theta}-\theta_0\|_1 \leq  2\lambda\sqrt{s_0}\|\hat{\theta}_{S_0}-\theta_0\|_2+\frac{w^2}{n}\|\theta_0\|_2\|\hat{\theta}_{S_0}-\theta_0\|_2\\
        &\Rightarrow  \frac{1}{2n}\|\bX^*(\hat{\theta}-\theta_0)\|_2^2+\frac{\lambda}{2}\|\hat{\theta}-\theta_0\|_1 \leq  2\lambda\sqrt{s_0}\|\hat{\theta}_{S_0}-\theta_0\|_2+\frac{w^2}{n}\|\theta_0\|_2\|\hat{\theta}_{S_0}-\theta_0\|_2\\
        &\Rightarrow  \frac{1}{2n}\|\bX^*(\hat{\theta}-\theta_0)\|_2^2+\frac{\lambda}{2}\|\hat{\theta}-\theta_0\|_1 \leq  \kappa_n\|\hat{\theta}_{S_0}-\theta_0\|_2\,.
    \end{align*}
    Because $\frac{1}{2n}\|\bX^*(\hat{\theta}-\theta_0)\|_2^2\geq 0 $,
    \begin{equation}\label{eq:bound-theta0-thetahat-2}
        \|\hat{\theta}-\theta_0\|_1\leq \frac{2\kappa_n}{\lambda}\|\hat{\theta}_{S_0}-\theta_0\|_2\,.
    \end{equation}
    Then by assumption $\mathcal{A}(\hat{\bSigma},S_0)$, 
    \begin{equation*}
        \|\hat{\theta}_{S_0}-\theta_0\|_2^2\leq \phi_w^{-1}\frac{1}{n}\|\bX^*(\hat{\theta}-\theta_0)\|_2^2\,.
    \end{equation*}
    Because $\frac{\lambda}{2}\|\hat{\theta}-\theta_0\|_1\geq 0$, we have
    \begin{align*}
        \frac{1}{2n}\|\bX^*(\hat{\theta}-\theta_0)\|_2^2\leq  \kappa_n\|\hat{\theta}_{S_0}-\theta_0\|_2
        \leq  \kappa_n\phi_w^{-1/2}\frac{1}{\sqrt{n}}\|\bX^*(\hat{\theta}-\theta_0)\|_2\,.
    \end{align*}
    So
    \begin{align*}
        \ \frac{1}{n}\|\bX^*(\hat{\theta}-\theta_0)\|_2^2\leq  \frac{4\kappa_n^2}{\phi_w}
        \quad \Rightarrow \quad \|\hat{\theta}_{S_0}-\theta_0\|_2\leq \frac{2\kappa_n}{\phi_w}\,.
    \end{align*}
    Finally by equation (\ref{eq:bound-theta0-thetahat-2}), we have
    \begin{equation*}
        \|\hat{\theta}-\theta_0\|_1\leq \frac{4\kappa_n^2}{\lambda\phi_w}\,.
    \end{equation*}

\subsection{Proof of Lemma \ref{lem:lambda}}\label{pf:lambda}

    Notice that
    \begin{align*}
        v^\top = \frac{1}{n}\big(&\xi^\top \bX+\xi^\top (\bR_1^\top \bR_1-\bI_n)\bX+w\xi^\top \bR_1^\top \bR_2 - w\theta_0^\top \bR_2^\top \bR_1\bX \\
            &-w^2\theta_0^\top (\bR_2^\top \bR_2-\bI_p)+w\bR_3^\top \bR_1X+w^2\bR_3^\top \bR_2\big)\,.
    \end{align*}
    For the first term, $(\xi\top \bX)_j = \sum_{i=1}^n \xi_i \bX_{ij}$, apply Hoeffiding-type inequality
    \[        \Prob(\frac{1}{n}|(\xi^\top \bX)_j|\geq t)\leq e\cdot \exp(-Cnt^2/\sigma^2)\,,
\]
    for some absolute constant $C>0$. Let $t=\alpha \sqrt{\frac{\log p}{n}}$, we have 
    \begin{equation*}
        \Prob\left(\frac{1}{n}|(\xi^\top \bX)_j|\geq \alpha\sigma \sqrt{\frac{\log p}{n}}\right)\le ep^{-C\alpha^2}\,.
    \end{equation*}
    Therefore by union bound, the probability of the event $\|\frac{1}{n}\xi^\top \bX\|_\infty=O\bigg(\sqrt{\frac{\log p}{n}}\bigg)$ goes to one when $n\to\infty$.
    For the rest of the terms, as we have shown in the inequalities (\ref{eq:boundR}) and (\ref{eq:boundXR}), the probability of the following events goes to 1 as $p\to \infty$:
    \begin{enumerate}
        \item $\frac{1}{n}\|\xi^\top (\bR_1^\top \bR_1-\bI_n)\bX\|_\infty=O\bigg(\sigma\sqrt{\frac{(\log p )^3}{r}} \bigg)$
        \item $\frac{1}{n}\|w\xi^\top \bR_1^\top \bR_2\|_\infty=O\bigg(\sigma \frac{\sqrt{p}\log p}{r^{1/4}\sqrt{n\epsilon}}\bigg)$
        \item $\frac{1}{n}\|w\theta_0^\top \bR_2^\top \bR_1X\|_\infty=O\bigg(\frac{\sqrt{p}\log p}{r^{1/4}\sqrt{n\epsilon}}\|\theta_0\|_1 \bigg)$
        \item $\frac{1}{n}\|w^2\theta_0^\top (\bR_2^\top \bR_2-\bI_p)\|_\infty=O\bigg( \frac{p\sqrt{\log p}}{n\epsilon}\|\theta_0\|_1\bigg)$
        \item $\frac{1}{n}\|w\bR_3^\top \bR_1\bX\|_\infty=O\bigg(\frac{\sqrt{p}\log p}{r^{1/4}\sqrt{n\epsilon}} \bigg)$
        \item $\frac{1}{n}\|w^2\bR_3^\top \bR_2\|_\infty=O\bigg(\frac{p\sqrt{\log p}}{n\epsilon} \bigg)$
    \end{enumerate}
    By union bound and the assumption $\sqrt{s_0}\to \infty$ as $p \to \infty$, we have 
    \begin{equation*}
        \|v\|_\infty = O\bigg(\max\bigg\{\sqrt{\frac{\log p}{n}},\sqrt{\frac{(\log p )^3}{r}} , \frac{\sqrt{p}\log p}{r^{1/4}\sqrt{n\epsilon}}\|\theta_0\|_1 , \frac{p\sqrt{\log p}}{n\epsilon}\|\theta_0\|_1\bigg\}\bigg)\,.
    \end{equation*}
    By the assumption 
    \[
    \sqrt{s_0} = O\left(\frac{\max\big\{\frac{1}{\sqrt{n}},\frac{\log p }{\sqrt{r}}\big\}}{\max\big\{\frac{\sqrt{p}\log p}{r^{1/4}\sqrt{n\epsilon}},\frac{p}{n\epsilon}\big\}}\right)\,,
    \] 
    \begin{equation*}
        \frac{\sqrt{p}\log p}{r^{1/4}\sqrt{n\epsilon}}\|\theta_0\|_1 =O\bigg(\max\bigg\{\sqrt{\frac{\log p}{n}},\sqrt{\frac{(\log p )^3}{r}}\bigg\}\bigg)\,,
    \end{equation*}
    \begin{equation*}
        \frac{p\sqrt{\log p}}{n\epsilon}\|\theta_0\|_1=O\bigg(\max\bigg\{\sqrt{\frac{\log p}{n}},\sqrt{\frac{(\log p )^3}{r}}\bigg\}\bigg)\,.
    \end{equation*}
    Therefore, the probability of the event
    \[\|v\|_\infty=O\bigg(\max\bigg\{\sqrt{\frac{\log p}{n}},\sqrt{\frac{(\log p)^3}{r}}\bigg\}\bigg)\,,\]
    converges to one as $n\to\infty$.


\section{Proof of Lemma \ref{lem:marginal}}\label{pf:marginal}
\begin{proof}
    We will prove that 
    \[
     \lim_{n\to\infty}\sup_{\|\theta_0\|_0\leq s_0}\left|\Prob\left(\frac{\hth^u_j-\theta_{0,j}}{\widehat{\sigma}\rho_n}\leq x\right)-\Phi(x)\right| \leq 0\,,
    \]
    and the inequality for the opposite direction to prove equality.
    By equation~\eqref{eq:debias-rewrite}, we have
    \[
    \frac{\hth^u_j-\theta_{0,j}}{\sigma\sqrt{\bQ_{jj}}}=\frac{e_j^\top\bX^{*\top}\bR_1\xi}{n \sigma\sqrt{\bQ_{jj}}}+\frac{\Delta_j}{\sigma\sqrt{\bQ_{jj}}}\,.
    \]
    Let $\tilde{Z} = \frac{e_j^\top\bX^{*\top}\bR_1\xi}{n \sigma\sqrt{\bQ_{jj}}}$, and $V^\top = \frac{e_j^\top\bX^{*\top}\bR_1}{n\sigma\sqrt{\bQ_{jj}}}$ such that $\tilde{Z}=V^\top\xi$. We claim that $\tilde{Z}$ is $\mcN(0,1)$ in distribution. Note that $\xi\sim\mcN(0,\sigma^2)$ and $\|V\|^2_2=1/\sigma^2$. Since $V$ is independent from $\xi$, $V^\top\xi|V \sim \mcN(0,1)$. Then by conditional probability, for any $x\in\mbR$
    \begin{align*}
         \Prob(\tilde{Z}\le x|V)&=\Prob(V^\top \xi\le x|V)= \Phi(x)\,,
    \end{align*}
    where $\Phi$ is the standard Gaussian cdf. Since $V$ is independent from $\xi$, $\tilde{Z}$ is standard Gaussian $V$-almost surely, which completes the proof of claim. Now fix $\varepsilon_1>0$, without losing generality, assume $\varepsilon_1<1$
    \begin{align*}
        \Prob\left(\frac{\hth^u_j-\theta_{0,j}}{\widehat{\sigma}\rho_n}\leq x\right)&=\Prob\left(\frac{\sigma}{\widehat{\sigma}}\frac{\sqrt{\bQ_{jj}}}{\rho_n}\tilde{Z}+\frac{\Delta_j}{\widehat{\sigma}\rho_n}\leq x\right)\\
        &\leq \Prob\left(\frac{\sigma}{\widehat{\sigma}}\frac{\sqrt{\bQ_{jj}}}{\rho_n}\tilde{Z}\leq x+\varepsilon_1\right)+\Prob\left(\frac{|\Delta_j|}{\widehat{\sigma}\rho_n}>\varepsilon_1\right)\\
        &\leq \Prob\left(\frac{\sqrt{\bQ_{jj}}}{\rho_n}\tilde{Z}\leq x+\varepsilon_1|x|+2\varepsilon_1\right)+\Prob\left(\frac{|\Delta_j|}{\widehat{\sigma}\rho_n}>\varepsilon_1\right)+\Prob\left(\left|\frac{\widehat{\sigma}}{\sigma}-1\right|>\varepsilon_1\right)\\
        &\leq \Prob\left(\tilde{Z}\leq x+3\varepsilon_1|x|+4\varepsilon_1\right)+\Prob\left(\frac{|\Delta_j|}{\widehat{\sigma}\rho_n}>\varepsilon_1\right)+\Prob\left(\left|\frac{\widehat{\sigma}}{\sigma}-1\right|>\varepsilon_1\right)\\
        &\;+\Prob\left(\left|\frac{\rho_n}{\sqrt{\bQ_{jj}}}-1\right|>\varepsilon_1\right)\,.
    \end{align*}
    By Lemma~\ref{lem:rho_n}, $\bQ_{jj}/\rho_n^2\to1$ as $n\to\infty$ in probability. Also, by the consistency of the estimator $\widehat{\sigma}$, the third term on the right hand side also vanishes as $n\to\infty$. Then, we take sup and limit on both sides:
    \[
    \lim_{n\to\infty}\sup_{\|\theta_0\|_0\leq s_0,j\in[p]}\Prob\left(\frac{\hth^u_j-\theta_{0,j}}{\widehat{\sigma}\rho_n}\leq x\right) \leq\Phi(x+3\varepsilon_1|x|+4\varepsilon_1)+\lim_{n\to\infty}\sup_{\|\theta_0\|_0\leq s_0}\Prob\left(\frac{\|\Delta\|_\infty}{\widehat{\sigma}\rho_n}>\varepsilon_1\right)\,.
    \]
    By union bounded, we have
    \[
    \Prob\left(\frac{\|\Delta\|_\infty}{\widehat{\sigma}\rho_n}>\varepsilon_1\right)\leq\Prob\left(\frac{\|\Delta\|_\infty}{\sigma\rho_n}>\frac{1}{2}\varepsilon_1\right)+\Prob\left(\frac{\widehat{\sigma}}{\sigma}>\frac{1}{2}\right)\,.
    \]
    Again by the consistency of $\widehat{\sigma}$, the second term on the right hand vanishes as $n\to\infty$. Finally, by Theorem~\ref{thm:suff-cond-debiasing}, $\|\Delta\|_\infty=o(\sigma\rho_n)$ asymptotically and the first term on the right hand side also vanishes.

    Since the choice of $\varepsilon_1$ is arbitrary, the proof for the inequality is complete. The other direction of the inequality can be done in the same manner.
\end{proof}

\section{Proof of Theorem~\ref{thm:utility-limit}}\label{pf:utility-limit}
We will prove the bound for $\widehat{\FDP}$. The other two claimed bounds on FDP and Power can be proven in a similar manner. First, consider the following indicator random variables:
$ h_j = \mathbb I_{\{W_j\leq \sigma\rho_nt_0\}}$, $ g_j= \mathbb I_{\{W_j\geq \sigma\rho_nt_0\}}$. Then we can re-write $\widehat{\FDP}$ in terms of $h_j$ and $g_j$. Recall that $S_0 =\{j\in[p]:\theta_{0,j}\neq0\}$, we have
\[
\widehat{\FDP}(\sigma\rho_nt_0) = \frac{1+\sum_{j\in S_0}h_j+\sum_{j\notin S_0}h_j}{(\sum_{j\in S_0}g_j+\sum_{j\notin S_0}g_j)\vee 1}\,.
\]
Notice that $f(Z;1)$, $f(Z';1)$ are i.i.d. copies of each other, so
\[
f(Z;1)-f(Z';1)\stackrel{d}{=}f(Z';1)-f(Z;1)\,,
\]
which implies that $\Prob(f(Z;1)-f(Z';1)\leq-t_0)=\Prob(f(Z;1)-f(Z';1)\geq t_0)$. We denote this probability by $P_0$ for simplicity. Also, we write $P_{1,\geq} := \Prob(f(\mu_0+Z;1)-f(Z';1)\geq t_0)$ and $P_{1,\leq} := \Prob(f(\mu_0+Z;1)-f(Z';1)\leq -t_0)$. Therefore,
\[
\widehat{\alpha}(\mu_0,t_0) = \frac{c_0P_{1,\leq}+(1-c_0)P_0}{c_0P_{1,\geq}+(1-c_0)P_0}\,.
\]
To prove the upper bound for $\widehat{\FDP}$, it suffices to show that
\begin{equation}\label{eq:FDPh-num1}
    \limsup_{n\to\infty}\frac{\frac{1}{s_0}\sum_{j\in S_0}h_j}{P_{1,\leq}}\leq 1\,,
\end{equation}
\begin{equation}\label{eq:FDPh-num2}
    \limsup_{n\to\infty}\frac{\frac{1}{p-s_0}\sum_{j\notin S_0}h_j}{P_0}\leq 1\,,
\end{equation}
\begin{equation}\label{eq:FDPh-dnom1}
    \liminf_{n\to\infty}\frac{\frac{1}{s_0}\sum_{j\in S_0}g_j}{P_{1,\geq}}\geq 1\,,
\end{equation}
\begin{equation}\label{eq:FDPh-dnom2}
    \quad \liminf_{n\to\infty}\frac{\frac{1}{p-s_0}\sum_{j\notin S_0}g_j}{P_0}\geq 1\,.
\end{equation}
To start with, we want to prove \eqref{eq:FDPh-num1}. The other three inequalities \eqref{eq:FDPh-num2}, \eqref{eq:FDPh-dnom1}, \eqref{eq:FDPh-dnom2} can be proved in a similar manner. Note that if $W_j\leq-\sigma\rho_nt_0$, by definition,
\[
f(\hth^u_j;\sigma\rho_n)-f(\hth^u_{j+p};\sigma\rho_n)\leq -\sigma\rho_nt_0\,.
\]
According to the decomposition in \eqref{eq:decompo-Z}, letting $\bLambda^\top = \bX^{*\top}\bR_1/n$ we have
\[
f(\theta_{0,j}+\bLambda_j^\top\xi+\Delta_j;\sigma\rho_n)-f(\theta_{0,j+p}+\bLambda_{j+p}^\top\xi+\Delta_{j+p};\sigma\rho_n)\leq -\sigma\rho_nt_0\,.
\]
By assumption, $f$ is $L$-Lipschitz, so we have 
\[
f(\theta_{0,j}+\bLambda_j^\top\xi;\sigma\rho_n)-f(\bLambda_{j+p}^\top\xi;\sigma\rho_n)\leq -\sigma\rho_nt_0+2L\|\Delta\|_\infty\,.
\]
In addition, $f$ preserves scalar multiplication, so we have
\[
f\left(\frac{\theta_{0,j}+\bLambda_j^\top\xi}{\sigma\rho_n};1\right)-f\left(\frac{\bLambda_{j+p}^\top\xi}{\sigma\rho_n};1\right)\leq -t_0+\frac{2L\|\Delta\|_\infty}{\sigma\rho_n}\,.
\]
By Theorem~\ref{thm:suff-cond-debiasing}, $\|\Delta\|_\infty=o(\sigma\rho_n)$. Specifically, for any $\varepsilon>0$ and $\lim_{n\to \infty}\Prob(\|\Delta\|_\infty/\sigma\rho_n\geq\varepsilon) =0$. Then conditioning on $\|\Delta\|_\infty/\sigma\rho_n<\varepsilon$,
\[
f\left(\frac{\theta_{0,j}+\bLambda_j^\top\xi}{\sigma\rho_n};1\right)-f\left(\frac{\bLambda_{j+p}^\top\xi}{\sigma\rho_n};1\right)\leq -t_0+\varepsilon\,.
\]
Let $h'_j$ be an indicator supported on the event above:
\begin{equation}\label{eq:h'}
    h'_j = \mathbb I_{\left\{f\left(\frac{\theta_{0,j}+\bLambda_j^\top\xi}{\sigma\rho_n};1\right)-f\left(\frac{\bLambda_{j+p}^\top\xi}{\sigma\rho_n};1\right)\leq -t_0+\varepsilon\right\}}\,.
\end{equation}
Therefore, $h_j\leq h'_j$.

Next, we want to show that $\frac{1}{s_0}\sum_{j\in S_0}\frac{h'_j}{\Expected(h'_j)}\to 1$ in probability using Chebyshev inequality. To do that, we claim the following lemma:
\begin{lemma}\label{lem:h'j-Cov}
    Given $\xi\sim \mcN(0,\sigma^2\bI_n)$, and $\{h'_j:j\in[p]\}$ defined in \eqref{eq:h'}, then $\max_{i\neq j}\Cov(h'_i;h'_j)=o(1)$ with high probability.
\end{lemma}
Without lost of the generality, we assume that for any $i,j\in S_0$, 
\[
\liminf_{n\to\infty}\Expected(h'_i)\Expected(h'_j)>0\,,
\]
and we will discuss the other case at the end of the proof (Note that $h'_i$'s are indicators and so $\Expected(h'_i)\ge0$, so the above assumption for now is precluding the case of zero expectation). Now, by Lemma~\ref{lem:h'j-Cov}, we have
\begin{align*}
    \Var\left(\frac{1}{s_0}\sum_{j\in S_0}\frac{h'_j}{\Expected(h'_j)}\right) &= \frac{1}{s_0^2}\sum_{i,j\in S_0}\frac{\Cov(h'_i;h'_j)}{\Expected(h'_i)\Expected(h'_j)}\\
    &\leq  \frac{1}{s_0^2}\left(\sum_{i,j\in S_0,i\neq j}\frac{o(1)}{\Expected(h'_i)\Expected(h'_j)}+\sum_{j\in S_0}\frac{1}{\Expected(h'_j)^2}\right)\\
    &\leq \frac{1}{s_0^2}\cdot\left( s_0^2\cdot o(1)+O(s_0)\right)\\
    &=  o(1)\,.
\end{align*}
By Chebyshev inequality, 
\[
\frac{1}{s_0}\sum_{j\in S_0}\frac{h'_j}{\Expected(h'_j)}\to 1\,.
\]
Now, note that 
\[
\Expected(h'_j) = \Prob\left(f\left(\frac{\theta_{0,j}}{\sigma\rho_n}+\frac{\bLambda_j^\top\xi}{\sigma\rho_n};1\right)-f\left(\frac{\bLambda_{j+p}^\top\xi}{\sigma\rho_n};1\right)\leq -t_0 +\varepsilon\right)\,.
\]
By Lemma~\ref{lem:marginal-gen}, the joint distribution of $(\bLambda_j^\top \xi/\sigma\rho_n,\bLambda_{j+p}\xi/\sigma\rho_n)$ converges to $\mcN({{\bf{0}}},\bI_2)$. Therefore,  
\begin{align*}
    \limsup_{n\to\infty}\left|\Expected(h'_j)-\Prob\left(f\left(\frac{\theta_{0,j}}{\sigma\rho_n}+Z;1\right)-f\left(Z';1\right)\leq -t_0 +\varepsilon\right)\right|=0\,.
\end{align*}
Since the cumulative distribution functions are right-continuous and $\varepsilon$ can be arbitrarily small,
\[
\limsup_{n\to\infty}\left|\Expected(h'_j)-\Prob\left(f\left(\frac{\theta_{0,j}}{\sigma\rho_n}+Z;1\right)-f\left(Z';1\right)\leq -t_0\right)\right|=0\,.
\]
Now, we claim the following lemma to close the gap from $\Prob(f(\theta_{0,j}/\sigma\rho_n+Z)-f(Z';1)\le t_0)$ to $P_{1,\leq}$. Recall that $P_{1,\leq}= \Prob(f(\mu_0+Z;1)-f(Z';1)\leq -t_0)$. We claim the following Lemma:
\begin{lemma}\label{lem:prob-bounds-mu}
    Suppose that $|\theta_{0,j}|\geq \mu_n$. Then,
    \begin{enumerate}
        \item $\Prob(f(\theta_{0,j}/\sigma\rho_n+Z;1)-f(Z';1)\leq -t_0)\leq \Prob(f(\mu_0+Z;1)-f(Z';1)\leq -t_0)$
        \item $\Prob(f(\theta_{0,j}/\sigma\rho_n+Z;1)-f(Z';1)\geq t_0)\geq \Prob(f(\mu_0+Z;1)-f(Z';1)\geq t_0)$
    \end{enumerate}
\end{lemma}
Putting everything together along with Lemma~\ref{lem:prob-bounds-mu}, we have that
\begin{align*}
    \limsup_{n\to\infty}\frac{\frac{1}{s_0}\sum_{j\in S_0} h_j}{P_{1,\leq}}&\leq \limsup_{n\to\infty}\frac{1}{s_0}\sum_{j\in S_0} \frac{h'_j}{P_{1,\leq}}\\
    &\leq \limsup_{n\to \infty}\frac{1}{s_0}\sum_{j\in S_0}\frac{h'_j}{\Prob(f(\theta_{0,j}/\sigma\rho_n+Z;1)-f(Z';1)\leq -t_0)}\\
    &\leq  \limsup_{n\to\infty}\frac{1}{s_0}\sum_{j\in S_0}\frac{h'_j}{\Expected(h'_j)}\\
    &= 1\,.
\end{align*}
This completes the proof for \eqref{eq:FDPh-num1}. With a similar manner, we can prove \eqref{eq:FDPh-num2}, \eqref{eq:FDPh-dnom1}, and \eqref{eq:FDPh-dnom2}. \eqref{eq:FDPh-num1}, \eqref{eq:FDPh-num2}, \eqref{eq:FDPh-dnom1}, and \eqref{eq:FDPh-dnom2} complete the proof of \eqref{eq:FDP-hat-limit}. Inequalities \eqref{eq:FDP-limit} and \eqref{eq:power-limit} can be proved in the same way and hence omitted.

As for the case when $\liminf_{n\to\infty}\Expected(h'_i)\Expected(h'_j)=0$, then without losing of generality, we assume that $\liminf_{n\to\infty}\Expected(h'_i)=0$, then $h'_i\to 0$ in probability because $h'_i\ge 0$. Since $\mu_0<\infty$, $P_{1,\le}\neq 0$. Therefore, 
\[
\limsup_{n\to\infty}\frac{h_i}{P_{1,\le}}\le \limsup_{n\to\infty}\frac{h'_i}{P_{1,\le}}=0\le1\,.
\]
Hence we obtain the same result.

\subsection{Proof of Lemma~\ref{lem:h'j-Cov}}
We want to show that $\max_{i\neq j}\Cov(h'_i;h'_j) = o(1)$. For any $j\in[p]$, we denote by 
\[
U_j:= (\bLambda_j\xi/\sigma\rho_n,\bLambda_{j+p}\xi/\sigma\rho_n)\,,
\]
the joint distribution of the noise terms. Recall from Theorem~\ref{thm:debiased-rewrite},
\[
\bQ=\frac{1}{n^2}\bX^{*\top}\bR_1\bR_1^\top\bX^*=\bLambda^\top\bLambda\,.
\]
Note that $\bQ$ is independent from $\xi$. Conditioning on $\bQ$, $\bLambda\xi/\sigma\rho_n|\bQ \sim \mcN(0, \bQ/\rho_n^2)$, so 
\[
(U_i,U_j)|\bQ \sim \mcN(0;\widetilde{\bQ})\,,\quad\widetilde{\bQ}=\begin{bmatrix}
    \widetilde{\bQ}_{i,i} & \widetilde{\bQ}_{i,j}\\
    \widetilde{\bQ}_{j,i} & \widetilde{\bQ}_{j,j}
\end{bmatrix}\,,
\]
where $\widetilde{\bQ}_{i,j}$ is the sub-matrix of $\bQ/\rho_n^2$ at the rows $\{i,i+p\}$, and the columns $\{j,j+p\}$. We can regard $h'_j$ as a function of $U_j$. Specifically, $h'_j = h'_j(U_j)$. By the law of total variance,
\begin{align}\label{eq:law-total-var}
    &\quad\,\, |\Cov(h'_i;h'_j)|\nonumber \\
    &=|\Expected(\Cov(h'_i;h'_j|\bQ)) + \Cov(\Expected(h'_i|\bQ);\Expected(h'_j|\bQ))|\nonumber\\
    &\le\left|\Expected\left(\Corr(h'_i;h'_j|\bQ)\cdot\sqrt{\Var(h'_i|\bQ)\Var(h'_j|\bQ)}\right)\right| + \left|\sqrt{\Var(\Expected(h'_i|\bQ))\Var(\Expected(h'_j|\bQ))}\right|
\end{align}
The last inequality follows from the definition of correlation and the fact that $|\Cov(X;Y)|\le\sqrt{\Var(X)\Var(Y)}$ (Cauchy-Schwartz inequality for Covariance). For the first term on the right-hand side, since $\bQ^{1/2}\xi|\bQ$ is jointly Gaussian, the correlation is maximized by linear function which coincides with the canonical correlation:
\[
\Corr(h'_i;h'_j)\le \kappa_{\max} := \max_{a,b}\Corr(a^\top U_i;b^\top U_j)\,.
\]
For arbitrary $\varepsilon>0$, suppose $|\bQ/\rho_n-\bI_{2p}|_\infty \le \varepsilon$, then 
\begin{align*}
    \kappa_{\max} &= \left\|\widetilde{\bQ}_{i,i}^{-1/2}\widetilde{\bQ}_{i,j}\widetilde{\bQ}_{j,j}^{-1/2}\right\|_2\\
    &\le  \left\|\widetilde{\bQ}_{i,i}^{-1/2}\right\|_2\left\|\widetilde{\bQ}_{i,j}\right\|_2\left\|\widetilde{\bQ}_{j,j}^{-1/2}\right\|_2\\
    &\le  \frac{1}{\sqrt{\sigma_{\min}(\widetilde{\bQ}_{i,i})}} \cdot\left\|\widetilde{\bQ}_{i,j}\right\|_F\cdot \frac{1}{\sqrt{\sigma_{\min}(\widetilde{\bQ}_{j,j})}}\\
    &\le \frac{1}{\sqrt{1-2\varepsilon}}\cdot 2\varepsilon\cdot \frac{1}{\sqrt{1-2\varepsilon}}\\
    &=  \frac{2\varepsilon}{1-2\varepsilon}\,.
\end{align*}
By Lemma~\ref{lem:rho_n}, $\Prob(|\bQ/\rho_n^2-\bI_p|_\infty > \varepsilon)\to 0$. We denote by $F_n:=\{|\bQ/\rho_n^2-\bI_p|_\infty > \varepsilon\}$, then
\begin{align}\label{eq:total-cov-1}
    &\quad\Expected(\Cov(h'_i;h'_j|\bQ))\nonumber \\
    &= \Expected(\Cov(h'_i;h'_j|\bQ)|F_n)\Prob(F_n) + \Expected(\Cov(h'_i;h'_j|\bQ)|F_n^c)\Prob(F_n^c)\nonumber\\
    &=  \Expected(\Cov(h'_i;h'_j|\bQ)|F_n^c)-\Expected(\Cov(h'_i;h'_j|\bQ)|F_n^c)\Prob(F_n)+ \Expected(\Cov(h'_i;h'_j|\bQ)|F_n)\Prob(F_n)\nonumber\\
    &\le  \Expected(\Cov(h'_i;h'_j|\bQ)|F_n^c) + 2\Prob(F_n)\nonumber\\
    &\le  \frac{2\varepsilon}{1-2\varepsilon} + o(1)\,.
\end{align}
The first inequality follows from the fact that $|\Cov(h'_i;h'_j|\bQ)|\le 1$. 

For $\Var(\Expected(h_i'|\bQ))$, denote by $A(\bQ):= \Expected(h'_i|\bQ)$. Recall that we can regard $h'_i = h'_i(U_i)$, where $U_i|\bQ \sim \mcN\left(0,\widetilde{\bQ}_{i,i}\right)$. Let $Z_1\sim\mcN({{\bf{0}}},\bI_2)$. 
Recall that the total variation distance between two distributions $\mathcal{D}_1$ and $\mathcal{D}_2$ on $\mathcal X$ with density functions $p_1(x)$ and $p_2(x)$ is defined as
\[
\mbox{TV}(\mathcal{D}_1,\mathcal{D}_2) = \frac{1}{2}\int_{\mathcal X}|p_1(x)-p_2(x)|dx\,.
\]
Now we suppose $U_i|\bQ$ has density function $\phi'(x)$ and $Z_1$ has the standard Gaussian density on $\mbR^2$, denoted by $\phi_2(x)$, then because $|h'_i(x)|\le 1$.
\begin{align*}
    |A(\bQ)-A(\bI_n)|&=|\Expected(h'_i(U_i)|\bQ)-\Expected(h'_i(Z_1))|\\
        &=|\Expected(h'_i(U_i|\bQ))-\Expected(h'_i(Z_1))|\\
        &=\left|\int_{\mbR^2}\left(h'_i(x)\phi'(x)-h'_i(x)\phi_2(x)\right)dx\right|\\
        &\le  \int_{\mbR^2}|h'_i(x)|\left|\phi'(x)-\phi_2(x)\right|dx\\
        &\le  2\mbox{TV}(\mcN(0,\widetilde{\bQ}_{i,i})|\mcN(\bf{0},\bI_2))\,.
\end{align*}
Since $h'_i$ is bounded, applying Pinsker's inequality, we have 
\begin{align*}
    |A(\bQ)-A(\bI_n)|  &\le  2\mbox{TV}(\mcN(0,\widetilde{\bQ}_{i,i})|\mcN(\bf{0},\bI_2))\\
        &\le 2\sqrt{\frac{1}{2}\mbox{KL}(\mcN(0,\widetilde{\bQ}_{i,i})||\mcN({{\bf{0}}},\bI_2))}\\
        &\le 2(\mbox{tr}(\widetilde{\bQ}_{i,i}^{-1})-2+\log(\det(\widetilde{\bQ}_{i,i}))
\end{align*}
Again we assume that $|\bQ/\rho_n-\bI_{2p}|_\infty \le \varepsilon$. Without losing of the generality, assume $\varepsilon<1/2$, then $\|\widetilde{\bQ}_{i,i}-\bI_2\|_2\le 2\varepsilon<1$. Then the following series converges:
\[
\widetilde{\bQ}_{i,i}^{-1} = \sum_{k=0}^\infty (\bI_2 - \widetilde{\bQ}_{i,i})^k
\]
Therefore, 
\begin{align}\label{eq:total-cov-2.1}
    \mbox{tr}(\widetilde{\bQ}_{i,i}^{-1}) &= 2+\sum_{k=1}^\infty\mbox{tr}(\bI_2 - \widetilde{\bQ}_{i,i})^k\nonumber\\
    &\le 2+ \sum_{k=1}^\infty2(2\varepsilon)^k\nonumber\\
    &=  2+\frac{4\varepsilon}{1-2\varepsilon}
\end{align}
On the other hand, since $|\bQ/\rho_n^2-\bI_2|_\infty\le \varepsilon$ and $\widetilde{\bQ}_{i,i}\in\mbR^{2\times 2}$, $|\det(\widetilde{\bQ}_{i,i})-1|\leq 2(\varepsilon+\varepsilon^2)$. Without losing of the generality, assume $\varepsilon$ is small enough such that $2(\varepsilon+\varepsilon^2)<1$. Then since $\log(1+x)\le x$ for any $x>-1$, 
\begin{equation}\label{eq:tota-cov-2.2}
    \log(\det(\widetilde{\bQ}_{i,i})) = \log(1+(\det(\widetilde{\bQ}_{i,i})-1)) \le 2(\varepsilon+\varepsilon^2)\,. 
\end{equation}
Putting together \eqref{eq:total-cov-2.1} and \eqref{eq:tota-cov-2.2}, we have that 
\[
|A(\bQ)-A(\bI_2)| \le 2 \left(\frac{4\varepsilon}{1-2\varepsilon}+2(\varepsilon+\varepsilon^2)\right)\,.
\]
Recall that $F_n:=\{|\bQ/\rho_n^2-\bI_p|_\infty > \varepsilon\}$. By the fact that $\Expected(X) = \arg\min_{x\in \mbR} \Expected((X-x)^2)$, then
\begin{align*}
    \Expected((A(\bQ)-\Expected(A(\bQ)))^2|F_n^c)\le 4\left(\frac{4\varepsilon}{1-2\varepsilon}+2(\varepsilon+\varepsilon^2)\right)^2\,.
\end{align*}
Recall that $A(\bQ) = \Expected(h'_i|\bQ)$, then by the law of total expectation, we have
\begin{align*}
    \Var(\Expected(h'_i|\bQ)) &= \Expected((\Expected(h'_i|\bQ)-\Expected(h'_i))^2)\\
        &= \Expected((\Expected(h'_i|\bQ)-\Expected(h'_i))^2|F_n)\Prob(F_n)+\Expected((\Expected(h'_i|\bQ)-\Expected(h'_i))^2|F_n^c)\Prob(F_n^c)\\
        &\le \Prob(F_n)+4\left(\frac{4\varepsilon}{1-2\varepsilon}+2(\varepsilon+\varepsilon^2)\right)^2\\
        &= 4\left(\frac{4\varepsilon}{1-2\varepsilon}+2(\varepsilon+\varepsilon^2)\right)^2+o(1)\,.
\end{align*}
Since $\varepsilon$ can be arbitrarily small, $\Var(\Expected(h'_i|\bQ))=o(1)$, for the same reason, $\Var(\Expected(h'_j|\bQ))=o(1)$.
Therefore, along with \eqref{eq:law-total-var} and \eqref{eq:total-cov-1}, we have that $\Cov(h'_i;h'_j)=o(1)$. Since the choices of $i$, $j$ are arbitrary, this completes the proof of the lemma.

\subsection{Proof of Lemma~\ref{lem:prob-bounds-mu}}
    For the first inequality, it suffices to show that the probability $\Prob(f(\mu+Z;1)-f(Z';1)\leq -t)$ is decreasing in $\mu$. Denote $X:=f(\mu+Z;1), Y := f(Z';1)$
    \begin{align*}
        \Prob(f(\mu+Z;1)-f(Z';1)\leq -t) &= \Prob(X\leq Y-t)\\
            &= \int_t^\infty \Prob(X\leq y-t)f_Y(y)dy\,.
    \end{align*}
    Notice $Y$ does not depend on $\mu$ so we are going to complete the proof by showing the cdf of $X$ is decreasing as $\mu$ increases. Since $f(\cdot;1)$ is Lipschitz-continuous and non-decreasing on $[0,\infty)$, for any $x\ge f(0;1) $, there exists $s_x := \max\{s>0:f(s;1)\le x\}$. Since $f(\cdot;1)$ is non-decreasing on $[0,\infty)$ and $f(x) = f(-x)$, we have that $f(
    \mu+Z)\le x \iff |\mu+Z|\le s_x$. So,
    \begin{align*}
        \Prob(X\leq x)  &= \Prob(|\mu+Z|\leq s_x)\\
        &= \Prob(\mu+Z\le s_x)-\Prob(\mu+Z\le -s_x)\\
        &=\Phi( s_x-\mu)-\Phi(-s_x-\mu)\,.
    \end{align*}
    where $\Phi$ denotes the standard Gaussian cdf. We consider the right-hand side of the equation as a function of $\mu$. By taking the derivative, we have
    \[
    \frac{d}{d\mu}(\Phi(s_x-\mu)-\Phi(-s_x-\mu))=\phi(-s_x-\mu)-\phi(s_z-\mu)\,,
    \]
    where $\phi$ denotes the standard Gaussian pdf. From the equation above we observe that when $\mu=0$, the derivative equals 0; when $\mu\neq 0$, the sign of the derivative is the opposite of that of $\mu$. Therefore, we can conclude that the above probability is maximized when $\mu=0$, meanwhile decreasing on $[0,\infty)$ and increasing on $(-\infty,0]$. So $\Prob(X\le x)$ is decreasing over $\mu\in[0,\infty)$.

    The second proof can be done in a similar manner, hence omitted.
\section{Proof of Corollary \ref{cor:power1}}\label{pf:power1}

We first prove the second statement. Consider $\sigma\rho_n=o(t_n)$ and $t_n = o(\mu_n)$. Using Theorem~\ref{thm:utility-limit} it suffices to show that 
\begin{align*}
&\lim\sup_{n\to\infty}\alpha(\mu_n/\sigma\rho_n,t_n/\sigma\rho_n) = 0\,,\\
&\lim\inf_{n\to\infty} \beta(\mu_n/\sigma\rho_n,t_n/\sigma\rho_n)=1\,.
\end{align*}
To show these, by the definition of $\alpha(\mu/(\sigma\rho_n),t/(\sigma\rho_n))$ and $\beta(\mu/(\sigma\rho_n),t/(\sigma\rho_n))$, it suffices to show that
\begin{align*}
&\lim\sup_{n\to\infty}\Prob\left(f(Z;1)-f(Z';1)\geq \frac{t_n}{\sigma\rho_n}\right) = 0\,,\\
&\lim\inf_{n\to\infty} \Prob\left(f\left(\frac{\mu_n}{\sigma\rho_n}+Z;1\right)-f\left(Z';1\right)\geq \frac{t}{\sigma\rho_n}\right)=1\,,
\end{align*}
where $Z,Z'\sim\mathcal{N}(0,1)$. For the first equation, under the conditions that $\lim_{n\to\infty}\frac{t_n}{\sigma\rho_n}= \infty$, this probability converges to zero. 
For the other equation, we let $\nu$ be the probability measure of $f(Z';1)$. By the assumptions of $f$ in Assumption~\ref{ass:f-feat-stat}, there exists some $c_1>0$ such that
\begin{align*}
&\quad\lim_{n\to\infty}\Prob\left(f\left(\frac{\mu_n}{\sigma\rho_n}+Z;1\right)-f\left(Z';1\right)\geq \frac{t_n}{\sigma\rho_n}\right)\\
&= \lim_{n\to\infty}\Prob\left(f\left(\frac{\mu_n}{\sigma\rho_n}+Z;1\right)\ge \frac{t_n}{\sigma\rho_n}+f\left(Z';1\right)\right)\\
&\ge \lim_{n\to\infty}\int\Prob\left(c_1\left(\frac{\mu_n}{\sigma\rho_n}+Z\right)\ge\frac{t_n}{\sigma\rho_n}+z\right)d\nu(z)\\
&= \int\lim_{n\to\infty}\Prob\left(Z\ge \frac{c_1^{-1}t_n-\mu_n}{\sigma\rho_n}+c_1^{-1}z\right)d\nu(z)\\
&= 1
\end{align*}
The last equality follows from the conditions that $\sigma\rho_n=o(t_n)$ and $t_n=o(\mu_n)$.

For the first statement, consider $t_n:=\min\{t\in\mathcal{W}:\widehat{\FDP}(t)\le q\}$. $\FDR(t_n)\le q$ is proven in Theorem~\ref{thm:FDR}. For $\Power(t_n)$, let $t_*$ be defined as the following:
\begin{align}\label{eq:t*}
t_*:=\min\left\{t>0:\frac{(1-c_0)\Prob(f(Z;1)-f(Z';1)\le -t)}{c_0+(1-c_0)\Prob(f(Z;1)-f(Z';1)\ge t)}\le q_*\right\}\,,
\end{align}
such that $q_*<q$ and $t_*>0$. Note that $\Prob(f(Z;1)-f(Z';1)\le -t_0) = \Prob(f(Z;1)-f(Z';1)\ge t_0)$ which is decreasing in $t_0$. Hence the fraction in~\eqref{eq:t*} is a decreasing function of $t$ and goes to zero as $t\to\infty$. This is implies that such $t_*$ and $q_*$ exist. We next construct a deterministic threshold based on $t_*$ which serves as an upper bound for the data-dependent knock-oof threshold $t_n$.

By Theorem~\ref{thm:utility-limit}, we have
\begin{align*}
    &\quad\limsup_{n\to\infty}\widehat{\FDP}(\sigma\rho_nt_*)\\
    &\le  \limsup_{n\to\infty}\widehat{\alpha}\left(\frac{\mu_n}{\sigma\rho_n},t_*\right)\\
    &=  \limsup_{n\to\infty}\frac{c_0\Prob\left(f\left(\frac{\mu_n}{\sigma\rho_n}+Z;1\right)-f\left(Z';1\right)\leq -t_*\right)+(1-c_0)\Prob(f(Z;1)-f(Z';1)\leq -t_*)}{c_0\Prob\left(f\left(\frac{\mu_n}{\sigma\rho_n}+Z;1\right)-f\left(Z';1\right)\ge t_*\right)+(1-c_0)\Prob(f(Z;1)-f(Z';1)\geq t_*)}\,.
\end{align*}
Under the condition that $\lim_{n\to\infty} \mu_n/(\sigma\rho_n)=\infty$, we have
\[
\limsup_{n\to\infty}\widehat{\FDP}(\sigma\rho_nt_*)\le \frac{(1-c_0)\Prob(f(Z;1)-f(Z';1)\leq -t_*)}{c_0+(1-c_0)\Prob(f(Z;1)-f(Z';1)\geq t_*)}\le q_*<q\,,
\]
which means that for $n$ large enough, $\widehat{\FDP}(\sigma\rho_nt_*)\le q$. This implies that the set $\{t\in\mathcal{W}:\widehat{\FDP}(t)\le q\}$ is non-empty and $t_n<\infty$. Consider the set $T_n:=\{j\in[p]:|W_j|\ge\sigma\rho_nt_*\}$. If $T_n$ is non-empty, then there exists $k\in[p]$ such that $|W_k| = \min (T_n)$. Note that
\begin{align*}
    \{j\in[p]:W_j\ge \sigma\rho_nt_*\} &= \{j\in[p]:W_j\ge |W_k|\} \\
    \{j\in[p]:W_j\le -\sigma\rho_nt_*\} &= \{j\in[p]:W_j\le -|W_k|\}\,.
\end{align*}
Hence $\widehat{{\FDP}}(\sigma\rho_nt_*) = \widehat{\FDP}(|W_k|)<q$ and $\Power(\sigma\rho_nt_*)=\Power(|W_k|)$. Then by the definition of $t_n$, $|W_k|\ge t_n$. If $T_n$ is empty, then $\sigma\rho_nt_*\ge\max(\mathcal{W})\ge t_n$ on the event $t_n<\infty$. Now we construct the following sequence
\[
t'_n = \begin{cases}
    \sigma\rho_nt_* & \mbox{if }T_n=\emptyset\,\\
    \min(T_n) & \mbox{if }T_n\neq \emptyset\,.
    \end{cases}
\]
By this construction, $t_n\le t_n'$ and $\liminf_{n\to\infty}\Power(t_n')=\liminf_{n\to\infty}\Power(\sigma\rho_nt_*)$. Since $\Power(t)$ increases as $t$ decreases, by Theorem~\ref{thm:utility-limit}, we have 
\begin{align*}
    \liminf_{n\to\infty}\Power(t_n)&\ge\liminf_{n\to\infty}\Power(t_n')\\
        &=\liminf_{n\to\infty}\Power(\sigma\rho_n t_*)\\
        &\ge\Prob\left(f\left(\frac{\mu_n}{\sigma\rho_n}+Z;1\right)-f\left(Z';1\right)\ge t_*\right)\,, 
\end{align*}
where the last probability converges to one as $\liminf_{n\to\infty}\mu_n/(\sigma\rho_n)=\infty$.
\section{Proof of Lemma~\ref{lem:lambda-bigger-than-rho_n}}\label{pf:lambda-bigger-than-rho_n}
Recall the condition~\ref{item:cond-error} in~\ref{thm:suff-cond-debiasing}:
\[
\frac{1}{\sigma\rho_n}\frac{4\kappa_n^2}{\lambda\phi_w} = o\bigg(\min\bigg\{\sqrt{\frac{n}{\log p}},\sqrt{\frac{r}{\log p}},\frac{n}{w^2}\sqrt{\frac{r}{\log p}},\frac{\sqrt{nr}}{w\log p}\bigg\}\bigg)\,.
\]
Focusing on the first and second terms on the right-hand side, we have
\[
\frac{1}{\sigma\rho_n}\frac{4\kappa_n^2}{\lambda\phi_w} = o\bigg(\min\bigg\{\sqrt{\frac{n}{\log p}},\sqrt{\frac{r}{\log p}}\bigg\}\bigg)\,.
\]
Substituting $\rho_n^2 = (n+r+1)/(nr)+w^2/(nr)$, $\phi_w = 1+w^2/n$, $\kappa_n = 2\lambda\sqrt{s_0}+w^2/n$, we have
\[
\left(2\lambda\sqrt{s_0}+\frac{w^2}{n}\right)^2 = o\left(\min\bigg\{\sqrt{\frac{n}{\log p}},\sqrt{\frac{r}{\log p}}\bigg\}\cdot\left(\frac{n+r+1+w^2}{nr}\right)^{1/2}\cdot \lambda\left(1+\frac{w^2}{n}\right) \right)\,.
\]
Focusing on the second term on the left-hand side and substituting $\lambda = \sigma C_\lambda \max\left\{\sqrt{\frac{\log p}{n}},\sqrt{\frac{(\log p)^3}{r}}\right\}$, we have 
\begin{align*}
    \frac{\left(\frac{w^2}{n}\right)^2}{1+\frac{w^2}{n}} &= o\left(\min\bigg\{\sqrt{\frac{n}{\log p}},\sqrt{\frac{r}{\log p}}\bigg\}\cdot\sqrt{\frac{n+r+1+w^2}{nr}}\cdot\sigma\max\left\{\sqrt{\frac{\log p}{n}},\sqrt{\frac{(\log p)^3}{r}}\right\}\right)\\
    &=o\left(\max\left\{\min\bigg\{\sqrt{\frac{r}{n}},\log p\bigg\},\min\bigg\{1,\sqrt{\frac{n(\log p)^2}{r}}\bigg\}\right\}\cdot\sqrt{\frac{n+r+1+w^2}{nr}}\cdot\sigma\right)\\
    &=o\left(\log p\cdot\sqrt{\frac{n+r+1+w^2}{nr}}\cdot\sigma\right)\,.
\end{align*}
Note that $\sqrt{\frac{n+r+1}{nr}}=o(1)$, and if
\begin{align*}
    &\quad \ \frac{\left(\frac{w^2}{n}\right)^2}{1+\frac{w^2}{n}} =o\left(\log p \cdot\sqrt{\frac{w^2}{nr}} \cdot \sigma\right)\\
    &\implies  \frac{\left(\frac{w^2}{n}\right)^{3/2}}{1+\frac{w^2}{n}} =o\left(\log p \cdot\sqrt{\frac{1}{r}} \cdot \sigma\right)\,.
\end{align*}
By the assumption $\sigma =\Theta(1)$, from the above bound  we conclude that $\frac{w^2}{n}=o(1)$. Hence, 
\[
\sigma\rho_n = \sigma\sqrt{\frac{1}{r}+\frac{1}{n}+\frac{w^2+1}{nr}}=O\left(\sigma\sqrt{\frac{1}{n}+\frac{1}{r}}\right)=o(\lambda)\,.
\]
\section{Proof of Corollary~\ref{cor:utility-LCD}}\label{pf:utility-LCD}
The structure of this proof is similar to that of Theorem~\ref{thm:utility-limit}. First we show that $\lim_{n\to\infty}\widehat \FDP(t_n) = 0$. Recall that $j\in[p]$, $W_j :=|\eta(\hth^u_j;\lambda)|-|\eta(\hth^u_{j+p};\lambda)|$. Define the indicator functions: $ h_j = \mathbb I_{\{W_j\leq -t_n\}}$, $ g_j= \mathbb I_{\{W_j\geq t_n\}}$. We then have
\[
\widehat{\FDP}(t_n) = \frac{1+\sum_{j\in S_0}h_j+\sum_{j\notin S_0}h_j}{(\sum_{j\in S_0}g_j+\sum_{j\notin S_0}g_j)\vee 1}\,.
\]
It suffices to show that 
\[
\frac{1}{p}\sum_{j\in S_0}h_j\to0\,,\quad \frac{1}{p}\sum_{j\notin S_0}h_j\to0\,.
\]
Note that if $W_j\leq-t_n$, by definition,
\[
\frac{1}{1+\frac{w^2}{n}}(|\hth^u_j|-\lambda)_+-\frac{1}{1+\frac{w^2}{n}}(|\hth^u_{j+p}|-\lambda)_+\leq -t_n\,.
\]
As shown in the proof of Lemma~\ref{lem:lambda-bigger-than-rho_n}, $w^2/n = o(1)$. Then multiplying both side by $1+\frac{w^2}{n}$, we have
\[
(|\hth^u_j|-\lambda)_+-(|\hth^u_{j+p}|-\lambda)_+\leq -\left(1+\frac{w^2}{n}\right)t_n\,.
\]
According to the decomposition in \eqref{eq:decompo-Z}, letting $\bLambda^\top = \bX^{*\top}\bR_1/n$ we have
\[
(|\theta_{0,j}+\bLambda_j^\top\xi+\Delta_j|-\lambda)_+-(|\bLambda_{j+p}^\top\xi+\Delta_{j+p}|-\lambda)_+\leq -\left(1+\frac{w^2}{n}\right)t_n\,.
\]
Note that $(x-\lambda)_+$ is 1-Lipschitz, which implies
\[
(|\theta_{0,j}+\bLambda_j^\top\xi|-\lambda)_+-(|\bLambda_{j+p}^\top\xi|-\lambda)_+\leq -\left(1+\frac{w^2}{n}\right)t_n+2\|\Delta\|_\infty\,.
\]
Dividing both sides by $\sigma\rho_n$, we have
\[
\left(\frac{|\theta_{0,j}+\bLambda_j^\top\xi|-\lambda}{\sigma\rho_n}\right)_+-\left(\frac{|\bLambda_{j+p}^\top\xi|-\lambda}{\sigma\rho_n}\right)_+\leq -\left(1+\frac{w^2}{n}\right)\frac{t_n}{\sigma\rho_n}+\frac{2\|\Delta\|_\infty}{\sigma\rho_n}\,.
\]
Suppose arbitrary $\varepsilon>0$. By Theorem~\ref{thm:suff-cond-debiasing}, $\|\Delta\|_\infty=o_p(\sigma\rho_n)$. Conditioning on $\|\Delta\|_\infty/\sigma\rho_n<\varepsilon$, we have that
\[
\left(\frac{|\theta_{0,j}+\bLambda_j^\top\xi|-\lambda}{\sigma\rho_n}\right)_+-\left(\frac{|\bLambda_{j+p}^\top\xi|-\lambda}{\sigma\rho_n}\right)_+\leq -\left(1+\frac{w^2}{n}\right)\frac{t_n}{\sigma\rho_n}  +\varepsilon\le-\frac{t_n}{\sigma\rho_n}+\varepsilon\,.
\]
Let $h'_j$ be an indicator supported on the event above:
\[
h'_j = \mathbb I_{\left\{\left(\frac{|\theta_{0,j}+\bLambda_j^\top\xi|-\lambda}{\sigma\rho_n}\right)_+-\left(\frac{|\bLambda_{j+p}^\top\xi|-\lambda}{\sigma\rho_n}\right)_+\leq -\frac{t_n}{\sigma\rho_n}+\varepsilon\right\}}\,.
\]
Therefore, $h_j\leq h'_j$. Consider the covariance of the cross terms, $\Cov(h'_i;h'_j)$. The same proof for Lemma~\ref{lem:h'j-Cov} can show that $\max_{i\neq j}\Cov(h'_i;h'_j)=o(1)$. Hence by Chebyshev's inequality,
\[
\frac{1}{p}\sum_{j\in S_0}\left(h'_j-\Expected(h'_j)\right)\to 0\,.
\]
Since $h'_j$ are non-negative functions, it suffices to show that $\Expected(h'_j) \to 0$ for any $j\in S_0$. Note that 
\[
\Expected(h'_j) = \Prob\left(\left(\frac{|\theta_{0,j}+\bLambda_j^\top\xi|-\lambda}{\sigma\rho_n}\right)_+-\left(\frac{|\bLambda_{j+p}^\top\xi|-\lambda}{\sigma\rho_n}\right)_+\le -\frac{t_n}{\sigma\rho_n}+\varepsilon\right)\,.
\]
By Lemma~\ref{lem:marginal}, $\bLambda_{j+p}^\top\xi/(\sigma\rho_n)$ weakly converges to a standard Gaussian random variable, denoted by $Z$. Define the event $E_n :=\left\{\left|\frac{\bLambda_{j+p}^\top\xi}{\sigma\rho_n}\right|>\frac{\lambda}{\sigma\rho_n}\right\}$. For any constant $M\le\lambda/(\sigma\rho_n)$, we have
\[
\lim_{n\to\infty}\Prob\left(\left|\frac{\bLambda_{j+p}^\top\xi}{\sigma\rho_n}\right|>\frac{\lambda}{\sigma\rho_n}\right)\le \Prob\left(\left|Z\right|>M\right)\,.
\]
Since $\lambda/(\sigma\rho_n)\to\infty$ as shown in Lemma~\ref{lem:lambda-bigger-than-rho_n}, $M$ can be taken arbitrarily large. Therefore, $\Prob(E_n)\to0$. Note that $\left(\frac{|\theta_{0,j}+\bLambda_j^\top\xi|-\lambda}{\sigma\rho_n}\right)_+\ge0$. Then conditioning on the event $E_n^c$, we have
\begin{align*}
    \left(\frac{|\theta_{0,j}+\bLambda_j^\top\xi|-\lambda}{\sigma\rho_n}\right)_+-\left(\frac{|\bLambda_{j+p}^\top\xi|-\lambda}{\sigma\rho_n}\right)_+=\left(\frac{|\theta_{0,j}+\bLambda_j^\top\xi|-\lambda}{\sigma\rho_n}\right)_+\ge0
\end{align*}
On the other hand, $\liminf_{n\to\infty}t_n/(\sigma\rho_n)>0$. Without lost of generality, assume $\varepsilon<\liminf_{n\to\infty}t_n/(\sigma\rho_n)$. Then we have that for $n$ large enough,
\[
\Prob\left(\left(\frac{|\theta_{0,j}+\bLambda_j^\top\xi|-\lambda}{\sigma\rho_n}\right)_+-\left(\frac{|\bLambda_{j+p}^\top\xi|-\lambda}{\sigma\rho_n}\right)_+\le -\frac{t_n}{\sigma\rho_n}+\varepsilon\bigg\vert E_n^c\right)=0\,.
\]
Hence by total expectation,
\[
\Prob\left(\left(\frac{|\theta_{0,j}+\bLambda_j^\top\xi|-\lambda}{\sigma\rho_n}\right)_+-\left(\frac{|\bLambda_{j+p}^\top\xi|-\lambda}{\sigma\rho_n}\right)_+\le -\frac{t_n}{\sigma\rho_n}+\varepsilon\right)\le\Prob(E_n)\to0\,.
\]
Same result can be shown for $j\notin S_0$ as $\theta_{0,j}=0$ still has that $\left(\frac{|\theta_{0,j}+\bLambda_j^\top\xi|-\lambda}{\sigma\rho_n}\right)_+\ge0$. Thus the proof for $\lim_{n\to\infty}\widehat \FDP(t_n) = 0$ is complete.

By a similar argument, we can show that
\[
\Prob\left(\left(\frac{|\bLambda_j^\top\xi|-\lambda}{\sigma\rho_n}\right)_+-\left(\frac{|\bLambda_{j+p}^\top\xi|-\lambda}{\sigma\rho_n}\right)_+\ge \left(\frac{t_n}{\sigma\rho_n}-\varepsilon\right)\right)\to0\,.
\]
Thus $\FDP(t_n) \to 0$.

As for power, it suffices to show that for any $j\in S_0$,
\[
\Expected(g_j)\to1\,,
\]
where $g_j$ is the indicator for $W_j\ge t_n$. Similar to the proof for $\widehat \FDP$, since $\|\Delta\|_\infty = o_p(\sigma\rho_n)$, for $j\in S_0$, $W_j\ge t_n$ is implied by 
\begin{align}\label{eq:power-1-1}
\left(\frac{|\theta_{0,j}+\bLambda_j^\top\xi|-\lambda}{\sigma\rho_n}\right)_+-\left(\frac{|\bLambda_{j+p}^\top\xi|-\lambda}{\sigma\rho_n}\right)_+\ge\left(1+\frac{w^2}{n}\right)\frac{t_n}{\sigma\rho_n}  +\varepsilon_1\,.
\end{align}
for any fixed $\varepsilon_1>0$ with high probability. By Lemma~\ref{lem:lambda-bigger-than-rho_n}, $w^2 = o(n)$. Thus, the inequality above is implied by
\begin{align}\label{eq:power-1-2}
\left(\frac{|\theta_{0,j}+\bLambda_j^\top\xi|-\lambda}{\sigma\rho_n}\right)_+-\left(\frac{|\bLambda_{j+p}^\top\xi|-\lambda}{\sigma\rho_n}\right)_+\ge\left(1+\varepsilon_2\right)\frac{t_n}{\sigma\rho_n}  +\varepsilon_1\,,
\end{align}
for any fixed $\varepsilon_2<0$ when $n$ is large enough. 

By assumption, for some $\varepsilon'>0$, we have $\liminf_{n\to\infty}\frac{\mu_n-(1+\varepsilon')t_n}{\lambda}>1$. There exists $d<\infty$ such that 
\[
0<d<\min\left\{1\,, \liminf_{n\to\infty} \frac{\mu_n-(1+\varepsilon')t_n}{\lambda}-1\right\}\,.
\]
Then for $n$ large enough, $\mu_n-(1+\varepsilon')t_n>(1+d)\lambda$. Dividing both sides by $\sigma\rho_n$ and rearranging the terms, we have
\begin{equation}\label{eq:LCD-power-pf-1}
    \frac{\mu_n}{\sigma\rho_n}-\left(1+\frac{d}{2}\right)\frac{\lambda}{\sigma\rho_n}>(1+\varepsilon')\frac{t_n}{\sigma\rho_n}+\frac{d}{2}\frac{\lambda}{\sigma\rho_n}\,.
\end{equation}
Consider the event
\[
F_n = \left\{\max\left\{\left|\frac{\bLambda_j^\top\xi}{\sigma\rho_n}\right|,\left|\frac{\bLambda_{j+p}^\top\xi}{\sigma\rho_n}\right|\right\}> \frac{d}{2}\cdot \frac{\lambda}{\sigma\rho_n}\right\}\,.
\]
By Lemma~\ref{lem:lambda-bigger-than-rho_n}, $\sigma\rho_n=o(\lambda)$, hence $\Prob(F_n)\to 0$. Note that $d<1$, so on the event $F_n^c$, we have 
\[
\left|\frac{\theta_{0,j}}{\sigma\rho_n}+\frac{\bLambda_j^\top\xi}{\sigma\rho_n}\right|-\frac{\lambda}{\sigma\rho_n}\ge\frac{\mu_n}{\sigma\rho_n}-\left(1+\frac{d}{2}\right)\frac{\lambda}{\sigma\rho_n}\,,\quad \frac{|\bLambda_{j+p}^\top\xi|-\lambda}{\sigma\rho_n}<0\,,
\]
where we used that $\theta_{0,j}\ge\mu_n$. By the inequalities above and~\eqref{eq:LCD-power-pf-1}, we have that conditioning on $F_n^c$,
\begin{align*}
    \left(\frac{|\theta_{0,j}+\bLambda_j^\top\xi|-\lambda}{\sigma\rho_n}\right)_+-\left(\frac{|\bLambda_{j+p}^\top\xi|-\lambda}{\sigma\rho_n}\right)_+ &\ge \left(\frac{\mu_n}{\sigma\rho_n}-\left(1+\frac{d}{2}\right)\frac{\lambda}{\sigma\rho_n}\right)_+\\
    &>(1+\varepsilon')\frac{t_n}{\sigma\rho_n}+\frac{d}{2}\frac{\lambda}{\sigma\rho_n}\,.
\end{align*}
Note that $\lambda/(\sigma\rho_n)\to \infty$ so the inequality above implies~\eqref{eq:power-1-2} and so~\eqref{eq:power-1-1}, and as discussed the latter implies $W_j\ge t_n$ with high probability. Therefore,
\begin{align*}
    \Expected(g_j) & = \Prob(W_j\ge t_n)\\
        & \ge \Prob\left(\left(\frac{|\theta_{0,j}+\bLambda_j^\top\xi|-\lambda}{\sigma\rho_n}\right)_+-\left(\frac{|\bLambda_{j+p}^\top\xi|-\lambda}{\sigma\rho_n}\right)_+>(1+\varepsilon')\frac{t_n}{\sigma\rho_n}+\frac{d}{2}\frac{\lambda}{\sigma\rho_n}\right)\\
        &\ge \Prob(F_n^c)\,.
\end{align*}
Since $\Prob(F_n)\to0$, the proof for $\lim_{n\to\infty}\Power(t_n)\to 1$ is complete.
\section{Validation of the Examples in Section~\ref{sec:Power-FDR-analysis}}\label{pf:examples}

\subsection{Example 1}
\begin{enumerate}
\item  For condition \ref{item:cond-comp}, $\frac{\kappa_n^2}{\lambda^2\phi_w}=o\left(\sqrt{\frac{n}{\log p}}\cdot\min\left\{1,\sqrt{\frac{r}{n}},\frac{\sqrt{rn}}{w^2}\right\}\right)$,\\
on the left hand side (LHS):
\[
\lambda = C_\lambda\cdot\max\left\{\sqrt{\frac{\log p}{n}},\sqrt{\frac{(\log p)^3}{r}}\right\}=\Omega\left(\sqrt{\frac{(\log p)^3}{p}}\right)\,,
\]
\[w^2=O\left(\frac{p\sqrt{r\log(1/\delta)}}{\epsilon}\right)=O(p^{3/4})\,,
\]
\[
\kappa_n = 2\lambda\sqrt{s_0}+\frac{w^2}{n}=o\left(\sqrt{\frac{(\log p)^3}{p}}\cdot\frac{p^{1/8}}{\log p}\right)+\Theta\left(p^{-1/4}\right)=O\left(p^{-1/4}\right)\,.
\]
Also note that since $\phi_w = 1+w^2/n=\Theta(1)$, 
\[
\mbox{LHS} = O\left(\frac{p^{1/2}}{(\log p)^3}\right)\,.
\]
On the right hand side (RHS):
\[
\min\left\{1,\sqrt{\frac{r}{n}},\frac{\sqrt{rn}}{w^2}\right\} = 1 \implies \mbox{RHS} = \sqrt{\frac{p}{\log p}}\gg \frac{p^{1/2}}{(\log p)^3}\,.
\]
\item For condition \ref{item:cond-lambda}, $\sqrt{s_0} = O\left(\frac{\max\big\{\frac{1}{\sqrt{n}},\frac{\log p }{\sqrt{r}}\big\}}{\max\big\{\frac{w\log p}{\sqrt{nr}},\frac{w^2}{n}\sqrt{\frac{\log p}{r}}\big\}}\right)$,\\
RHS:
\[
\max\left\{\frac{1}{\sqrt{n}},\frac{\log p}{\sqrt{r}}\right\}=\frac{\log p}{\sqrt{p}},\ \max\left\{\frac{w\log p}{\sqrt{nr}},\frac{w^2}{n}\sqrt{\frac{\log p}{r}}\right\}=O\left(p^{-5/8}\log p\right)\,,
\]
so
\[
\mbox{RHS}=\Omega(p^{1/8})\gg\sqrt{s_0}=o(p^{1/8}/\log p)\,.
\]

\item Condition \ref{item:cond-sparsity}, $\frac{1}{\sigma\rho_n}\sqrt{s_0} = o\bigg(\min \bigg\{\frac{n}{w^2}\sqrt{\frac{r}{\log p}},\frac{\sqrt{nr}}{w\log p}\bigg\}\bigg)$,\\
LHS:
\[
    \rho_n^2 = \frac{n+r+1}{ nr}+\frac{w^2}{nr} = \Theta\left(\frac{1}{p}\right)\,.
\]
Multiply $\rho_n$ on both sides, then RHS
\[
    \rho_n\min \bigg\{\frac{n}{w^2}\sqrt{\frac{r}{\log p}},\frac{\sqrt{nr}}{w\log p}\bigg\} = \Omega\left(\frac{p^{1/8}}{\log p}\right)\gg \sqrt{s_0}\,.
\]
\item For condition \ref{item:cond-error}, $\frac{1}{\sigma\rho_n}\frac{4\kappa_n^2}{\lambda\phi_w} = o\bigg(\min\bigg\{\sqrt{\frac{n}{\log p}},\sqrt{\frac{r}{\log p}},\frac{n}{w^2}\sqrt{\frac{r}{\log p}},\frac{\sqrt{nr}}{w\log p}\bigg\}\bigg)$,\\
RHS:
\[
\min\bigg\{\sqrt{\frac{n}{\log p}},\sqrt{\frac{r}{\log p}},\frac{n}{w^2}\sqrt{\frac{r}{\log p}},\frac{\sqrt{nr}}{w\log p}\bigg\} = \sqrt{\frac{p}{\log p}}\,.
\]
LHS:
\[
\rho_n=\Theta(p^{-1/2}),\ \kappa_n^2 = \left(2\lambda\sqrt{s_0}+\frac{w^2}{n}\right)^2=O(p^{-1/2}),\ \lambda\phi_w = \Omega\left(\sqrt{(\log p)^3/p}\right)\,,
\]
\[
\implies LHS = O\left(\sqrt{\frac{p}{(\log p)^3}}\right)\ll RHS\,.
\]
Therefore $\rho^{-1}\|\Delta\|_\infty=o(1)$ with high probability under this setting.
\end{enumerate}

\subsection{Example 2}
\begin{enumerate}
\item For condition \ref{item:cond-comp}, $\frac{\kappa_n^2}{\lambda^2\phi_w}=o\left(\sqrt{\frac{n}{\log p}}\cdot\min\left\{1,\sqrt{\frac{r}{n}},\frac{\sqrt{rn}}{w^2}\right\}\right)$\\
LHS:
\[
\lambda = C_\lambda\max\left\{\sqrt{\frac{\log p}{n}},\sqrt{\frac{(\log p)^3}{r}}\right\}=\Theta\left((p\log p)^{-1}\right)\,,
\]
\[
w^2=\Theta\left(\frac{p\sqrt{r\log (1/\delta)}}{\epsilon}\right)=\Theta\left(p^2(\log p)^3\right)\,,
\]
\[
\kappa_n =2\lambda\sqrt{s_0}+\frac{w^2}{n}=\Theta\left(p^{-1/2}(\log p)^{-1}+p^{-1}(\log p)^{-3}\right)=\Theta\left(p^{-1/2}(\log p)^{-1}\right)\,.
\]
Since $\phi_w = 1+\frac{w^2}{n}=\Theta(1)$, then 
\[
\frac{\kappa_n^2}{\lambda^2\phi_w}=\Theta(p)\,.
\]
RHS:
\[
\min\left\{1,\sqrt{\frac{r}{n}},\frac{\sqrt{rn}}{w^2}\right\}=\sqrt{\frac{r}{n}}=p^{-1/2}(\log p)^{-1/2}\,,
\]
\[
\sqrt{\frac{n}{\log p}}\cdot p^{-1/2}(\log p)^{-1/2}=p
(\log p)^2\gg \mbox{LHS}=\Theta(p)\,.
\]
\item For condition \ref{item:cond-lambda}, $\sqrt{s_0} = O\left(\frac{\max\big\{\frac{1}{\sqrt{n}},\frac{\log p }{\sqrt{r}}\big\}}{\max\big\{\frac{w\log p}{\sqrt{nr}},\frac{w^2}{n}\sqrt{\frac{\log p}{r}}\big\}}\right)$,\\
RHS:
\[
\max\left\{\frac{1}{\sqrt{n}},\frac{\log p}{\sqrt{r}}\right\}=\frac{\log p}{\sqrt{r}} = p^{-1}(\log p)^{-3/2}\,,
\]
\[
\max\left\{\frac{w\log p}{\sqrt{nr}},\frac{w^2}{n}\sqrt{\frac{\log p}{r}}\right\}=\Theta\left(\max\left\{p^{-3/2}(\log p)^{-3},p^{-2}(\log p)^{-5} \right\}\right)=\Theta\left(p^{-3/2}(\log p)^{-3}\right)
\]
\[
\implies\frac{\max\big\{\frac{1}{\sqrt{n}},\frac{\log p }{\sqrt{r}}\big\}}{\max\big\{\frac{w\log p}{\sqrt{nr}},\frac{w^2}{n}\sqrt{\frac{\log p}{r}}\big\}}=p^{1/2}(\log p)^{3/2}\gg\mbox{LHS}= O(p^{1/2})\,.
\]
\item For condition \ref{item:cond-sparsity}, $\frac{1}{\sigma\rho_n}\sqrt{s_0} = o\bigg(\min \bigg\{\frac{n}{w^2}\sqrt{\frac{r}{\log p}},\frac{\sqrt{nr}}{w\log p}\bigg\}\bigg)$,\\
LHS:
\[
\rho_n^2 = \Theta\left(\frac{1}{n}+\frac{1}{r}+\frac{w^2}{nr}\right)=\Theta\left(p^{-2}(\log p)^{-5}\right)\implies \rho_n=\Theta\left( p^{-1}(\log p)^{-5/2}\right)\,,
\]
\[
\frac{1}{\sigma\rho_n}\sqrt{s_0} =O\left( p^{3/2}(\log p)^{5/2}\right)\,.
\]
RHS:
\[
\min\left\{ \frac{n}{w^2}\sqrt{\frac{r}{\log p}}, \frac{\sqrt{nr}}{w\log p}\right\}=p^{3/2}(\log p)^{3}\gg \mbox{LHS}=O(p^{3/2}(\log p)^{5/2})\,.
\]
\item For condition \ref{item:cond-error}, $\frac{1}{\sigma\rho_n}\frac{4\kappa_n^2}{\lambda\phi_w} = o\bigg(\min\bigg\{\sqrt{\frac{n}{\log p}},\sqrt{\frac{r}{\log p}},\frac{n}{w^2}\sqrt{\frac{r}{\log p}},\frac{\sqrt{nr}}{w\log p}\bigg\}\bigg)$,\\
LHS:
\[
\frac{1}{\sigma\rho_n}\frac{4\kappa_n^2}{\lambda\phi_w}=O\left(p(\log p)^{5/2}\cdot p^{-1}(\log p)^{-2}\cdot p\log p\right)=O\left( p(\log p)^{3/2}\right)\,.
\]
RHS:
\[
\min\bigg\{\sqrt{\frac{n}{\log p}},\sqrt{\frac{r}{\log p}},\frac{n}{w^2}\sqrt{\frac{r}{\log p}},\frac{\sqrt{nr}}{w\log p}\bigg\} = p(\log p)^2\gg \mbox{LHS}=O\left(p(\log p)^{3/2} \right)\,.
\]
\end{enumerate}

\section{Proof of Corollary~\ref{cor:sample-privacy-tradeoff}}\label{proof:cor:sample-privacy-tradeoff}
    We prove this corollary by checking each condition. Let $\alpha+\gamma=3+d$ for $d>0$, let $r=p^{\beta}$ where $\beta=2+h$ for some $h>0$ such that $h<\alpha-2$ and $h<d$.
    \begin{enumerate}
        \item $\max\{1,2+\beta-2\alpha-2\gamma+\min\{\alpha,\beta\}\}<\max\{0,1+\frac{1}{2}\beta-\alpha-\gamma\}+\min\{\frac{1}{2}\beta,\frac{1}{2}\alpha,\alpha+\gamma-1\}$.\\
        LHS: $\min\{\alpha,\beta\}=\beta$; $2+\beta-2\alpha-2\gamma+\min\{\alpha,\beta\}=2h-2d<0$. So $\mbox{LHS}=1$.\\
        RHS: $1+\frac{1}{2}\beta-\alpha-\gamma=-1-d+\frac{1}{2}h<0$, and $\min\{\frac{1}{2}\beta,\frac{1}{2}\alpha,\alpha+\gamma-1\}=\frac{1}{2}\beta=1+\frac{1}{2}h$. So $\mbox{RHS}=0+1+\frac{1}{2}h>1=\mbox{LHS}$.
        \item $\max\{1-\frac{1}{4}\beta-\frac{1}{2}\alpha-\frac{1}{2}\gamma,\frac{3}{2}-\alpha-\gamma\}<\max\{-\frac{1}{2}\alpha,-\frac{1}{2}\beta\}$.\\
        LHS: $1-\frac{1}{4}\beta-\frac{1}{2}\alpha-\frac{1}{2}\gamma=-1-\frac{1}{2}d-\frac{1}{4}h$; $\frac{3}{2}-\alpha-\gamma=\frac{3}{2}-d$. So $\mbox{LHS}=-1-\frac{1}{2}d-\frac{1}{4}h$.\\
        RHS: $\mbox{RHS}=\max\{-\frac{1}{2}\alpha,-\frac{1}{2}\beta\}=-1-\frac{1}{2}h>-1-\frac{1}{2}d-\frac{1}{4}h=\mbox{LHS}$
        \item $1+\min\{\alpha,\beta,\alpha+\gamma+\frac{1}{2}\beta-1\}<\min\{2\alpha+2\gamma-2,\frac{1}{2}\beta+\alpha+\gamma-1\}$.\\
        LHS: $\min\{\alpha,\beta,\alpha+\gamma+\frac{1}{2}\beta-1\}=2+h$.\\
        RHS: $2\alpha+2\gamma-2=4+2d$; $\frac{1}{2}\beta+\alpha+\gamma-1=3+d+\frac{1}{2}h$.\\
        $\mbox{RHS}=3+d+\frac{1}{2}h>3+d=\mbox{LHS}$
        \item $\min\{\frac{1}{2}\alpha,\frac{1}{2}\beta,\frac{1}{2}\alpha+\frac{1}{2}\gamma-\frac{1}{2}+\frac{1}{4}\beta\}+\max\{1-\min\{\alpha,\beta\},2+\beta-2\alpha-2\gamma\}<\\
        -\frac{1}{2}\min\{\alpha,\beta\}+\max\{0,1+\frac{1}{2}\beta-\alpha-\gamma\}+\min\{\frac{1}{2}\alpha,\frac{1}{2}\beta,\alpha+\gamma-1,\frac{1}{4}\beta+\frac{1}{2}(\alpha+\gamma)-\frac{1}{2}\}$.\\
        LHS: $\min\{\frac{1}{2}\alpha,\frac{1}{2}\beta,\frac{1}{2}\alpha+\frac{1}{2}\gamma-\frac{1}{2}+\frac{1}{4}\beta\}=1+\frac{1}{2}h$; $\max\{1-\min\{\alpha,\beta\},2+\beta-2\alpha-2\gamma\}=-1-h$. So $\mbox{LHS}=-\frac{1}{2}h$.\\
        RHS: $-\frac{1}{2}\min\{\alpha,\beta\}=-1-\frac{1}{2}h$; $\max\{0,1+\frac{1}{2}\beta-\alpha-\gamma\}=0$; $\min\{\frac{1}{2}\alpha,\frac{1}{2}\beta,\alpha+\gamma-1,\frac{1}{4}\beta+\frac{1}{2}(\alpha+\gamma)-\frac{1}{2}\}=1+\frac{1}{2}h$. \\
        So $\mbox{RHS}=0>-\frac{h}{2}=\mbox{LHS}$
    \end{enumerate}

\vskip 0.2in

\end{document}